\documentclass{article}

\usepackage[preprint]{neurips_2026}

\usepackage[utf8]{inputenc} 
\usepackage[T1]{fontenc}    
\usepackage{hyperref}       
\usepackage{url}            
\usepackage{booktabs}       
\usepackage{amsfonts}       
\usepackage{amsmath}        
\usepackage{amssymb}
\usepackage{nicefrac}       
\usepackage{microtype}      
\usepackage{xcolor}         

\usepackage{hyperref}
\usepackage{graphicx}
\usepackage{float}
\usepackage{url}
\usepackage{amsthm}
\usepackage{amsmath}
\usepackage{caption}
\usepackage{subcaption}
\usepackage{wrapfig}
\usepackage{comment}
\usepackage{booktabs}
\usepackage{siunitx}
\usepackage{multirow}
\usepackage{tikz}
\usetikzlibrary{arrows.meta,positioning}
\usepackage{pgfplots}
\pgfplotsset{compat=1.18}
\usepgfplotslibrary{groupplots}

\usepackage[most]{tcolorbox}
\usepackage{xcolor}

\definecolor{propback}{HTML}{EFF5FC}
\definecolor{propborder}{HTML}{A9C6E5}

\definecolor{scaleback}{HTML}{FFF0F0}
\definecolor{scaleborder}{HTML}{C96868}

\tcolorboxenvironment{proposition}{
  enhanced,
  breakable,
  colback=propback,
  colframe=propborder,
  boxrule=0.6pt,
  arc=2pt,
  left=7pt,
  right=7pt,
  top=5pt,
  bottom=5pt,
  before skip=8pt,
  after skip=8pt
}

\newtcolorbox{scalinglaw}{
  enhanced,
  breakable,
  colback=scaleback,
  colframe=scaleborder,
  boxrule=0.6pt,
  arc=2pt,
  left=7pt,
  right=7pt,
  top=6pt,
  bottom=6pt,
  before skip=8pt,
  after skip=8pt
}

\usepackage{algorithm}
\usepackage{algpseudocode}

\newtheorem{theorem}{Theorem}[section]
\newtheorem{proposition}[theorem]{Proposition}

\newtheorem{corollary}{Corollary}[theorem]

\newtheorem*{theorem*}{Theorem}
\newtheorem*{lemma*}{Lemma}
\newtheorem*{proposition*}{Proposition}

\newcommand{\model}{\mathrm{m}}

\newcommand{\E}{\mathbb{E}}
\newcommand{\argmin}{\operatorname*{argmin}}

\title{Information-Theoretic Scaling Laws \\ for Generative Diffusion Training}

\author{%
    Druv Pai\thanks{Equal contribution, ordered alphabetically.
} \\
    UC Berkeley \\
    \And
    Dejan Stan\v{c}evi\'{c} \footnotemark[1] \\
    Radboud University \\
    \And
    Yi Ma \\
    HKU, UC Berkeley \\
    \And
    Luca Ambrogioni \\
    Radboud University \\
}

\begin{document}

\maketitle

\begin{abstract}
    Diffusion models trained on increasingly large datasets sampled from the same data distribution will converge to the same learned distribution. However, in practice, a model trained on finite data depends on the particular training dataset; this dependence is reflected in properties of the model's outputs, such as memorization and sample quality. To understand this dependence through an information-theoretic lens, we study \textit{model information}, a computationally tractable measure of how much samples from the trained model depend on the particular training dataset. We explore its scaling behavior as model size and training dataset size vary. Specifically, we uncover two qualitatively distinct scaling regimes: (1) a nonlinear model information regime marking the transition from pure memorization to imperfect generalization, and (2) a generalization regime characterized by linear scaling of the model information and progressively improved generalization and sample quality. We then posit an information-theoretically principled scaling law to quantify these two regimes, and show that this law yields induced scaling laws for memorization and sample quality. We finally demonstrate that these induced scaling laws extrapolate well outside the fit's domain. Our results demonstrate information-theoretic principles for memorization and sample quality scaling laws in diffusion models. Our code is publicly available on \href{https://github.com/DruvPai/diffusion_info_theory_scaling_law}{GitHub}.
\end{abstract}

\section{Introduction}

Diffusion models \citep{sohldickstein2015deepunsupervisedlearningusing,ho2020denoisingdiffusionprobabilisticmodels,song2021scorebasedgenerativemodelingstochastic} have emerged as one of the most successful approaches to deep generative modeling, achieving remarkable results across diverse domains, including images, videos, audio, and scientific applications \citep{rombach2022highresolutionimagesynthesislatent,ho2022videodiffusionmodels,chen2020wavegradestimatinggradientswaveform,watson2023de}. Their training objective, denoising score matching \citep{vincent2011connection}, ensures that when sufficiently powerful models are trained with arbitrarily large amounts of training data, their output distribution converges to the true data distribution, independently of the exact training data used, and such trained models \textit{generalize} \citep{yoon2023diffusion,gu2023memorization} --- they generate realistic samples not contained in the training data. However, when trained on finite datasets, these models can \textit{memorize} and reproduce individual training examples \citep{carlini2023extracting,somepalli2023diffusion}, producing realistic samples without capturing the underlying data distribution \citep{marion2026understanding}. A fundamental methodological question is: {\em can we characterize the occurrence of memorization and generalization in diffusion models?}

Early empirical studies found that memorization diminishes with increasing training-set size and that this transition from memorization to generalization depends on the underlying denoiser neural network capacity, architecture, and training procedure \citep{yoon2023diffusion,gu2023memorization}. Theoretical work has further investigated why these models may generalize despite training objectives that favor memorization, highlighting the roles of model capacity \citep{buchanan2025edge}, implicit regularization \citep{bonnaire2025why}, and architectural inductive biases \citep{kamb2024analytic}. A complementary perspective on generalization emerges from studying the dependence of trained models on their training datasets: memorizing models are highly sensitive to the particular examples encountered during training, whereas generalizing models tend to capture the underlying data distribution more independently of the training set. For example, \citet{kadkhodaie2023generalization} related generalization to the convergence of independently trained models to the same output distribution, while other studies showed that the influence of individual training examples and the variability of learned denoisers decrease with dataset size \citep{dai2026outputs,wang2026randommatrix}. However, the quantitative behavior of this dependence of the learned distribution on the training dataset across dataset sizes and model capacities remains under-explored.

\paragraph{Our contributions.} We define the \textit{model information}, an information-theoretic measure of how much a generated sample reveals about the identity of the model that produced it (and thus, under deterministic training, the exact training dataset used). The idea is sketched in Figure \ref{fig:model_information_schematic}. We show:
\begin{enumerate}
    \item Theoretically, model information characterizes statistical measures of bias, variance, memorization, and generalization (Section \ref{sec:model_information}).
    \item Empirically, model information is relatively cheap to compute and follows a predictable scaling law, the form of which we derive via solving a simple and analytically tractable linear-Gaussian model, as training-set size and model capacity increase (Section \ref{sub:information_scaling_law}).
    \item Model information is strongly correlated with empirically designed measures of memorization, generalization, and sample quality (Section \ref{sub:information_vs_memorization}) to the extent that it induces scaling laws for these quantities that extrapolate well to new parameters (Section \ref{sub:extrapolation}).
\end{enumerate}
Together, these results establish a quantitative framework for information-theoretically principled predictable scaling of memorization, generalization, and sample quality in diffusion models.

\begin{figure}[t]
    \centering
    \includegraphics[width=0.85\linewidth]{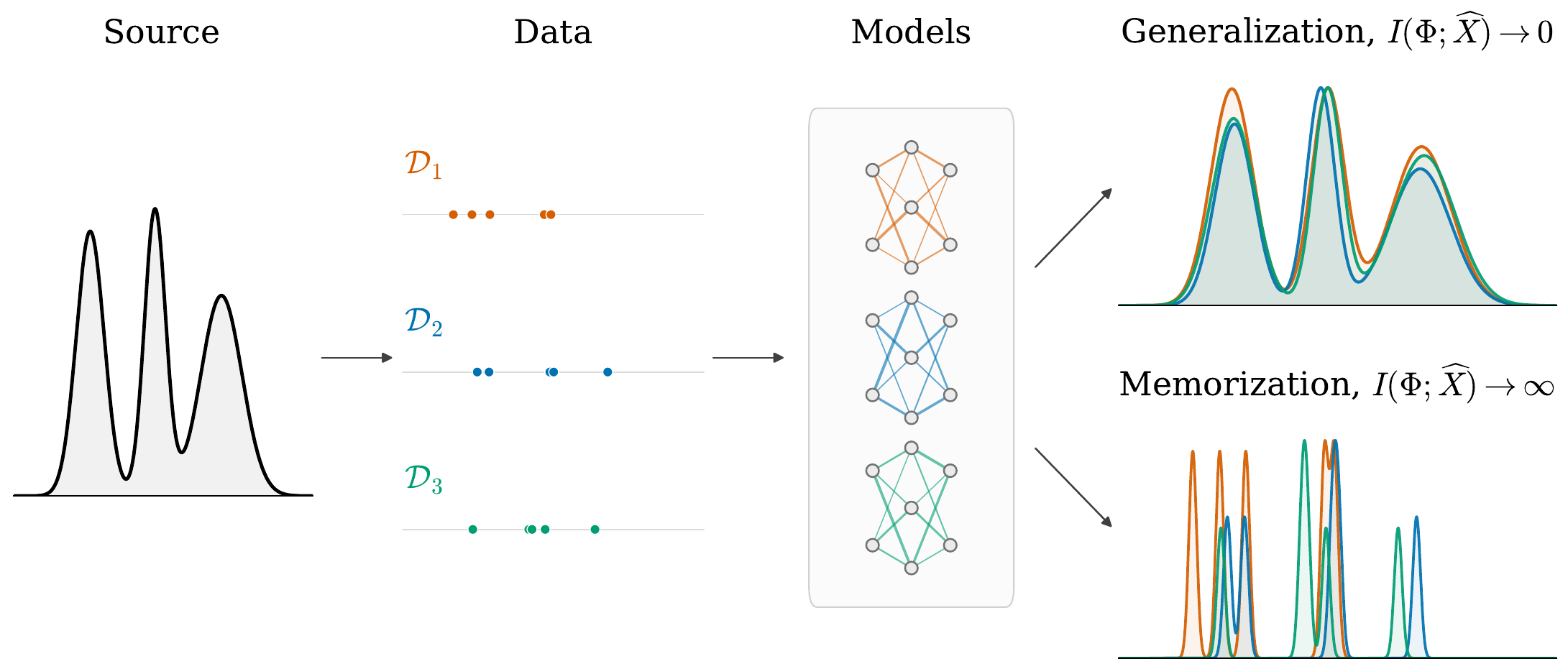}
    \caption{\small \textbf{Model information distinguishes
        generalization from memorization.} A source distribution
        $p_{\phi_\star}$ generates independent finite datasets $\mathcal
        D_i$, which are used to train models $\Phi_i$. In the generalization regime, the
        trained models have nearly
        indistinguishable output distributions, so a sample
        $\widehat X$ carries little information about which model (or
        training dataset)
        produced it. In the memorization regime, each trained model's outputs
        concentrate around its own training data, so generated
    samples reveal the training dataset.}
    \label{fig:model_information_schematic}
\end{figure}

\section{Generative Diffusion Models}
Diffusion models learn to generate samples by learning to predict the
clean sample from a noisy one. To sample, they invert a stochastic
process that gradually transforms
the target distribution $\phi_{\star}$ into a simple distribution,
such as an isotropic Gaussian distribution. The process is usually
represented by a linear SDE driven by an independent Wiener process $W_t$:
\begin{equation}\label{eq:forward_sde}
    dX_t = f(t)\,X_t\,dt + g(t)\,dW_t, \qquad X_0 \sim \phi_{\star}.
\end{equation}
The marginals of the forward process \eqref{eq:forward_sde} can be
written in closed form as \citep{sarkka2019applied}
\begin{align}\label{eq:forward_solution}
    &X_t \stackrel{d}{=} \alpha_t X_0 + \sigma_t \epsilon,
    \qquad X_{0} \sim \phi_{\star}, \qquad
    \epsilon \sim \mathcal{N}(0,I),
\end{align}
where $\alpha_t$ and $\sigma_t$ are non-negative functions that can
be derived from $f(t)$ and $g(t)$. The time-reversed stochastic
process of \eqref{eq:forward_sde} is given by
\citep{anderson1982reverse}:
\begin{equation}\label{eq:reverse_sde}
    dX_t = (f(t)\,X_t - g^2(t)\,\nabla \log
    p_t^{\phi_{\star}}(X_t))\, dt
    + g(t)\,d\bar{W}_t,
\end{equation}
where $\bar{W}_t$ is a Wiener process in reverse time, and the
dynamics are solved from a large terminal time $T$ back to $0$. Here,
$p_t^{\phi_{\star}}$ is the density of the marginal distribution of
$X_t$, and $\nabla \log p_t^{\phi_{\star}}$ is the corresponding
score. The notation emphasizes that both
quantities depend on the
data source $\phi_{\star}$. Since this score is not directly
available to us, we learn it from a
dataset $\mathcal{D} =
\{x^1, \dots,
x^N\} \subseteq \mathbb{R}^d$ sampled from
$\phi_{\star}$. Once we learn a score estimator $\hat{s}_t$ for each
time $t \in [0,
T]$, we may run the reverse SDE
\eqref{eq:reverse_sde} with the learned score $\hat{s}$ to sample
from a learned model $\hat{\phi} \approx \phi_{\star}$.

\section{Model Information}
\label{sec:model_information}

For a ground truth source $\phi_\star$, we consider a probability
distribution over learned models $\Pi$, whose randomness
is induced by the selection of the dataset
$\mathcal{D}$.\footnote{and any randomness
    in the training algorithm, but empirically this seems to matter little
\citep{bertrand2025closed}.} We
write $\Phi \sim \Pi$ to denote this random learned model. In
this context, memorization and generalization can be succinctly described as:
\begin{center}
    \textit{Memorization:} The trained model
    $\Phi$ depends \textit{extremely strongly} on the specific dataset
    $\mathcal{D}$. \\ \vspace{0.5em}
    \textit{Generalization:} The trained model
    $\Phi$ depends \textit{extremely weakly} on the
    specific dataset $\mathcal{D}$.
\end{center}
To define this dependence quantitatively, we use an
information-theoretic perspective: given a sample $\hat{X}$ from the
learned model $\Phi$, how much information does it contain about the
dataset $\mathcal{D}$, or equivalently, the realization of
$\Phi$ which generated $\hat{X}$? This allows us to define the
\textit{model information} as the mutual information between $\Phi$
and $\hat{X}$, i.e., $I(\Phi;\hat{X})$. Concretely, define the
\textit{model information curve} $I_{\model}(t)$ as the mutual
information of $\Phi$ and $\hat{X}_t$ (s.t.~$I(\Phi;\hat{X}) = I_{\model}(0)$):
\begin{equation}\label{eq:model_information}
    I_\model(t) := I(\Phi;\hat{X}_t), \qquad \hat{X} \sim \Phi,
    \qquad \hat{X}_t
    \sim \mathcal{N}(\alpha_t \hat{X}, \sigma_t^2 I),
\end{equation}
i.e., $\hat{X}_t \sim p_t^{\Phi}$. Note that e.g., the neural network
architecture and
optimization procedure used to estimate the score, the sampling
algorithm, etc.,  all enter through the
learned model distribution $\Pi$. 
Schematic and graphical representations of the model information are shown in Figures \ref{fig:model_information_schematic} and \ref{fig:model_information_variables}.

\begin{figure}[t!]
    \centering
    \resizebox{0.72\linewidth}{!}{%
    \begin{tikzpicture}[
            node distance=1.0cm and 1.25cm,
            var/.style={draw, rounded corners, align=center, inner sep=4pt},
            obs/.style={draw, rounded corners, align=center, inner sep=4pt,
            fill=gray!10},
            arrow/.style={-Latex, thick},
            indirect/.style={-Latex, thick, densely dotted}
        ]
        \node[var] (source) {\(\phi_\star\)\\data source};
        \node[obs, below left=of source] (data)
        {\(\mathcal D=\{X^i\}_{i=1}^N\)\\training data};
        \node[obs, below right=of source, xshift=0.3cm] (pop)
        {\(X\sim\phi_\star\)\\population sample};
        \node[var, below=of data] (model)
        {\(\Phi\sim\Pi\)\\learned model};
        \node[obs, right=of model] (sample)
        {\(\hat{X}\sim\Phi\)\\model sample};
        \node[obs, right=of sample] (npop)
        {\(X_t = \alpha_t X+\sigma_t\epsilon\)\\noisy observation};
        \node[obs, below=of sample] (noised)
        {\(\hat{X}_t=\alpha_t\hat{X}+\sigma_t\epsilon\)\\noisy observation};
        \node[var, right=of noised] (noise)
        {\(\epsilon\sim\mathcal N(0,I)\)\\forward noise};

        \draw[arrow] (source) -- (data);
        \draw[arrow] (source) -- (pop);
        \draw[arrow] (data) -- (model);
        \draw[arrow] (model) -- (sample);
        \draw[arrow] (pop) -- (npop);
        \draw[arrow] (noise) -- (npop);
        \draw[arrow] (sample) -- (noised);
        \draw[arrow] (noise) -- (noised);
        \draw[indirect] (model) -- node[above, font=\small]
        {\(p_t^\Phi\)} (noised);
        \draw[indirect] (source) -- node[below, font=\small]
        {\(p_t^{\phi_\star}\)} (npop);
    \end{tikzpicture}}
    \caption{\small \textbf{Variable dependence for the model
        information.} The data
        source \(\phi_\star\) generates the training dataset
        \(\mathcal D\), whose randomness induces the distribution \(\Pi\) over
        learned models. A model \(\Phi\) trained on \(\mathcal D\) generates
    \(\hat{X}\), and the forward channel produces \(\hat{X}_t\).}
    \label{fig:model_information_variables}
\end{figure}
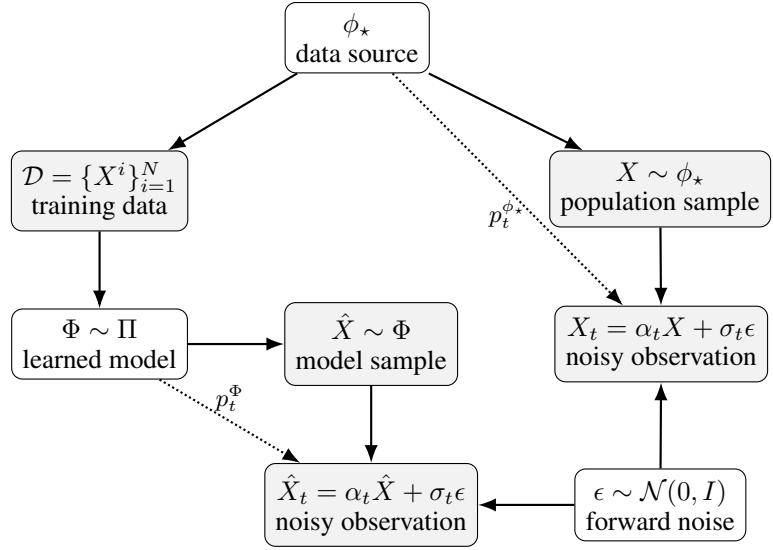

Mutual informations are in general extremely hard to compute
\citep{mcallester2020formal}.
However, \textit{surprisingly}, in this case we can tractably estimate the
model information by using the following score-based representation
for its time derivative.

\begin{proposition}
    \label{prop:model_information_rate_score_based}
    Suppose the learnable models $\phi \in
    \operatorname{supp}(\Pi)$ each have finite second moment. Let
    \begin{equation}
        \bar{p}_t(x) = \mathbb{E}_{\Phi \sim \Pi}[p_t^{\Phi}(x)]
        \label{eq:noisy_marginal_law}
    \end{equation}
    be the marginal density of noised samples drawn from the learned models. Then
    \begin{equation}
        \dot I_{\model}(t)
        =
        -\frac{g(t)^2}{2}
        \mathbb{E}_{\Phi \sim \Pi,~ \hat{X}_t \sim p_t^{\Phi}}
        [
            \|
            \nabla \log p_t^{\Phi}(\hat{X}_t)
            -
            \nabla \log \bar{p}_t(\hat{X}_t)
            \|^2
        ].
    \end{equation}
\end{proposition}

\begin{proof}
    See Appendix \ref{sec:model_information_rate_score_space}.
\end{proof}

Note that the score $\nabla \log \bar{p}_t$ is the \textit{posterior}
mean of the
individual model scores:
\begin{equation}\label{eq:mixture_score_expectation}
    \nabla \log \bar{p}_t(x) = \mathbb{E}_{\Phi \sim \Pi \mid
    \{\hat{X}_t = x\}}[
    \nabla \log p_t^{\Phi}(x)] =
    \int\nabla \log p_{t}^{\phi}(x)~\mathrm{d}\mathbb{P}_{\Pi \mid
    \hat{X}_t}(\phi \mid x).
\end{equation}
This identity gives three ways to estimate the mixture score: we can
fit it directly by denoising samples generated by trained models,
learn the posterior probabilities with a classifier, or compute these
probabilities directly from the trained models via integration.
In Appendix \ref{sub:conditional_score_expectation_estimator} we discuss the
three estimators and the estimation error from using a finite and
small number of trained models. In practice, we use the third method,
since all three estimators yield qualitatively similar results.

When the terminal distribution is independent of the learned model,
we can recover the model information by integrating this rate. In
particular, if $\alpha_T = 0$, such as in DDPM
\citep{ho2020denoisingdiffusionprobabilisticmodels,song2021scorebasedgenerativemodelingstochastic},
DDIM
\citep{song2022denoisingdiffusionimplicitmodels}, and flow
matching \citep{lipman2022flow}, then $\hat{X}_T = \sigma_T \epsilon$
is independent of $\Phi$, and
hence $I_{\model}(T) = I(\Phi;\hat{X}_T) = 0$. Therefore, we may compute
\begin{equation}\label{eq:model_information_integration}
    I_{\model}(t)
    =
    \int_T^t \dot I_{\model}(s)\,ds
    =
    -\int_t^T \dot I_{\model}(s)\,ds,
    \qquad
    \text{in particular}
    \qquad
    I(\Phi;\hat{X})
    =
    -\int_0^T \dot I_{\model}(s)\,ds,
\end{equation}
where \(\hat{X} = \hat{X}_0\) is a clean sample from \(\Phi\). When
\(\alpha_T > 0\) but \(\sigma_T/\alpha_T\) is large, such as in EDM
\citep{karras2022elucidating}, the same equality approximately holds.
The dependency graph of the random variables required for the definition of model information is in
Figure~\ref{fig:model_information_variables}.

\subsection{Properties of Model Information}
\label{sub:properties_model_information}

We present several structural properties of model information that
clarify what the quantity measures beyond its definition. Proofs of
the results stated in this section can be found in
Appendix~\ref{app:properties_model_information}.


\begin{proposition}
    \label{thm:model_information_bias_variance}
    For every fixed $x$ such that the relevant conditional
    expectations are well-defined,
    \begin{align}
        &\mathbb E_{\Phi \sim \Pi \mid \{\hat
        X_t=x\}}
        [
            \|\nabla \log p_t^{\Phi}(x)-\nabla \log p_t^{\phi_{\star}}(x)\|^2
        ] \\
        \label{eq:score_bias_variance_decomposition}
        &=
        \|\nabla \log \bar{p}_t(x)-\nabla \log p_t^{\phi_{\star}}(x)\|^2
        +
        \mathbb E_{\Phi \sim \Pi \mid \{\hat{X}_t = x\}}
        [
            \|\nabla \log p_t^{\Phi}(x)-\nabla \log \bar{p}_t(x)\|^2
        ].
    \end{align}
\end{proposition}

We may connect the two terms in
\eqref{eq:score_bias_variance_decomposition} and a conditional bias--variance
decomposition, as discussed by
\citet{kadkhodaie2023generalization}.
The first term $\|\nabla \log \bar
p_t(x)-\nabla \log p_t^{\phi_{\star}}(x)\|^2$ measures the systematic
discrepancy between
the ensemble score and the population score, and therefore
corresponds to a bias term.
The second term $\mathbb E_{\Phi \sim \Pi \mid
\{\hat{X}_t = x\}}
[
    \|\nabla \log p_t^{\Phi}(x)-\nabla \log \bar{p}_t(x)\|^2
]$ measures the dispersion of trained model scores around their
posterior average, and therefore corresponds to a variance term.

Furthermore, the model information tells us how much an observable
can fluctuate around the mean. Let $P_t^\phi$ and $\bar P_t$ denote the probability
measures on $\mathcal X$ with densities $p_t^\phi$ and $\bar p_t$ respectively.

\begin{proposition}
    \label{thm:event_stability_bound}
    Let $Z$ be any random variable such that $Z\to\Phi\to\hat X_t$ forms a Markov chain.
    Then, for any measurable assignment $(Z,\Phi)\mapsto
    A_{Z,\Phi}\subseteq\mathcal X$,
    \begin{equation}
        \mathbb E_{Z,\Phi}\left[\left|P_t^\Phi(A_{Z,\Phi})-\bar
        P_t(A_{Z,\Phi})\right|\right]
        \leq
        \sqrt{1-\exp\left(-I_m(t)\right)}.
        \label{eq:event_stability_bound}
    \end{equation}
\end{proposition}

Alternatively, we can substitute the right-hand side of the inequality
with $\sqrt{I_m(t)/2}$.

A particularly relevant case is dataset-specific memorization. Let
$\mathcal D$ be the random training dataset jointly inducing the
learned model $\Phi$, and let $A_{\mathrm{mem}}(\mathcal
D)\subseteq\mathcal X$ be any measurable set of samples declared
memorized relative to $\mathcal D$. Since generation from the trained
model satisfies $\mathcal D\to\Phi\to\hat X_t$,
Proposition~\ref{thm:event_stability_bound} applies with $Z=\mathcal
D$ and $A_{\mathcal D,\Phi}=A_{\mathrm{mem}}(\mathcal D)$, yielding
\begin{equation}
    \mathbb E_{\mathcal
    D,\Phi}\left[P_t^\Phi\left(A_{\mathrm{mem}}(\mathcal D)\right)\right]
    \leq
    \mathbb E_{\mathcal D}\left[\bar
    P_t\left(A_{\mathrm{mem}}(\mathcal D)\right)\right]
    +
    \sqrt{1-\exp\left(-I_m(t)\right)}.
    \label{eq:memorization_ratio_bound}
\end{equation}

\section{Scaling Laws for Model Information, Memorization, and Sample Quality}
\label{sec:model_information_scaling}


We now study how model information depends on the model size $P$ and the
training-set size $N$, and how it relates to memorization and sample quality.
We train unconditional EDM diffusion models \citep{karras2022elucidating}
with DDPM++ backbones \citep{song2021scorebasedgenerativemodelingstochastic} (a UNet with, say, $L$ layers in each of the encoder and decoder)
on ImageNet \citep{deng2009imagenet}. We report the following findings:
\begin{enumerate}
    \item Model information follows a scaling law whose form we derive from
        the Marchenko--Pastur law for the spectrum of an empirical covariance
        (Section~\ref{sub:information_scaling_law}). We can measure it from a
        few trained models per \((P, N)\) pair.
    \item We show that model information is strongly correlated with memorization when $N$ is small, and with FID when $N$ is large
        (Section~\ref{sub:information_vs_memorization}).
    \item From these two findings, we extrapolate model information,
        memorization, and FID to a new \(P\) and \(N\) and show that these predictions are accurate (Section~\ref{sub:extrapolation}).
\end{enumerate}
In Appendix~\ref{sec:experiments} we repeat the analysis on CIFAR-10,
Fashion-MNIST, and FFHQ.

\subsection{A Scaling Law for Model Information}
\label{sub:information_scaling_law}

\paragraph{Setup.} We train DDPM++ models of four sizes, from $6.7\times10^6$ to
$8.3\times10^7$ parameters (i.e., encoder/decoder depth $L \in [2, 5]$), on training sets of twelve sizes from $N=10$ to
$N=100{,}000$, drawn from the $N_{\mathrm{full}}=1{,}271{,}167$ images of
ImageNet. For each $(P,N)$ we train four models on independently drawn
training sets. We use two of them for the mixture score and evaluate the rate
on the other two; later in this section we show that this is surprisingly enough for a good estimate. When $N$ is small enough that the models memorize their
training sets, $I_\model(0)=\infty$, so we integrate the estimated rate over
$[t_{\min},T]$ for a small $t_{\min}>0$ and write $\hat I_\model(P,N)$ for
the result. The estimator is described and justified in
Appendix~\ref{sec:model_information_computations}; all other experimental setup details are in
Appendix~\ref{sub:appendix_experimental_protocol}.

\paragraph{The linear-Gaussian model.} We derive the form of the scaling law from a
model in which model information can be computed in closed form
(Appendix~\ref{app:linear_gaussian}). The data are \(\mathcal N(0,I_d)\), the
learned model is the empirical covariance of \(N\) samples, and model
information is evaluated after adding Gaussian noise of variance \(\rho\). As
\(N,d\to\infty\) with \(N/d\to\eta\), the eigenvalues of the empirical
covariance follow the Marchenko--Pastur law, and for small \(\rho\) model
information per dimension converges to
(Propositions~\ref{prop:linear_gaussian_asymptotic}
and~\ref{prop:linear_gaussian_small_rho})
\begin{scalinglaw}
\begin{equation}
    \label{eq:linear_gaussian_information_profile}
    \mathcal I_\ell(\eta)
    =
    \begin{cases}
        \ell(1-\eta)
        +\tfrac12\big[\eta\log\eta+\eta+(1-\eta)\log(1-\eta)\big],
        & \eta<1,\\[2pt]
        \tfrac12\big[1+(\eta-1)\log(1-\eta^{-1})\big],
        & \eta\ge 1,
    \end{cases}
\end{equation}
\end{scalinglaw}
where \(\ell=\tfrac12\log(1/\rho)\). For \(N<d\) the empirical covariance is
zero on \(d-N\) directions, so the model memorizes its training set, and each
of these directions adds \(\ell\) nats to model information. For \(N>d\) the
covariance has full rank, the model generalizes, and model information decays
as \(d^2/(4N)\).

\paragraph{From the linear-Gaussian model to the scaling law.} In the linear-Gaussian
model, all \(d\) directions stop memorizing at the same training-set size
\(N=d\), so the transition from memorization to generalization is sharp. In
our measurements the transition is gradual, extending over more than a
decade of \(N\) (Figure~\ref{fig:imagenet_scaling}). We reconcile
these observations with the model by positing that different parts of
the data are memorized more easily than others. To model this, we
split the coordinates into independent
groups of different sizes. Each group is modeled using the
linear-Gaussian model, so a group of size \(d_z\) stops memorizing at
\(N=d_z\), and model information is the sum over the groups. We want
the distribution of group sizes to have only two parameters: one for
the location of the transition and one for its width. The transition
is spread over a range of \(\log N\), so we take \(\log d_z\) to be
normal with mean \(\log N_{\mathrm c}(P)\) and standard deviation
\(\sigma\). This is the distribution of largest entropy for a given
mean and variance of \(\log d_z\). Then \(N_{\mathrm c}(P)\) is the
training-set size at which a typical group stops memorizing; we
call it the critical training-set size. A larger model can memorize a
larger training set, so \(N_{\mathrm c}(P)\) increases with \(P\). We
take \(N_{\mathrm c}(P)\) and the number of groups to be power laws in \(P\).
Finally, the estimate is positive even when all models train on the same
data, because of the variation between training seeds and the bias of the
estimator (Appendix~\ref{sub:appendix_scaling_law_bias}). We add a term
\(F(P)\) for the value of the estimate when all models train on the same data. With \(\widetilde P=P/P_{\mathrm{ref}}\) for a fixed unit \(P_{\mathrm{ref}}\), the such-derived scaling law is
\begin{scalinglaw}
\vspace{-1em}
\begin{align}
    \label{eq:linear_gaussian_information_law}
    \hat I_\model(P,N)
    &\approx
    \underbrace{c_0+c_1\widetilde P^{-\beta}}_{F(P)}
    +\underbrace{k\widetilde P^{-b}\,
        \E_{z\sim\mathcal N(0,1)}\Big[d_z\,
            \mathcal
    I_\ell\Big(\frac{N}{d_z}\Big)\Big]}_{E(P,N)},
    \\
    d_z&=N_{\mathrm c}(P)\,e^{\sigma z},
    \qquad
    N_{\mathrm c}(P)=N_{\mathrm c}(P_{\mathrm{ref}})\widetilde P^{a}.
    \nonumber
\end{align}
\end{scalinglaw}
For \(N\ll N_{\mathrm c}(P)\), \(E(P,N)\) is approximately constant,
and the scaling law attains its maximum over \(N\),
\(\hat I_{\max}(P):=\sup_{N>0}\hat I_\model(P,N)=F(P)+k\widetilde
P^{-b}\ell N_{\mathrm c}(P)e^{\sigma^2/2}\). For \(N\gg N_{\mathrm
c}(P)\), \(E(P,N)\) decays as \(N^{-1}\); we call the range
\(N\ge e^{2\sigma}N_{\mathrm c}(P)\), which contains all but about \(2\%\) of
the group sizes \(d_z\), the \emph{\(N^{-1}\) regime}. We
fit the nine parameters
\(c_0,c_1,\beta,k,b,N_{\mathrm c}(P_{\mathrm{ref}}),a,\sigma,\ell\) by least
squares in log-information on all \((P,N)\) pairs
(Appendix~\ref{sub:appendix_scaling_law_fit}).\footnote{\citet{buchanan2025edge}
    estimate a critical dataset size
that grows linearly in \(P\) in a toy setup.}

\begin{figure}[t!]
    \centering
    \input{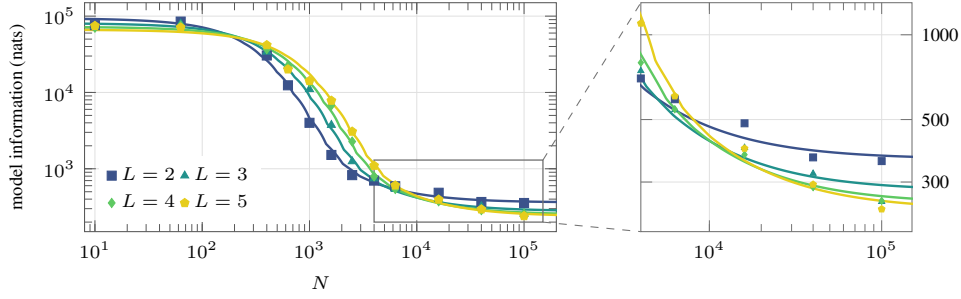}

\noindent\makebox[\linewidth][c]{%
    \begin{tikzpicture}
        \begin{axis}[
                name=main,
                width=0.56\linewidth,
                height=0.33\linewidth,
                xmode=log,
                ymode=log,
                xmin=8, xmax=2e5, ymin=150, ymax=1.5e5,
                xlabel={$N$},
                ylabel={model information (nats)},
                grid=major,
                grid style={line width=0.1pt, draw=gray!24},
                tick label style={font=\scriptsize},
                label style={font=\scriptsize},
                legend style={font=\scriptsize, draw=none, fill=white,
                at={(0.03,0.04)}, anchor=south west, legend columns=2},
                axis line style={black!65},
                tick style={black!65}
            ]
            \edmRawFit{figures/pgfplots/data/imagenet64_lawfit_l2.csv}{edmDepthTwo}
            \edmRawPoints{figures/pgfplots/data/imagenet64_raw_l2.csv}{edmDepthTwo}{square*}{$L=2$}
            \edmRawFit{figures/pgfplots/data/imagenet64_lawfit_l3.csv}{edmDepthThree}
            \edmRawPoints{figures/pgfplots/data/imagenet64_raw_l3.csv}{edmDepthThree}{triangle*}{$L=3$}
            \edmRawFit{figures/pgfplots/data/imagenet64_lawfit_l4.csv}{edmDepthFour}
            \edmRawPoints{figures/pgfplots/data/imagenet64_raw_l4.csv}{edmDepthFour}{diamond*}{$L=4$}
            \edmRawFit{figures/pgfplots/data/imagenet64_lawfit_l5.csv}{edmDepthFive}
            \edmRawPoints{figures/pgfplots/data/imagenet64_raw_l5.csv}{edmDepthFive}{pentagon*}{$L=5$}
            \draw[black!55, line width=0.5pt]
            (axis cs:4000,200) rectangle (axis cs:1.5e5,1300);
            \coordinate (zoomtop) at (axis cs:1.5e5,1300);
            \coordinate (zoombot) at (axis cs:1.5e5,200);
        \end{axis}
        \begin{axis}[
                name=zoom,
                at={(main.south east)}, anchor=south west,
                xshift=0.08\linewidth,
                width=0.37\linewidth,
                height=0.33\linewidth,
                xmode=log, ymode=log,
                xmin=4000, xmax=1.5e5, ymin=200, ymax=1300,
                xtick={1e4,1e5}, xticklabels={$10^4$,$10^5$},
                ytick={300,500,1000}, yticklabels={300,500,1000},
                yticklabel pos=right,
                grid=major,
                grid style={line width=0.1pt, draw=gray!24},
                tick label style={font=\scriptsize},
                axis line style={black!55},
                tick style={black!55}
            ]
            \edmRawFit{figures/pgfplots/data/imagenet64_lawfit_l2.csv}{edmDepthTwo}
            \edmZoomPoints{figures/pgfplots/data/imagenet64_raw_l2.csv}{edmDepthTwo}{square*}
            \edmRawFit{figures/pgfplots/data/imagenet64_lawfit_l3.csv}{edmDepthThree}
            \edmZoomPoints{figures/pgfplots/data/imagenet64_raw_l3.csv}{edmDepthThree}{triangle*}
            \edmRawFit{figures/pgfplots/data/imagenet64_lawfit_l4.csv}{edmDepthFour}
            \edmZoomPoints{figures/pgfplots/data/imagenet64_raw_l4.csv}{edmDepthFour}{diamond*}
            \edmRawFit{figures/pgfplots/data/imagenet64_lawfit_l5.csv}{edmDepthFive}
            \edmZoomPoints{figures/pgfplots/data/imagenet64_raw_l5.csv}{edmDepthFive}{pentagon*}
        \end{axis}
        \draw[black!55, line width=0.4pt, dashed] (zoomtop) --
        (zoom.north west);
        \draw[black!55, line width=0.4pt, dashed] (zoombot) --
        (zoom.south west);
    \end{tikzpicture}%
}
    \vspace{-1.5em}
    \caption{\small \textbf{Model information scaling on ImageNet.}
        Model information as a function of training-set size for four
        model sizes (written in terms of depth \(L\)), with the fit of
        \eqref{eq:linear_gaussian_information_law}. The right panel 
        shows the $N^{-1}$ regime.}
    \label{fig:imagenet_scaling}
\end{figure}

\paragraph{Results.} The fitted scaling law is close to the measurements at
every \((P,N)\) pair (Figure~\ref{fig:imagenet_scaling}). At every model size
we see the three regimes of the linear-Gaussian model: model information is
close to \(\hat I_{\max}(P)\) while the models memorize, it decreases around
\(N_{\mathrm c}(P)\), and in the \(N^{-1}\) regime it decreases to \(F(P)\)
as \(N^{-1}\). For larger models the transition is at larger
\(N\), roughly in proportion to the width of the network.
Importantly, our scaling law's theoretical foundation is crucial: if we replace \(\mathcal I_\ell\) in \eqref{eq:linear_gaussian_information_law}
by a sigmoid in \(\log N\) and add a power of \(N\) whose exponent is a
fitted parameter (i.e., a naive and purely empirical scaling law), the fit is non-trivially worse (Appendix~\ref{sub:appendix_law_comparison}).

\paragraph{Computing model information from only a few trained
models.} Model information is a property of the
distribution of trained models, and we estimate it from a few draws from this
distribution. The estimate changes little with the number of draws.

\begin{wrapfigure}{r}{0.44\linewidth}
    \vspace{-1em}
    \centering
    \begin{minipage}{\linewidth}
        \centering
        \input{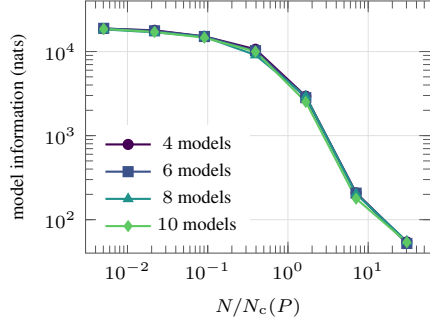}

\begin{tikzpicture}
\begin{axis}[
  width=\linewidth, height=0.80\linewidth,
  xmode=log, ymode=log,
  xmin=0.003, xmax=60, ymin=40, ymax=4e4,
  xlabel={$N/N_{\mathrm c}(P)$},
  ylabel={model information (nats)},
    ytick={1e2,1e3,1e4},
  grid=major, grid style={line width=0.1pt, draw=gray!24},
  tick label style={font=\scriptsize}, label style={font=\scriptsize},
  title style={font=\small},
  legend style={font=\scriptsize, draw=none, fill=white,
    at={(0.03,0.03)}, anchor=south west, legend columns=1},
  axis line style={black!65}, tick style={black!65}
]
\edmInfoSeries{figures/pgfplots/data/fashion_mnist_budget_l4_m4.csv}{edmDepthOne}{*}{$4$ models}
\edmInfoSeries{figures/pgfplots/data/fashion_mnist_budget_l4_m6.csv}{edmDepthTwo}{square*}{$6$ models}
\edmInfoSeries{figures/pgfplots/data/fashion_mnist_budget_l4_m8.csv}{edmDepthThree}{triangle*}{$8$ models}
\edmInfoSeries{figures/pgfplots/data/fashion_mnist_budget_l4_m10.csv}{edmDepthFour}{diamond*}{$10$ models}
\end{axis}
\end{tikzpicture}
    \end{minipage}
    \caption{\small \textbf{The estimate depends weakly on the number of
        trained models.} Fashion-MNIST at depth \(L=4\), from four to ten
    trained models per \((P,N)\), has nearly identical model information across different scales (Appendix~\ref{sub:appendix_bank_sensitivity}).}
    \label{fig:fmnist_budget}
    \vspace{-1em}
\end{wrapfigure}
To show this, we trained ten models per \((P, N)\) on Fashion-MNIST
(to save computational resources; see
    Appendix~\ref{sub:appendix_bank_sensitivity} for further details, including
the experimental setup), and we compute the estimate
from four, six, eight, and ten of them, with half for the mixture score and
half for evaluation (Figure~\ref{fig:fmnist_budget}). The four curves are
almost the same. The scaling law fitted to the estimates from four
models is close to the scaling law fitted to the estimates from ten
(Appendix~\ref{sub:appendix_bank_sensitivity}). Thus four trained models per
\((P,N)\) are enough to fit the scaling law, so we use the same number for
all other experiments. A priori, this result should be
\textit{extremely surprising}. Model
informations are mutual informations, which are both computationally
and statistically
intractable to estimate, especially in high dimensions
\citep{mcallester2020formal}. We provide no formal explanation for
the accuracy of these estimations. However, our understanding is that
our score representation (Proposition
\ref{prop:model_information_rate_score_based}) reduces the problem,
therefore pushing the computational and statistical difficulty of
estimating this mutual information, into training a small number of
neural networks. This latter problem is itself computationally and
statistically intractable in the worst case, yet very tractable in
the average case (i.e., all of modern deep learning). This is one
possible reason why this particular mutual information is
surprisingly cheap to estimate.

\subsection{Model Information vs Memorization and Sample Quality}
\label{sub:information_vs_memorization}

\begin{figure}[t!]
    \centering
    \input{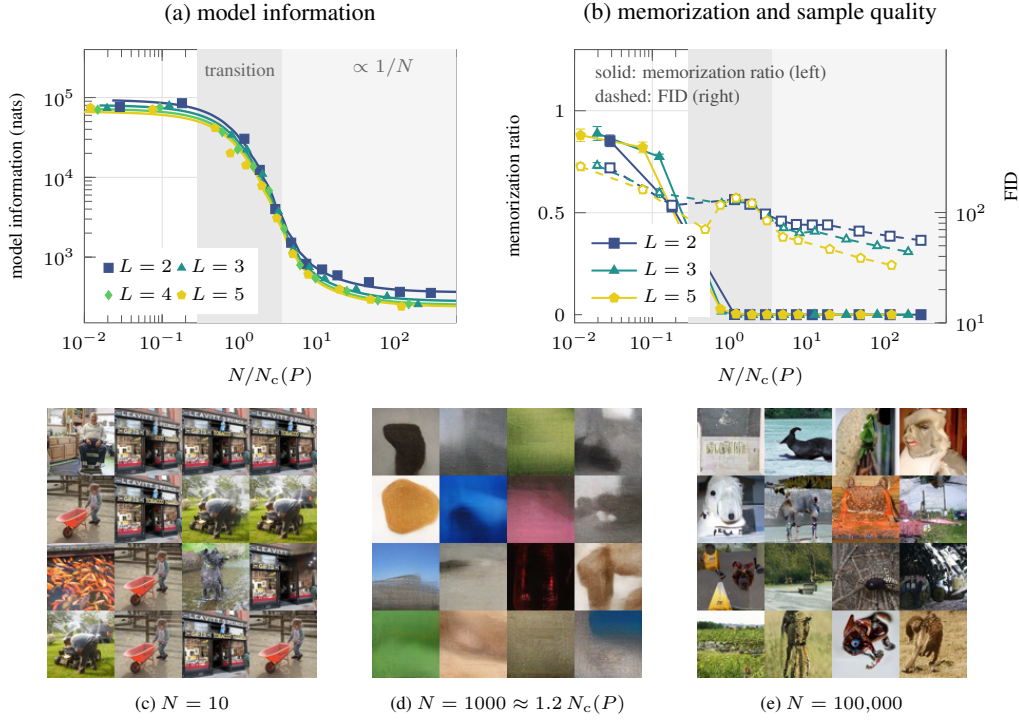}
    \caption{\small \textbf{Memorization occurs below $N=N_{\mathrm c}(P)$,
        while sample quality improves in the $N^{-1}$ regime.} The sample size
        is normalized by the critical training-set size $N_{\mathrm c}(P)$ of
        \eqref{eq:linear_gaussian_information_law}. The dark band,
        $N/N_{\mathrm c}(P)\in[e^{-2\sigma},e^{2\sigma}]$, contains the sizes
        at which most groups stop memorizing; the lighter band is the
        $N^{-1}$ regime. Below, we show random samples from models at depth
        $L=5$ at three training-set sizes $N$. At (d) the model does not copy
        its training images, and its samples are new but blurred. At (e) the
    samples are more realistic, and FID is still decreasing in \(N\).}
    \label{fig:imagenet_regimes}
    \vspace{-1.5em}
\end{figure}

We compare model information with a nearest-neighbor memorization
test \citep{yoon2023diffusion}, in which a sample is memorized if the
ratio between its distances to its first and second nearest neighbors in the
training set is less than \(1/3\), and with the FID
\citep{heusel2017gans,parmar2022aliased}. In
Figure~\ref{fig:imagenet_regimes} we plot both against \(N/N_{\mathrm
c}(P)\). We observe that memorization stops near \(N_{\mathrm
c}(P)\). At every model size, the memorization ratio is near one for
small \(N\), is one half near
\(N_{\mathrm c}(P)/4\), and is zero just above \(N_{\mathrm c}(P)\). For small
\(N\) the samples are copies of training images
(Figure~\ref{fig:imagenet_regimes}(c)), and in the \(N^{-1}\) regime they are
new images (Figure~\ref{fig:imagenet_regimes}(e)). FID is not monotone in
\(N\). While the models memorize, FID decreases, since a larger set of
memorized images covers more of the data distribution. Just after
memorization stops, FID increases to a local maximum. At these \(N\) the
models do not copy their training images, the samples are blurred
(Figure~\ref{fig:imagenet_regimes}(d)), and FID is high. In the
linear-Gaussian model this is
the point \(N=d\), where the smallest eigenvalue of the empirical covariance is
zero. In the \(N^{-1}\) regime, FID decreases with \(N\).

\begin{figure}[t!]
    \centering
    \input{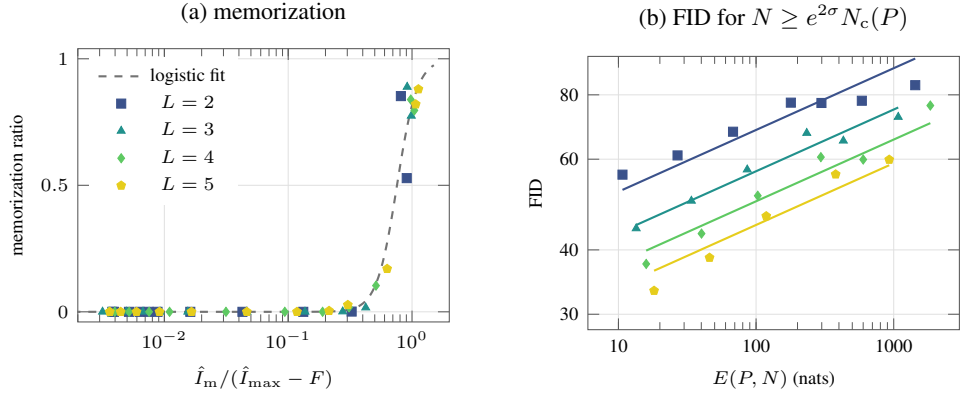}

\providecommand{\corrPoints}[5]{%
  \addplot[color=#2, mark=#3, mark size=1.7pt, only marks,
    mark options={solid, fill=#2, draw=#2}]
    table[x=#4,y=#5,col sep=comma] {#1};
}
\providecommand{\corrLine}[2]{%
  \addplot[color=#2, line width=0.8pt, no marks, forget plot]
    table[x=exc,y=fid,col sep=comma] {#1};
}

\noindent\makebox[\linewidth][c]{%
\hspace*{-0.02\linewidth}%
\begin{minipage}[t]{0.465\linewidth}
\centering
\begin{tikzpicture}
\begin{axis}[
width=\linewidth, height=0.80\linewidth,
xmode=log,
xmin=2e-3, xmax=2, ymin=-0.04, ymax=1.04,
xlabel={$\hat I_\model/(\hat I_{\max}-F)$},
ylabel={memorization ratio},
title={(a) memorization},
ytick={0,0.5,1},
grid=major, grid style={line width=0.1pt, draw=gray!24},
tick label style={font=\scriptsize},
label style={font=\scriptsize},
title style={font=\small},
legend style={font=\scriptsize, draw=none, fill=white,
at={(0.03,0.97)}, anchor=north west, legend columns=1},
axis line style={black!65},
tick style={black!65}
]
\addplot[black!55, line width=0.8pt, no marks, dashed]
  table[x=rel,y=mem,col sep=comma] {figures/pgfplots/data/imagenet64_corr_memcurve.csv};
\addlegendentry{logistic fit}
\corrPoints{figures/pgfplots/data/imagenet64_corr_l2.csv}{edmDepthTwo}{square*}{rel}{mem}
\addlegendentry{$L=2$}
\corrPoints{figures/pgfplots/data/imagenet64_corr_l3.csv}{edmDepthThree}{triangle*}{rel}{mem}
\addlegendentry{$L=3$}
\corrPoints{figures/pgfplots/data/imagenet64_corr_l4.csv}{edmDepthFour}{diamond*}{rel}{mem}
\addlegendentry{$L=4$}
\corrPoints{figures/pgfplots/data/imagenet64_corr_l5.csv}{edmDepthFive}{pentagon*}{rel}{mem}
\addlegendentry{$L=5$}
\end{axis}
\end{tikzpicture}
\end{minipage}
\hspace{0.01\linewidth}
\begin{minipage}[t]{0.465\linewidth}
\centering
\begin{tikzpicture}
\begin{axis}[
width=\linewidth, height=0.80\linewidth,
xmode=log, ymode=log,
xmin=6, xmax=3000, ymin=28,
xlabel={$E(P,N)$ (nats)}, ymax=95,
ylabel={FID},
title={(b) FID for $N\ge e^{2\sigma}N_{\mathrm c}(P)$},
ytick={30,40,60,80}, yticklabels={30,40,60,80},
xtick={10,100,1000}, xticklabels={10,100,1000},
grid=major, grid style={line width=0.1pt, draw=gray!24},
tick label style={font=\scriptsize},
label style={font=\scriptsize},
title style={font=\small},
axis line style={black!65},
tick style={black!65}
]
\corrLine{figures/pgfplots/data/imagenet64_corr_fidline_l2.csv}{edmDepthTwo}
\corrPoints{figures/pgfplots/data/imagenet64_corr_regime_l2.csv}{edmDepthTwo}{square*}{exc}{fid}
\corrLine{figures/pgfplots/data/imagenet64_corr_fidline_l3.csv}{edmDepthThree}
\corrPoints{figures/pgfplots/data/imagenet64_corr_regime_l3.csv}{edmDepthThree}{triangle*}{exc}{fid}
\corrLine{figures/pgfplots/data/imagenet64_corr_fidline_l4.csv}{edmDepthFour}
\corrPoints{figures/pgfplots/data/imagenet64_corr_regime_l4.csv}{edmDepthFour}{diamond*}{exc}{fid}
\corrLine{figures/pgfplots/data/imagenet64_corr_fidline_l5.csv}{edmDepthFive}
\corrPoints{figures/pgfplots/data/imagenet64_corr_regime_l5.csv}{edmDepthFive}{pentagon*}{exc}{fid}
\end{axis}
\end{tikzpicture}
\end{minipage}%
\hspace*{0.02\linewidth}%
}
    \vspace{-1.5em}
    \caption{\small \textbf{Model information is correlated with memorization
        and with FID.} \textbf{(a)} The memorization ratio at all $(P,N)$
        pairs against model information divided by
        $\hat I_{\max}(P)-F(P)=\sup_N E(P,N)$, computed from the fitted scaling law
        \eqref{eq:linear_gaussian_information_law}, with a logistic fit.
        \textbf{(b)} FID against $E(P,N)$ of the fitted scaling law in the
        $N^{-1}$ regime. The
        lines are one linear fit of $\log\mathrm{FID}$ on $\log E$ and
    $\log P$, evaluated at each model size.}
    \label{fig:imagenet_correlation}
\end{figure}

Model information is strongly correlated with memorization for small \(N\) and
with FID for large \(N\) (Figure~\ref{fig:imagenet_correlation}). In
particular, all observations lie on a single increasing curve of
memorization ratio versus model information (as a fraction of its
maximum). On the other hand, in the \(N^{-1}\) regime no sample is
memorized, and FID decreases as model information decreases. In this
regime the measured model information is close to \(F(P)\), while FID is
still decreasing. Therefore, we use the term \(E(P,N)\) of the fitted scaling
law \eqref{eq:linear_gaussian_information_law}. At each model size,
\(\log\mathrm{FID}\) is linear in \(\log E\), and a single linear regression
of \(\log\mathrm{FID}\) on \(\log E\) and \(\log P\) has \(R^2=0.95\) across
all of our observations.

\subsection{Predicting Memorization and Sample Quality for a New Model}
\label{sub:extrapolation}

With the scaling law and the two fits of
Figure~\ref{fig:imagenet_correlation}, we can predict memorization and FID
before we train a model: we predict model information with the scaling law,
and then predict memorization and FID from model information. We test this
on a new pair, depth \(L=5\) and \(N=250{,}000\), which is outside the range
of our experiments in both \(P\) and \(N\). We fit the scaling law and both
regressions only on the models at \(L\in\{2,3,4\}\). Thus the critical
training-set size at \(L=5\), and with it the range of \(N\) in which we
predict FID, is also a prediction. We then train four models in the
new setting. In this extrapolation, the empirically estimated model
information is within \(7\%\) of the prediction, no sample is
memorized, as predicted, and the FID is within \(15\%\) of the
prediction. In Figure~\ref{fig:imagenet_extrapolation} we compare the
predictions with the measurements at all values of \(N\) for the new value of
\(P\) corresponding to \(L=5\). We observe that the predicted
memorization ratio is
close to the measured memorization ratio, and likewise for predicted
and measured FID, for all values of \(N\).

\begin{figure}[t!]
    \centering
    \input{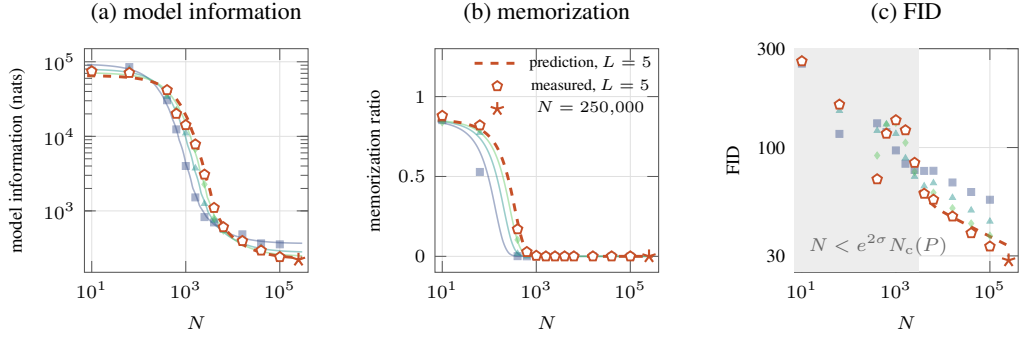}

\providecommand{\exContext}[5]{%
  \addplot[color=#2, mark=#3, mark size=1.3pt, only marks, forget plot, opacity=0.45,
    mark options={solid, fill=#2, draw=#2}]
    table[x=n,y=#5,col sep=comma] {#1};
  \addplot[color=#2, line width=0.6pt, no marks, forget plot, opacity=0.45]
    table[x=n,y=#5,col sep=comma] {#4};
}

\providecommand{\exHeld}[3]{%
  \addplot[color=edmSetWarm, line width=1.1pt, no marks, dashed]
    table[x=n,y=#3,col sep=comma] {#1};
  \addplot[color=edmSetWarm, mark=pentagon*, mark size=1.9pt, only marks,
    mark options={solid, fill=white, draw=edmSetWarm, line width=0.7pt}]
    table[x=n,y=#3,col sep=comma] {#2};
}

\providecommand{\exBeyond}[1]{%
  \addplot[color=edmSetWarm, mark=star, mark size=2.6pt, only marks,
    mark options={solid, draw=edmSetWarm, line width=0.8pt}]
    table[x=n,y=#1,col sep=comma] {figures/pgfplots/data/imagenet64_extrap_250k.csv};
}

\noindent\makebox[\linewidth][c]{%
\begin{minipage}[t]{0.325\linewidth}
\centering
\begin{tikzpicture}
\begin{axis}[
  width=\linewidth,
  height=1.0\linewidth,
  xmode=log, xmin=7, xmax=4e5,
  xlabel={$N$}, xtick={1e1,1e3,1e5},
  grid=major,
  grid style={line width=0.1pt, draw=gray!24},
  tick label style={font=\scriptsize},
  label style={font=\scriptsize},
  title style={font=\small},
  axis line style={black!65},
  tick style={black!65},
  ymode=log, ymin=150, ymax=1.5e5,
  ytick={1e3,1e4,1e5},
  ylabel={model information (nats)},
  title={(a) model information}
]
\exContext{figures/pgfplots/data/imagenet64_extrap_obs_l2.csv}{edmDepthTwo}{square*}{figures/pgfplots/data/imagenet64_extrap_law_l2.csv}{info}
\exContext{figures/pgfplots/data/imagenet64_extrap_obs_l3.csv}{edmDepthThree}{triangle*}{figures/pgfplots/data/imagenet64_extrap_law_l3.csv}{info}
\exContext{figures/pgfplots/data/imagenet64_extrap_obs_l4.csv}{edmDepthFour}{diamond*}{figures/pgfplots/data/imagenet64_extrap_law_l4.csv}{info}
\exHeld{figures/pgfplots/data/imagenet64_extrap_law_l5.csv}{figures/pgfplots/data/imagenet64_extrap_obs_l5.csv}{info}
\exBeyond{info}
\end{axis}
\end{tikzpicture}
\end{minipage}\hfill
\begin{minipage}[t]{0.325\linewidth}
\centering
\begin{tikzpicture}
\begin{axis}[
  width=\linewidth,
  height=1.0\linewidth,
  xmode=log, xmin=7, xmax=4e5,
  xlabel={$N$}, xtick={1e1,1e3,1e5},
  grid=major,
  grid style={line width=0.1pt, draw=gray!24},
  tick label style={font=\scriptsize},
  label style={font=\scriptsize},
  title style={font=\small},
  axis line style={black!65},
  tick style={black!65},
  ymin=-0.1, ymax=1.30,
  ytick={0,0.5,1},
  ylabel={memorization ratio},
  title={(b) memorization},
  legend style={
    at={(1,1)},
    anchor=north east,
    font=\tiny,
    draw=none,
    fill=none,
    legend columns=1,
    inner xsep=0.5pt,
    inner ysep=0.5pt,
    row sep=-1pt,
    column sep=1pt,
    /tikz/every even column/.append style={column sep=1pt}
  }
]
\exContext{figures/pgfplots/data/imagenet64_extrap_obs_l2.csv}{edmDepthTwo}{square*}{figures/pgfplots/data/imagenet64_extrap_law_l2.csv}{mem}
\exContext{figures/pgfplots/data/imagenet64_extrap_obs_l3.csv}{edmDepthThree}{triangle*}{figures/pgfplots/data/imagenet64_extrap_law_l3.csv}{mem}
\exContext{figures/pgfplots/data/imagenet64_extrap_obs_l4.csv}{edmDepthFour}{diamond*}{figures/pgfplots/data/imagenet64_extrap_law_l4.csv}{mem}
\exHeld{figures/pgfplots/data/imagenet64_extrap_law_l5.csv}{figures/pgfplots/data/imagenet64_extrap_obs_l5.csv}{mem}
\exBeyond{mem}
\legend{{prediction, $L=5$},{measured, $L=5$},{$N=250{,}000$}}
\end{axis}
\end{tikzpicture}
\end{minipage}
\hfill
\begin{minipage}[t]{0.325\linewidth}
\centering
\begin{tikzpicture}
\begin{axis}[
  width=\linewidth,
  height=1.0\linewidth,
  xmode=log, xmin=7, xmax=4e5,
  xlabel={$N$}, xtick={1e1,1e3,1e5},
  grid=major,
  grid style={line width=0.1pt, draw=gray!24},
  tick label style={font=\scriptsize},
  label style={font=\scriptsize},
  title style={font=\small},
  axis line style={black!65},
  tick style={black!65},
  ymode=log, ymin=25, ymax=300,
  ytick={30,100,300},
  yticklabels={30,100,300},
  ylabel={FID},
  title={(c) FID}
]
\path[fill=black!7] (axis cs:7,25) rectangle (axis cs:3130,300);

\node[
  font=\scriptsize,
  black!60,
  anchor=south west,
  align=left
] at (axis cs:9,27)
  {$N<e^{2\sigma}N_{\mathrm c}(P)$};

\addplot[
  color=edmDepthTwo,
  mark=square*,
  mark size=1.3pt,
  only marks,
  forget plot,
  opacity=0.45,
  mark options={solid, fill=edmDepthTwo}
] table[x=n,y=fid,col sep=comma]
  {figures/pgfplots/data/imagenet64_extrap_obs_l2.csv};

\addplot[
  color=edmDepthThree,
  mark=triangle*,
  mark size=1.3pt,
  only marks,
  forget plot,
  opacity=0.45,
  mark options={solid, fill=edmDepthThree}
] table[x=n,y=fid,col sep=comma]
  {figures/pgfplots/data/imagenet64_extrap_obs_l3.csv};

\addplot[
  color=edmDepthFour,
  mark=diamond*,
  mark size=1.3pt,
  only marks,
  forget plot,
  opacity=0.45,
  mark options={solid, fill=edmDepthFour}
] table[x=n,y=fid,col sep=comma]
  {figures/pgfplots/data/imagenet64_extrap_obs_l4.csv};

\exHeld{figures/pgfplots/data/imagenet64_extrap_fid_l5.csv}{figures/pgfplots/data/imagenet64_extrap_obs_l5.csv}{fid}
\exBeyond{fid}
\end{axis}
\end{tikzpicture}
\end{minipage}%
}
    \vspace{-1.5em}
    \caption{\small \textbf{Predicting model information, memorization, and
        FID at a new model size.} All fits use only the models at depths
        $L\in\{2,3,4\}$ (faint markers and curves). Open markers are the
        measurements at $L=5$, the dashed curves are the predictions for $L=5$,
        and the star is the new pair $(L=5,N=250{,}000)$, outside the
        range of our experiments. \textbf{(a)} Model information and the
        fitted scaling law
        \eqref{eq:linear_gaussian_information_law}. \textbf{(b)} Memorization
        ratio. \textbf{(c)} FID, predicted in the $N^{-1}$ regime; the shaded
    region is below the predicted start of this regime for $L=5$.}
    \label{fig:imagenet_extrapolation}
\end{figure}

\section{Related work}

\paragraph{Memorization and generalization.}
Empirical work has proposed sample-level diagnostics for replication and memorization \citep{somepalli2023diffusion,wen2024detecting,ross2024geometric}, while theoretical studies have investigated the mechanisms and transitions underlying memorization and generalization \citep{zhang2024emergence,bonnaire2025why,kamb2024analytic,Biroli_2024,pham2025memorization,niedoba2024mechanistic,buchanan2025edge}. \citet{kadkhodaie2023generalization} related generalization to the convergence of models trained on independent datasets. Our model information quantifies the variability of their learned distributions, connecting naturally to their bias--variance perspective. More recently, \citet{hunt2026exact} used information theory to characterize memorization--generalization transitions in Bayesian diffusion models. Most closely related to our work, \citet{yi2023generalization} proposed mutual information between training datasets and generated samples as a measure of generalization, which coincides with model information under deterministic training. However, they did not develop a practical estimator for this quantity or characterize its behavior. Our work addresses these questions through an estimator that uses a few independently trained models, an analytically tractable model, and empirical scaling laws.

\paragraph{Information-theoretic approaches to diffusion models.}
Diffusion models can be viewed as noisy communication channels through which information is progressively lost. The evolution of mutual information under Gaussian noising can be expressed through denoising and score-based identities \citep{guo2005mutual,kong2023information,franzese2024minde}, enabling practical mutual-information estimation using diffusion models. This perspective has been used to study information about clean samples \citep{kong2023information,stancevic2026entropic}, latent abstractions and semantic classes \citep{franzese2025latent,handke2026the}, and relationships between images and text \citep{kong2023interpretable}. Building on this line of work, we extend score-based mutual-information estimation to ensembles of trained generative models by approximating the ensemble with finitely many trained models, and demonstrate stable estimation using only a few independently trained networks.

\paragraph{Scaling laws.}
Scaling laws have played a central role in the development of large language models, providing quantitative predictions of how performance improves with increasing dataset size, model capacity, and compute \citep{kaplan2020scaling,hoffmann2022trainingcomputeoptimallargelanguagemodels}. More recently, similar relationships have been established for diffusion models, showing that both training loss and generative quality exhibit predictable scaling behavior \citep{li2024scalability,liang2026scaling,chickering2026abra}. Our work establishes scaling laws for model information across memorization and generalization regimes, enabling predictions of both memorization and generative quality.

\section{Conclusion}
In this work, we show that model information can be efficiently estimated and exhibits predictable scaling behavior across training-set sizes and model capacities. An analytically tractable model motivates a scaling law that captures the transition from memorization to generalization observed in neural networks. We show that this transition closely tracks the disappearance of memorization, while decreasing model information is associated with improvements in generative fidelity in the generalization regime. These relationships allow us to predict memorization and sample quality for previously unseen combinations of training-set size and model capacity. Our results thus provide a quantitative connection between dataset dependence, memorization, and generative fidelity, offering a new perspective on how generalization emerges in diffusion models.

Several questions remain open. A deeper theoretical understanding of the scaling law is needed. Understanding how its parameters, including the critical training-set size $N_{\mathrm c}(P)$ and the width $\sigma$ of the transition, emerge from network architecture and learning dynamics remains open. Another direction is to investigate how random initialization and stochastic optimization contribute to model information independently of training-set variability. Finally, empirical improvements to the estimation procedure such that it can be applied reliably or efficiently to pretrained models would also be a way to expand this methodology to frontier settings.

\bibliography{main.bib}

@article{sohldickstein2015deepunsupervisedlearningusing,
  title   = {Deep unsupervised learning using nonequilibrium thermodynamics},
  author  = {Sohl-Dickstein, Jascha and Weiss, Eric and Maheswaranathan, Niru and Ganguli, Surya},
  journal = {International Conference on Machine Learning},
  year    = {2015}
}

@article{ho2020denoisingdiffusionprobabilisticmodels,
  title   = {Denoising diffusion probabilistic models},
  author  = {Ho, Jonathan and Jain, Ajay and Abbeel, Pieter},
  journal = {Advances in neural information processing systems},
  volume  = {33},
  pages   = {6840--6851},
  year    = {2020}
}

@article{song2022denoisingdiffusionimplicitmodels,
  title   = {Denoising Diffusion Implicit Models},
  author  = {Jiaming Song and Chenlin Meng and Stefano Ermon},
  journal = {International COnference on Learning Representations},
  year    = {2022}
}

@article{anderson1982reverse,
  title     = {Reverse-time diffusion equation models},
  author    = {Anderson, Brian DO},
  journal   = {Stochastic Processes and their Applications},
  volume    = {12},
  number    = {3},
  pages     = {313--326},
  year      = {1982},
  publisher = {Elsevier}
}

@book{sarkka2019applied,
  title     = {Applied stochastic differential equations},
  author    = {S{\"a}rkk{\"a}, Simo and Solin, Arno},
  volume    = {10},
  year      = {2019},
  publisher = {Cambridge University Press}
}

@article{hoffmann2022trainingcomputeoptimallargelanguagemodels,
  title   = {Training Compute-Optimal Large Language Models},
  author  = {Hoffmann, Jordan and Borgeaud, Sebastian and Mensch, Arthur and Buchatskaya, Elena and Cai, Trevor and Rutherford, Eliza and Casas, Diego de Las and Hendricks, Lisa Anne and Welbl, Johannes and Clark, Aidan and Hennigan, Tom and Noland, Eric and Millican, Katie and van den Driessche, George and Damoc, Bogdan and Guy, Aurelia and Osindero, Simon and Simonyan, Karen and Elsen, Erich and Rae, Jack W. and Vinyals, Oriol and Sifre, Laurent},
  journal = {arXiv preprint arXiv:2203.15556},
  year    = {2022}
}

@article{song2021scorebasedgenerativemodelingstochastic,
  title   = {Score-Based Generative Modeling through Stochastic Differential Equations},
  author  = {Yang Song and Jascha Sohl-Dickstein and Diederik P. Kingma and Abhishek Kumar and Stefano Ermon and Ben Poole},
  journal = {International Conference on Learning Representations},
  year    = {2021}
}

@article{rombach2022highresolutionimagesynthesislatent,
  title   = {High-resolution image synthesis with latent diffusion models},
  author  = {Rombach, Robin and Blattmann, Andreas and Lorenz, Dominik and Esser, Patrick and Ommer, Bj{\"o}rn},
  journal = {Proceedings of the IEEE/CVF conference on computer vision and Pattern Recognition},
  year    = {2022}
}

@article{ho2022videodiffusionmodels,
  title   = {Video diffusion models},
  author  = {Ho, Jonathan and Salimans, Tim and Gritsenko, Alexey and Chan, William and Norouzi, Mohammad and Fleet, David J},
  journal = {Advances in Neural Information Processing Systems},
  year    = {2022}
}

@article{chen2020wavegradestimatinggradientswaveform,
  title   = {WaveGrad: Estimating Gradients for Waveform Generation},
  author  = {Nanxin Chen and Yu Zhang and Heiga Zen and Ron J. Weiss and Mohammad Norouzi and William Chan},
  journal = {International Conference on Learning Representations},
  year    = {2021}
}

@article{Biroli_2024,
  title   = {Dynamical regimes of diffusion models},
  volume  = {15},
  doi     = {10.1038/s41467-024-54281-3},
  number  = {1},
  journal = {Nature Communications},
  author  = {Biroli, Giulio and Bonnaire, Tony and de Bortoli, Valentin and Mézard, Marc},
  year    = {2024},
  month   = nov
}

@article{karras2022elucidating,
  title   = {Elucidating the design space of diffusion-based generative models},
  author  = {Karras, Tero and Aittala, Miika and Aila, Timo and Laine, Samuli},
  journal = {Advances in neural information processing systems},
  volume  = {35},
  pages   = {26565--26577},
  year    = {2022}
}

@inproceedings{
kong2023information,
title={Information-Theoretic Diffusion},
author={Xianghao Kong and Rob Brekelmans and Greg Ver Steeg},
booktitle={The Eleventh International Conference on Learning Representations },
year={2023},
url={https://openreview.net/forum?id=UvmDCdSPDOW}
}

@article{kong2023interpretable,
  title   = {Interpretable diffusion via information decomposition},
  author  = {Kong, Xianghao and Liu, Ollie and Li, Han and Yogatama, Dani and Steeg, Greg Ver},
  journal = {arXiv preprint arXiv:2310.07972},
  year    = {2023}
}

@article{franzese2025latent,
  title     = {Latent Abstractions in Generative Diffusion Models},
  author    = {Franzese, Giulio and Martini, Mattia and Corallo, Giulio and Papotti, Paolo and Michiardi, Pietro},
  journal   = {Entropy},
  volume    = {27},
  number    = {4},
  pages     = {371},
  year      = {2025},
  publisher = {MDPI}
}

@article{guo2005mutual,
  title     = {Mutual information and minimum mean-square error in Gaussian channels},
  author    = {Guo, Dongning and Shamai, Shlomo and Verd{\'u}, Sergio},
  journal   = {IEEE transactions on information theory},
  volume    = {51},
  number    = {4},
  pages     = {1261--1282},
  year      = {2005},
  publisher = {IEEE}
}

@article{vincent2011connection,
  title     = {A connection between score matching and denoising autoencoders},
  author    = {Vincent, Pascal},
  journal   = {Neural computation},
  volume    = {23},
  number    = {7},
  pages     = {1661--1674},
  year      = {2011},
  publisher = {MIT Press}
}

@article{heusel2017gans,
  title   = {Gans trained by a two time-scale update rule converge to a local nash equilibrium},
  author  = {Heusel, Martin and Ramsauer, Hubert and Unterthiner, Thomas and Nessler, Bernhard and Hochreiter, Sepp},
  journal = {Advances in neural information processing systems},
  volume  = {30},
  year    = {2017}
}

@article{krizhevsky2009learning,
  title     = {Learning multiple layers of features from tiny images},
  author    = {Krizhevsky, Alex and Hinton, Geoffrey and others},
  year      = {2009},
  journal   = {University of Toronto},
  publisher = {Toronto, ON, Canada}
}

@inproceedings{zhang2024emergence,
  title     = {The emergence of reproducibility and consistency in diffusion models},
  author    = {Zhang, Huijie and Zhou, Jinfan and Lu, Yifu and Guo, Minzhe and Wang, Peng and Shen, Liyue and Qu, Qing},
  booktitle = {Forty-first International Conference on Machine Learning},
  year      = {2024}
}

@article{kadkhodaie2023generalization,
  title   = {Generalization in diffusion models arises from geometry-adaptive harmonic representations},
  author  = {Kadkhodaie, Zahra and Guth, Florentin and Simoncelli, Eero P and Mallat, St{\'e}phane},
  journal = {arXiv preprint arXiv:2310.02557},
  year    = {2023}
}

@article{buchanan2025edge,
  title   = {On the edge of memorization in diffusion models},
  author  = {Buchanan, Sam and Pai, Druv and Ma, Yi and De Bortoli, Valentin},
  journal = {arXiv preprint arXiv:2508.17689},
  year    = {2025}
}

@article{gu2023memorization,
  title   = {On memorization in diffusion models},
  author  = {Gu, Xiangming and Du, Chao and Pang, Tianyu and Li, Chongxuan and Lin, Min and Wang, Ye},
  journal = {arXiv preprint arXiv:2310.02664},
  year    = {2023}
}

@inproceedings{wen2024detecting,
  title     = {Detecting, explaining, and mitigating memorization in diffusion models},
  author    = {Wen, Yuxin and Liu, Yuchen and Chen, Chen and Lyu, Lingjuan},
  booktitle = {The Twelfth International Conference on Learning Representations},
  year      = {2024}
}

@article{bertrand2025closed,
  title   = {On the closed-form of flow matching: Generalization does not arise from target stochasticity},
  author  = {Bertrand, Quentin and Gagneux, Anne and Massias, Mathurin and Emonet, R{\'e}mi},
  journal = {arXiv preprint arXiv:2506.03719},
  year    = {2025}
}

@inproceedings{deng2009imagenet,
  title        = {Imagenet: A large-scale hierarchical image database},
  author       = {Deng, Jia and Dong, Wei and Socher, Richard and Li, Li-Jia and Li, Kai and Fei-Fei, Li},
  booktitle    = {2009 IEEE conference on computer vision and pattern recognition},
  pages        = {248--255},
  year         = {2009},
  organization = {Ieee}
}

@inproceedings{yoon2023diffusion,
  title     = {Diffusion probabilistic models generalize when they fail to memorize},
  author    = {Yoon, TaeHo and Choi, Joo Young and Kwon, Sehyun and Ryu, Ernest K},
  booktitle = {ICML 2023 workshop on structured probabilistic inference $\{$$\backslash$\&$\}$ generative modeling},
  year      = {2023}
}

@inproceedings{carlini2023extracting,
  title     = {Extracting Training Data from Diffusion Models},
  author    = {Carlini, Nicholas and Hayes, Jamie and Nasr, Milad and Jagielski, Matthew and Sehwag, Vikash and Tram{\`e}r, Florian and Balle, Borja and Ippolito, Daphne and Wallace, Eric},
  booktitle = {32nd USENIX Security Symposium},
  pages     = {5253--5270},
  year      = {2023},
  publisher = {USENIX Association}
}

@inproceedings{somepalli2023diffusion,
  title     = {Diffusion Art or Digital Forgery? Investigating Data Replication in Diffusion Models},
  author    = {Somepalli, Gowthami and Singla, Vasu and Goldblum, Micah and Geiping, Jonas and Goldstein, Tom},
  booktitle = {Proceedings of the IEEE/CVF Conference on Computer Vision and Pattern Recognition},
  pages     = {6048--6058},
  year      = {2023}
}

@article{kamb2024analytic,
  title   = {An Analytic Theory of Creativity in Convolutional Diffusion Models},
  author  = {Kamb, Mason and Ganguli, Surya},
  journal = {arXiv preprint arXiv:2412.20292},
  year    = {2024}
}

@article{niedoba2024mechanistic,
  title   = {Towards a Mechanistic Explanation of Diffusion Model Generalization},
  author  = {Niedoba, Matthew and Zwartsenberg, Berend and Murphy, Kevin and Wood, Frank},
  journal = {arXiv preprint arXiv:2411.19339},
  year    = {2024}
}

@article{kaplan2020scaling,
  title   = {Scaling Laws for Neural Language Models},
  author  = {Kaplan, Jared and McCandlish, Sam and Henighan, Tom and Brown, Tom B. and Chess, Benjamin and Child, Rewon and Gray, Scott and Radford, Alec and Wu, Jeffrey and Amodei, Dario},
  journal = {arXiv preprint arXiv:2001.08361},
  year    = {2020}
}

@article{pham2025memorization,
  title   = {Memorization to generalization: Emergence of diffusion models from associative memory},
  author  = {Pham, Bao and Raya, Gabriel and Negri, Matteo and Zaki, Mohammed J and Ambrogioni, Luca and Krotov, Dmitry},
  journal = {arXiv preprint arXiv:2505.21777},
  year    = {2025}
}

@article{stancevic2026entropic,
  title   = {Entropic time schedulers for generative diffusion models},
  author  = {Stancevic, Dejan and Handke, Florian and Ambrogioni, Luca},
  journal = {Advances in Neural Information Processing Systems},
  volume  = {38},
  pages   = {44222--44252},
  year    = {2026}
}

@article{watson2023de,
  title   = {De novo design of protein structure and function with RFdiffusion},
  author  = {Watson, Joseph L. and Juergens, David and Bennett, Nathaniel R. and others},
  journal = {Nature},
  volume  = {620},
  pages   = {1089--1100},
  year    = {2023},
  doi     = {10.1038/s41586-023-06415-8}
}

@inproceedings{ross2024geometric,
  title     = {A Geometric Framework for Understanding Memorization in Generative Models},
  author    = {Ross, Brendan Leigh and Kamkari, Hamidreza and Wu, Tongzi and Hosseinzadeh, Rasa and Liu, Zhaoyan and Stein, George and Cresswell, Jesse C. and Loaiza-Ganem, Gabriel},
  booktitle = {The Thirteenth International Conference on Learning Representations},
  year      = {2025}
}

@inproceedings{
bonnaire2025why,
title={Why Diffusion Models Don{\textquoteright}t Memorize:  The Role of Implicit Dynamical Regularization in Training},
author={Tony Bonnaire and Rapha{\"e}l Urfin and Giulio Biroli and Marc Mezard},
booktitle={The Thirty-ninth Annual Conference on Neural Information Processing Systems},
year={2025},
url={https://openreview.net/forum?id=BSZqpqgqM0}
}

@inproceedings{
handke2026the,
title={The Entropic Signature of Class Speciation in Diffusion Models},
author={Florian Handke and Dejan Stancevic and Felix Koulischer and Thomas Demeester and Luca Ambrogioni},
booktitle={Forty-third International Conference on Machine Learning},
year={2026},
url={https://openreview.net/forum?id=rp02mO31RL}
}

@article{hunt2026exact,
  title   = {An Exact Information Theory of Generalization Phase Transitions in Bayesian Diffusion Models},
  author  = {Hunt, Henry and Kamb, Mason and Ganguli, Surya},
  journal = {arXiv preprint arXiv:2607.08041},
  year    = {2026}
}

@inproceedings{mcallester2020formal,
  title        = {Formal limitations on the measurement of mutual information},
  author       = {McAllester, David and Stratos, Karl},
  booktitle    = {International Conference on Artificial Intelligence and Statistics},
  pages        = {875--884},
  year         = {2020},
  organization = {PMLR}
}

@article{lipman2022flow,
  title   = {Flow matching for generative modeling},
  author  = {Lipman, Yaron and Chen, Ricky TQ and Ben-Hamu, Heli and Nickel, Maximilian and Le, Matt},
  journal = {arXiv preprint arXiv:2210.02747},
  year    = {2022}
}

@inproceedings{parmar2022aliased,
  title     = {On Aliased Resizing and Surprising Subtleties in {GAN} Evaluation},
  author    = {Parmar, Gaurav and Zhang, Richard and Zhu, Jun-Yan},
  booktitle = {CVPR},
  year      = {2022}
}

@book{bai2010spectral,
  title={Spectral Analysis of Large Dimensional Random Matrices},
  author={Bai, Zhidong and Silverstein, Jack W.},
  edition={2},
  series={Springer Series in Statistics},
  publisher={Springer},
  year={2010}
}

@book{cochran1977sampling,
  title={Sampling Techniques},
  author={Cochran, William G.},
  edition={3},
  publisher={John Wiley \& Sons},
  year={1977}
}

@article{yi2023generalization,
  title   = {On the Generalization of Diffusion Model},
  author  = {Yi, Mingyang and Sun, Jiacheng and Li, Zhenguo},
  journal = {arXiv preprint arXiv:2305.14712},
  year    = {2023},
  url     = {https://arxiv.org/abs/2305.14712}
}

@inproceedings{franzese2024minde,
  title     = {{MINDE}: Mutual Information Neural Diffusion Estimation},
  author    = {Franzese, Giulio and Bounoua, Mustapha and Michiardi, Pietro},
  booktitle = {International Conference on Learning Representations},
  year      = {2024},
  url       = {https://proceedings.iclr.cc/paper_files/paper/2024/hash/47f75e809409709c6d226ab5ca0c9703-Abstract-Conference.html}
}

@inproceedings{wang2026randommatrix,
  title     = {A Random Matrix Theory Perspective on the Consistency of Diffusion Models},
  author    = {Wang, Binxu and Zavatone-Veth, Jacob A. and Pehlevan, Cengiz},
  booktitle = {Proceedings of the 43rd International Conference on Machine Learning},
  series    = {Proceedings of Machine Learning Research},
  volume    = {306},
  pages     = {132077--132128},
  publisher = {PMLR},
  year      = {2026},
  url       = {https://proceedings.mlr.press/v306/wang26kf.html}
}

@article{dai2026outputs,
  title   = {Outputs of Generative Diffusion Models Are Often Unattributable},
  author  = {Dai, Zheng and Gifford, David K.},
  journal = {Nature Communications},
  volume  = {17},
  pages   = {6974},
  year    = {2026},
  doi     = {10.1038/s41467-026-75667-5},
  url     = {https://doi.org/10.1038/s41467-026-75667-5}
}

@article{marion2026understanding,
  title   = {Understanding Diffusion Models Requires Rethinking (Again) Generalization},
  author  = {Marion, Pierre and Wu, Yu-Han},
  journal = {arXiv preprint arXiv:2605.06077},
  year    = {2026},
  url     = {https://arxiv.org/abs/2605.06077}
}

@inproceedings{li2024scalability,
  title     = {On the Scalability of Diffusion-based Text-to-Image Generation},
  author    = {Li, Hao and Zou, Yang and Wang, Ying and Majumder, Orchid and Xie, Yusheng and Manmatha, R. and Swaminathan, Ashwin and Tu, Zhuowen and Ermon, Stefano and Soatto, Stefano},
  booktitle = {Proceedings of the IEEE/CVF Conference on Computer Vision and Pattern Recognition (CVPR)},
  pages     = {9400--9409},
  year      = {2024},
  doi       = {10.1109/CVPR52733.2024.00898},
  url       = {https://openaccess.thecvf.com/content/CVPR2024/html/Li_On_the_Scalability_of_Diffusion-based_Text-to-Image_Generation_CVPR_2024_paper.html}
}

@inproceedings{liang2026scaling,
  title     = {Scaling Laws for Diffusion Transformers},
  author    = {Liang, Zhengyang and He, Hao and Yang, Ceyuan and Dai, Bo},
  booktitle = {International Conference on Learning Representations (ICLR)},
  year      = {2026},
  url       = {https://proceedings.iclr.cc/paper_files/paper/2026/hash/880d8999c07a8efc9bbbeb0c38f50765-Abstract-Conference.html}
}

@article{chickering2026abra,
  title   = {{Abra}: Scaling Diffusion Image Training},
  author  = {Chickering, Kyle and Lin, Wei-An and Bhanded, Swayam and Saunders, Dan and Tripathi, Akshat and Song, Jiaming and Buch, Shyamal and Yan, Xinchen},
  journal = {arXiv preprint arXiv:2608.17286},
  year    = {2026},
  doi     = {10.48550/arXiv.2608.17286},
  url     = {https://arxiv.org/abs/2608.17286}
}

\appendix

\newpage

\section{Computing the Model Information}
\label{sec:model_information_computations}
In this section, we derive the score-space representation of the model
information rate and equivalent computational forms.

\subsection{Score-Space
    Representation of the Model Information
Rate}\label{sec:model_information_rate_score_space}

For \(t>0\) with \(\sigma_t>0\), the Gaussian smoothing in
\eqref{eq:forward_solution} makes the densities $p_t^{\psi}$ smooth and
strictly positive. We assume the standard integrability and decay conditions
needed to differentiate under the integral in \eqref{eq:noisy_marginal_law}
and to integrate by parts below; for example, these conditions hold if the
model laws have finite second moment.

\begin{proposition*}[Proposition \ref{prop:model_information_rate_score_based}]
    Under the preceding assumptions, the model information rate satisfies
    \begin{equation}
        \dot I_\model(t)
        =
        -\frac{g(t)^2}{2}
        \E_{\Phi\sim\Pi, \hat{X}_t\sim p_t^\Phi}
        \|
        \nabla \log p_t^\Phi(\hat{X}_t)
        -
        \nabla \log \bar{p}_t(\hat{X}_t)
        \|^2.
        \label{eq:appendix_model_information_rate_score_space}
    \end{equation}
\end{proposition*}

\begin{proof}
    We prove the claim by differentiating the mutual information
    between the model identity and the noised observation. Since
    $I_\model(t)=I(\Phi;\hat{X}_t)$, it is enough to compute
    $\frac{d}{dt}I_\model(t)$. By the mixture representation
    \eqref{eq:noisy_marginal_law},
    \begin{equation}
        I(\Phi;\hat{X}_t)
        =
        \int
        D_\mathrm{KL}(p_t^\phi \mid \bar{p}_t)\, \Pi(d\phi).
        \label{eq:appendix_mutual_information_kl}
    \end{equation}

    Fix $\phi$, and set
    \begin{equation}
        r_t^\phi(x):=\frac{p_t^\phi(x)}{\bar{p}_t(x)}.
    \end{equation}
    Both $p_t^\phi$ and $\bar p_t$ evolve under the same forward
    Fokker--Planck equation
    \begin{equation}
        \partial_t u_t
        =
        \mathcal L_t^*u_t,
        \qquad
        \mathcal L_t^*u
        =
        -\nabla\cdot(f(t)x\,u)
        +
        \frac{g(t)^2}{2}\Delta u,
        \label{eq:appendix_fokker_planck}
    \end{equation}
    whose adjoint on test functions is
    \begin{equation}
        \mathcal L_t\varphi
        =
        f(t)x\cdot\nabla \varphi
        +
        \frac{g(t)^2}{2}\Delta \varphi .
    \end{equation}
    Indeed, the claim for $\bar p_t$ follows from the linearity of
    the Fokker--Planck equation and the definition $\bar p_t=\int
    p_t^\phi\,\Pi(d\phi)$:
    \begin{equation}
        \partial_t\bar p_t
        =
        \int \partial_t p_t^\phi\,\Pi(d\phi)
        =
        \int \mathcal L_t^*p_t^\phi\,\Pi(d\phi)
        =
        \mathcal L_t^*\bar p_t.
    \end{equation}

    Define $D_t(\phi):=D_\mathrm{KL}(p_t^\phi \mid \bar{p}_t)=\int
    p_t^\phi\log r_t^\phi\,dx$. Differentiating and applying the
    product rule gives
    \begin{align}
        \frac{d}{dt}D_t(\phi)
        &=
        \int (\partial_t p_t^\phi)\log r_t^\phi\,dx
        +
        \int p_t^\phi\,\partial_t\log r_t^\phi\,dx \notag\\
        &=
        \int (\partial_t p_t^\phi)\log r_t^\phi\,dx
        +
        \int p_t^\phi
        \left(
            \frac{\partial_t p_t^\phi}{p_t^\phi}
            -
            \frac{\partial_t\bar p_t}{\bar p_t}
        \right)dx \notag\\
        &=
        \int (\partial_t p_t^\phi)\log r_t^\phi\,dx
        +
        \int \partial_t p_t^\phi\,dx
        -
        \int p_t^\phi\,\frac{\partial_t \bar{p}_t}{\bar{p}_t}\,dx.
    \end{align}
    Since $p_t^\phi$ is a probability density for every $t$, we have
    $\int p_t^\phi\,dx=1$, and therefore $\int\partial_t p_t^\phi\,dx=0$. Hence
    \begin{equation}
        \frac{d}{dt}D_t(\phi)
        =
        \int (\partial_t p_t^\phi)\log r_t^\phi\,dx
        -
        \int p_t^\phi\,\frac{\partial_t \bar{p}_t}{\bar{p}_t}\,dx.
    \end{equation}
    Substituting the Fokker--Planck equations for $p_t^\phi$ and
    $\bar p_t$, and using the adjoint relation, gives
    \begin{align}
        \frac{d}{dt}D_t(\phi)
        &=
        \int (\mathcal L_t^*p_t^\phi)\log r_t^\phi\,dx
        -
        \int r_t^\phi\,\mathcal L_t^*\bar p_t\,dx \notag\\
        &=
        \int p_t^\phi\,\mathcal L_t \log r_t^\phi\,dx
        -
        \int \bar{p}_t\,\mathcal L_t r_t^\phi\,dx.
        \label{eq:appendix_kl_derivative}
    \end{align}

    Using again $p_t^\phi=r_t^\phi\bar p_t$, we can rewrite
    \eqref{eq:appendix_kl_derivative} as
    \begin{equation}
        \frac{d}{dt}D_t(\phi)
        =
        \int \bar p_t
        \left(
            r_t^\phi\,\mathcal L_t\log r_t^\phi
            -
            \mathcal L_t r_t^\phi
        \right)dx.
    \end{equation}
    Now substitute the expression for $\mathcal L_t$. Suppressing the
    superscripts temporarily and writing $r=r_t^\phi$, the integrand becomes
    \begin{align}
        r\,\mathcal L_t\log r-\mathcal L_t r
        &=
        f(t)x\cdot
        \left(
            r\nabla\log r-\nabla r
        \right)
        +
        \frac{g(t)^2}{2}
        \left(
            r\Delta\log r-\Delta r
        \right).
    \end{align}
    The drift term vanishes pointwise. For the diffusion term, we use
    \begin{equation}
        \Delta\log r
        =
        \frac{\Delta r}{r}
        -
        \frac{\|\nabla r\|^2}{r^2},
    \end{equation}
    and therefore
    \begin{align}
        r\,\Delta \log r-\Delta r
        &=
        \Delta r-\frac{\|\nabla r\|^2}{r}-\Delta r \notag\\
        &=
        -\frac{\|\nabla r\|^2}{r}
        =
        -r\,\|\nabla\log r\|^2.
    \end{align}
    Consequently,
    \begin{align}
        \frac{d}{dt}D_t(\phi)
        &=
        -\frac{g(t)^2}{2}
        \int \bar p_t(x)r_t^\phi(x)
        \|
        \nabla\log r_t^\phi(x)
        \|^2\,dx \notag\\
        &=
        -\frac{g(t)^2}{2}
        \int p_t^\phi(x)
        \|
        \nabla \log r_t^\phi(x)
        \|^2\,dx.
        \label{eq:appendix_kl_dissipation}
    \end{align}
    Since
    \begin{equation}
        \nabla\log r_t^\phi(x)
        =
        \nabla\log p_t^\phi(x)
        -
        \nabla\log \bar{p}_t(x),
    \end{equation}
    we obtain
    \begin{equation}
        \frac{d}{dt}D_t(\phi)
        =
        -\frac{g(t)^2}{2}
        \int p_t^\phi(x)
        \|
        \nabla\log p_t^\phi(x)
        -
        \nabla\log\bar p_t(x)
        \|^2\,dx.
    \end{equation}
    Finally, averaging over $\phi\sim\Pi$ and interchanging
    differentiation and integration, we get
    \begin{align}
        \dot I_\model(t)
        &=
        \int \frac{d}{dt}D_t(\phi)\,\Pi(d\phi) \notag\\
        &=
        -\frac{g(t)^2}{2}
        \E_{\Phi\sim\Pi,\hat X_t\sim p_t^\Phi}
        \|
        \nabla\log p_t^\Phi(\hat X_t)
        -
        \nabla\log\bar p_t(\hat X_t)
        \|^2,
    \end{align}
    which proves \eqref{eq:appendix_model_information_rate_score_space}.
\end{proof}

\subsection{Approximating the Ensemble by Finitely Many Models}
\label{sub:direct_estimator}
The ensemble distribution $\bar p_t=\E_{\Phi\sim\Pi}[p_t^\Phi]$ and
its score $\bar s_t=\nabla\log\bar p_t$ generally cannot be computed
exactly. We therefore approximate the ensemble using a finite set of
independently trained models. There are two natural constructions,
depending on whether the same models are used both to approximate the
ensemble and to evaluate the score discrepancy, or whether an
independent set of models is held out for evaluation. We show that
these give lower and upper bounds, respectively, on model information.

\paragraph{Finite ensemble.} Let $B=(\Psi_1,\ldots,\Psi_K)$ with
$\Psi_1,\ldots,\Psi_K\sim_{\mathrm{i.i.d.}}\Pi$ be a set of reference
models, and define
\begin{equation}
    q_{B,t}(x)
    :=
    \frac{1}{K}
    \sum_{j=1}^K p_t^{\Psi_j}(x),
    \qquad
    s_{B,t}(x)
    :=
    \nabla\log q_{B,t}(x).
    \label{eq:finite_bank_ensemble}
\end{equation}

\paragraph{In-sample construction.} Using the same models to define the
ensemble and evaluate the discrepancy gives
\begin{equation}
    I_{\model,K}^{\mathrm{in}}(t)
    :=
    \E_B\left[
        \frac{1}{K}\sum_{j=1}^K
        D_{\mathrm{KL}}
        \left(
            p_t^{\Psi_j}\mid q_{B,t}
        \right)
    \right].
    \label{eq:one_bank_information}
\end{equation}

\begin{proposition}
    \label{prop:one_bank_lower_bound}
    The in-sample rate satisfies
    \begin{equation}
        \dot I_{\model,K}^{\mathrm{in}}(t)
        =
        -\frac{g(t)^2}{2}
        \E_B\left[
            \frac{1}{K}\sum_{j=1}^K
            \E_{X_t\sim p_t^{\Psi_j}}
            \|s_t^{\Psi_j}(X_t)-s_{B,t}(X_t)\|^2
        \right]
        \geq
        \dot I_\model(t).
        \label{eq:one_bank_rate_bound}
    \end{equation}
    Consequently,
    \begin{equation}
        I_{\model,K}^{\mathrm{in}}(t)
        \leq
        I_\model(t).
        \label{eq:one_bank_information_bound}
    \end{equation}
\end{proposition}

\begin{proof}
    The expression for $\dot I_{\model,K}^{\mathrm{in}}(t)$ follows
    from the same KL-dissipation identity as
    Proposition~\ref{prop:model_information_rate_score_based},
    applied conditionally on $B$. For fixed $B$, write
    \begin{equation}
        s_t^{\Psi_j}(x)-\bar s_t(x)
        =
        \left(s_t^{\Psi_j}(x)-s_{B,t}(x)\right)
        +
        \left(s_{B,t}(x)-\bar s_t(x)\right).
    \end{equation}
    After squaring and averaging over $j$, the cross
    term vanishes since
    \begin{align}
        \frac{1}{K}\sum_{j=1}^K
        p_t^{\Psi_j}(x)
        \left(s_t^{\Psi_j}(x)-s_{B,t}(x)\right)
        &=
        \frac{1}{K}\sum_{j=1}^K\nabla p_t^{\Psi_j}(x)
        -
        q_{B,t}(x)s_{B,t}(x) \notag\\
        &=
        0.
    \end{align}
    Hence
    \begin{align}
        &\frac{1}{K}\sum_{j=1}^K
        \E_{X_t\sim p_t^{\Psi_j}}
        \|s_t^{\Psi_j}(X_t)-\bar s_t(X_t)\|^2 \notag\\
        &\qquad=
        \frac{1}{K}\sum_{j=1}^K
        \E_{X_t\sim p_t^{\Psi_j}}
        \|s_t^{\Psi_j}(X_t)-s_{B,t}(X_t)\|^2
        +
        \E_{X_t\sim q_{B,t}}
        \|s_{B,t}(X_t)-\bar s_t(X_t)\|^2.
        \label{eq:one_bank_rate_decomposition}
    \end{align}
    Taking expectation over $B$ and using
    Proposition~\ref{prop:model_information_rate_score_based} gives
    $\dot I_{\model,K}^{\mathrm{in}}(t)\geq\dot I_\model(t)$. Since
    both information quantities vanish as $t\to\infty$,
    \begin{equation}
        I_{\model,K}^{\mathrm{in}}(t)
        =
        -\int_t^\infty \dot I_{\model,K}^{\mathrm{in}}(s)\,ds
        \leq
        -\int_t^\infty \dot I_\model(s)\,ds
        =
        I_\model(t).
    \end{equation}
\end{proof}

The in-sample construction also satisfies
\begin{equation}
    I_{\model,K}^{\mathrm{in}}(t)\leq\log K,
    \label{eq:one_bank_logk_bound}
\end{equation}
since, conditional on $B$, it is the mutual information between a
uniformly distributed index $J\in[K]$ and $X_t\sim p_t^{\Psi_J}$.

\paragraph{Held-out construction.} Let
$B=(\Psi_1,\ldots,\Psi_K)$ and
$B'=(\Phi_1,\ldots,\Phi_K)$ be two independent sets of models, with all models drawn
i.i.d.\ from $\Pi$. We use $B$ to define the finite ensemble
$q_{B,t}$ and $B'$ for evaluation. Define
\begin{align}
    I_{\model,K}^{\mathrm{out}}(t)
    &:=
    \E_{B,B'}
    \left[
        \frac{1}{K}
        \sum_{i=1}^K
        D_{\mathrm{KL}}
        \left(
            p_t^{\Phi_i}\mid q_{B,t}
        \right)
    \right] \notag\\
    &=
    \E_{B,\Phi}
    D_{\mathrm{KL}}
    \left(
        p_t^\Phi\mid q_{B,t}
    \right),
    \qquad \Phi\sim\Pi,
    \label{eq:two_bank_information}
\end{align}
where the second equality follows because the models in $B'$ are i.i.d.\
draws from $\Pi$ and independent of $B$.

\begin{proposition}
    \label{prop:two_bank_upper_bound}
    The held-out rate satisfies
    \begin{equation}
        \dot I_{\model,K}^{\mathrm{out}}(t)
        =
        -\frac{g(t)^2}{2}
        \E_{B,\Phi,\,X_t\sim p_t^\Phi}
        \|s_t^\Phi(X_t)-s_{B,t}(X_t)\|^2
        \leq
        \dot I_\model(t).
        \label{eq:two_bank_rate_bound}
    \end{equation}
    Consequently,
    \begin{equation}
        I_{\model,K}^{\mathrm{out}}(t)
        \geq
        I_\model(t).
        \label{eq:two_bank_information_bound}
    \end{equation}
\end{proposition}

\begin{proof}
    Again, the rate expression follows from the KL-dissipation identity. Write
    \begin{equation}
        s_t^\Phi(X_t)-s_{B,t}(X_t)
        =
        \left(s_t^\Phi(X_t)-\bar s_t(X_t)\right)
        +
        \left(\bar s_t(X_t)-s_{B,t}(X_t)\right).
    \end{equation}
    The cross term vanishes because $B$ is independent of $(\Phi,X_t)$ and
    \begin{equation}
        \E_{\Phi\mid X_t=x}
        \left[
            s_t^\Phi(x)-\bar s_t(x)
        \right]
        =
        0.
    \end{equation}
    Therefore,
    \begin{align}
        &\E_{B,\Phi,\,X_t\sim p_t^\Phi}
        \|s_t^\Phi(X_t)-s_{B,t}(X_t)\|^2 \notag\\
        &\qquad=
        \E_{\Phi,\,X_t\sim p_t^\Phi}
        \|s_t^\Phi(X_t)-\bar s_t(X_t)\|^2
        +
        \E_B\E_{X_t\sim\bar p_t}
        \|\bar s_t(X_t)-s_{B,t}(X_t)\|^2.
        \label{eq:two_bank_rate_decomposition}
    \end{align}
    Proposition~\ref{prop:model_information_rate_score_based} then
    gives $\dot I_{\model,K}^{\mathrm{out}}(t)\leq\dot I_\model(t)$.
    Integrating from $t$ to $\infty$ gives
    \begin{equation}
        I_{\model,K}^{\mathrm{out}}(t)
        =
        -\int_t^\infty\dot I_{\model,K}^{\mathrm{out}}(s)\,ds
        \geq
        -\int_t^\infty\dot I_\model(s)\,ds
        =
        I_\model(t).
    \end{equation}
\end{proof}

Thus the two constructions give a lower and an upper bound on model information:
\begin{equation}
    I_{\model,K}^{\mathrm{in}}(t)
    \leq
    I_\model(t)
    \leq
    I_{\model,K}^{\mathrm{out}}(t).
    \label{eq:finite_bank_information_bracketing}
\end{equation}
Equivalently, direct KL decompositions give
\begin{equation}
    I_\model(t)-I_{\model,K}^{\mathrm{in}}(t)
    =
    \E_BD_{\mathrm{KL}}(q_{B,t}\mid\bar p_t),
    \qquad
    I_{\model,K}^{\mathrm{out}}(t)-I_\model(t)
    =
    \E_BD_{\mathrm{KL}}(\bar p_t\mid q_{B,t}).
    \label{eq:finite_bank_kl_gaps}
\end{equation}

\paragraph{Choice of construction.} The held-out construction preserves
the singular pure-memorization limit. If independently drawn clean
model laws are mutually singular, then $I_\model(0)=\infty$. For a
set of $K$ mutually singular models, however,
$q_{B,0}=K^{-1}p_0^{\Psi_j}$ on the support of $p_0^{\Psi_j}$, giving
$I_{\model,K}^{\mathrm{in}}(0)=\log K$. An independent evaluation
model is instead almost surely singular with respect to $q_{B,0}$, so
$I_{\model,K}^{\mathrm{out}}(0)=\infty$. We therefore use the
held-out construction in our experiments.

\subsection{Mixture Score Estimators}
\label{sub:conditional_score_expectation_estimator}

The experiments use $K_{\mathrm{train}}$ models
$\Psi_1,\ldots,\Psi_{K_{\mathrm{train}}}$ to approximate the learned
model law $\Pi$. These models are independent of the
$K_{\mathrm{val}}$ models used for evaluation. Their noised mixture
density and posterior weights are
\begin{equation}
    q_{\mathrm{train},t}(x)
    :=
    \frac{1}{K_{\mathrm{train}}}
    \sum_{j=1}^{K_{\mathrm{train}}} p_t^{\Psi_j}(x),
    \qquad
    W_{j,\mathrm{train}}(x,t)
    :=
    \frac{p_t^{\Psi_j}(x)}
    {\sum_{\ell=1}^{K_{\mathrm{train}}}p_t^{\Psi_\ell}(x)}.
    \label{eq:finite_bank_posterior_weights}
\end{equation}
Differentiating the mixture density gives
\begin{equation}
    \nabla\log q_{\mathrm{train},t}(x)
    =
    \sum_{j=1}^{K_{\mathrm{train}}}
    W_{j,\mathrm{train}}(x,t)\nabla\log p_t^{\Psi_j}(x).
    \label{eq:finite_bank_mixture_score}
\end{equation}
We consider three approximations to this score, all constructed
solely from the reference models and therefore independent of the
validation models.

The first fits $\nabla\log q_{\mathrm{train},t}$ directly by
denoising score matching on samples from a uniformly chosen
reference model:
\begin{equation}
    \hat s_{B,t}
    \in
    \argmin_s
    \E_{\substack{J\sim\operatorname{Unif}[K_{\mathrm{train}}],
            \\ \hat X\sim\Psi_J,
    \epsilon\sim\mathcal N(0,I) \\ \hat X_t=\alpha_t\hat X+\sigma_t\epsilon}}
    \|
    s(\hat X_t)
    -
    (\alpha_t\hat X-\hat X_t)/\sigma_t^2
    \|^2.
    \label{eq:direct_estimator_reference_objective}
\end{equation}
With unlimited generated data and sufficient capacity, the minimizer
is $\nabla\log q_{\mathrm{train},t}$. We call this the
\emph{distilled ensemble score} estimator.

The second estimator learns the weights in
\eqref{eq:finite_bank_posterior_weights}. We train a time conditioned
$K_{\mathrm{train}}$ class classifier on labelled samples from the training
models after forward noising. With equal class frequencies, the population
minimizer of the cross entropy loss is $W_{j,\mathrm{train}}(x,t)$.
Substitution into \eqref{eq:finite_bank_mixture_score} gives the classifier
mixture score. We call this the \emph{model-based classifier} estimator,
since it requires training an additional network.

The third estimator obtains the same weights from the training model
densities. For each training model, the diffusion density identity of
\citet{kong2023information} gives
\begin{equation}
    \log p_t^{\Psi_j}(x)
    =
    \log p_{G_j,t}(x)
    +
    \frac{1}{2}\int_0^\infty
    \left[
        m_{G_j,t}(x;\gamma)-m_{j,t}(x;\gamma)
    \right]d\gamma,
    \label{eq:finite_bank_diffusion_density_identity}
\end{equation}
where $G_j$ is a moment matched Gaussian reference and
$m_{j,t}(x;\gamma)$ is the mean squared reconstruction error for a
Gaussian channel of signal to noise ratio $\gamma$ starting at
$(x,t)$. Tweedie's formula expresses this reconstruction through the
score of $\Psi_j$ at the later channel time. We evaluate
\eqref{eq:finite_bank_diffusion_density_identity} by quadrature and
Monte Carlo with shared random numbers, then apply a softmax over the
estimated log densities. This estimator requires score evaluations
from the training models but no additional network, and we call it the
\emph{model-free classifier} estimator.

All three approaches can be used for $\hat s_{B,t}$ in
\eqref{eq:general_two_bank_rate_bound}. By
Corollary~\ref{cor:general_two_bank_upper_bound}, any additional
error from using finitely many reference models or from score estimation
contributes a non-negative term
to the expected information-loss rate. The main-text figures use the
model-free classifier estimator; the model-based classifier and
distilled estimators are evaluated on the same noised validation samples.

\begin{corollary}
    \label{cor:general_two_bank_upper_bound}
    Let $\hat s_{B,t}$ be any square-integrable score approximation
    constructed from the reference models $B$ independently of the
    evaluation model $\Phi$. Then
    \begin{equation}
        -\frac{g(t)^2}{2}
        \E_{B,\Phi,\,X_t\sim p_t^\Phi}
        \|s_t^\Phi(X_t)-\hat s_{B,t}(X_t)\|^2
        \leq
        \dot I_\model(t).
        \label{eq:general_two_bank_rate_bound}
    \end{equation}
    Integrating from $t$ to $\infty$ therefore gives an upper bound
    on $I_\model(t)$.
\end{corollary}

\begin{proof}
    The same orthogonal decomposition gives
    \begin{align}
        &\E_{B,\Phi,\,X_t\sim p_t^\Phi}
        \|s_t^\Phi(X_t)-\hat s_{B,t}(X_t)\|^2 \notag\\
        &\qquad=
        \E_{\Phi,\,X_t\sim p_t^\Phi}
        \|s_t^\Phi(X_t)-\bar s_t(X_t)\|^2
        +
        \E_B\E_{X_t\sim\bar p_t}
        \|\bar s_t(X_t)-\hat s_{B,t}(X_t)\|^2,
    \end{align}
    and the second term is non-negative.
\end{proof}

\subsection{Model Information via Integration}
\label{sec:model_information_integration}
Let \(t_{\min}=t_0<t_1<\cdots<t_R=T\) be an evaluation grid. We use the
trapezoidal estimate in the log signal-to-noise ratio
\(\lambda_t=\log(\alpha_t^2/\sigma_t^2)\), which decreases in \(t\):
\begin{equation}
    \hat I_\model(t_{\min})
    :=
    -\sum_{r=0}^{R-1}
    \frac{\lambda_{t_r}-\lambda_{t_{r+1}}}{2}
    \left[
        \frac{\widehat{\dot I}_\model(t_r)}{|\dot\lambda_{t_r}|}
        +
        \frac{\widehat{\dot I}_\model(t_{r+1})}{|\dot\lambda_{t_{r+1}}|}
    \right].
    \label{eq:finite_interval_information_estimate}
\end{equation}
With exact scores and exact integration, this quantity equals
$I_\model(t_{\min})-I_\model(T)$. The terminal term vanishes when
$\alpha_T=0$ and is negligible when $\sigma_T / \alpha_T$ is large. Our
variance-exploding (EDM) schedule approaches this
condition at its terminal endpoint.

The clean-sample information $I_\model(0)$ can be infinite when the
learned distributions contain model-specific atoms. We therefore take
a positive $t_{\min}$ as the lower bound of the interval, hence estimating $I_\model(t_{\min})$. In the image experiments, we
use 25 points spaced uniformly in $t\in[0.05,0.99]$. Every reported estimate of model information is
\eqref{eq:finite_interval_information_estimate} with $t_{\min}=0.05$ and
$T=0.99$, which we write as $\hat I_\model(P,N)$ for models with $P$
parameters trained on $N$ samples.

For the numerical estimation of model information, the main
approximation errors are the error from finitely many reference models,
the score-estimation error
described above, and trapezoidal quadrature error. Interpreting
\eqref{eq:finite_interval_information_estimate} as an estimate of
$I_\model(t_{\min})$ additionally incurs the terminal remainder $I_\model(T)$.

\section{Properties of Model Information}
\label{app:properties_model_information}

This appendix collects descriptions, proofs, and additional details
regarding structural
properties of model information. We first prove a
posterior-score identity and the associated bias-variance decompositions in
score space and distribution space. We then derive information-theoretic
bounds showing that small model information limits the discrepancy between
model-dependent observables under an individual trained model and under the
ensemble model, with memorization diagnostics as a particular case. Finally,
we discuss how the same construction extends beyond dataset-induced learned
models to other choices of the random hypothesis $\Phi$.

\subsection{Bias-variance decompositions}
\label{app:bias_variance_decompositions}

We first prove the posterior-score identity, 
which states that the ensemble score is the posterior mean of the
individual learned-model scores after observing the noised sample.

\begin{proposition}
    \label{prop:posterior_score_identity}
    Let $\Phi\sim\Pi$ be the random learned model, let $\hat X_t\mid
    \Phi=\phi\sim p_t^\phi$, let $\bar p_t(x)=\mathbb
    E_{\Phi\sim\Pi}[p_t^\Phi(x)]$, and let $\Pi_t(\cdot\mid x)$ be the
    law of $\Phi$ given $\hat X_t=x$. Assume that $\bar p_t(x)>0$ and
    that differentiation under the integral defining $\bar p_t$ is
    justified. Then
    \begin{equation}
        \bar s_t(x)
        =
        \mathbb E_{\Phi\sim\Pi_t(\cdot\mid x)}
        \left[
            s_\Phi(x,t)
        \right].
        \label{eq:app_posterior_score_mean_identity}
    \end{equation}
    Consequently, $\bar s_t(x)$ is the Bayes-optimal predictor of
    $s_\Phi(x,t)$ under squared error conditional on $\hat X_t=x$.
\end{proposition}

\begin{proof}
    By Bayes' rule, the posterior distribution over learned models
    after observing $\hat X_t=x$ is
    \begin{equation}
        \Pi_t(d\phi\mid x)
        =
        \frac{p_t^\phi(x)\Pi(d\phi)}{\bar p_t(x)}.
    \end{equation}
    Using the mixture identity $\bar p_t(x)=\int
    p_t^\phi(x)\Pi(d\phi)$ and differentiating under the integral
    sign, we obtain
    \begin{equation}
        \nabla \bar p_t(x)
        =
        \int \nabla p_t^\phi(x)\Pi(d\phi)
        =
        \int p_t^\phi(x)s_\phi(x,t)\Pi(d\phi).
    \end{equation}
    Dividing by $\bar p_t(x)$ gives
    \begin{equation}
        \nabla \log \bar p_t(x)
        =
        \int s_\phi(x,t)\frac{p_t^\phi(x)\Pi(d\phi)}{\bar p_t(x)}
        =
        \int s_\phi(x,t)\Pi_t(d\phi\mid x).
    \end{equation}
    Therefore,
    \begin{equation}
        \bar s_t(x)
        =
        \mathbb E_{\Phi\sim\Pi_t(\cdot\mid x)}
        \left[
            s_\Phi(x,t)
        \right].
    \end{equation}
    The Bayes-optimality statement follows from the standard
    characterization of conditional expectation as the minimizer of
    conditional mean-squared error:
    \begin{equation}
        \mathbb E_{\Phi\sim\Pi_t(\cdot\mid x)}
        \left[
            s_\Phi(x,t)
        \right]
        =
        \arg\min_{a\in\mathbb R^d}
        \mathbb E_{\Phi\sim\Pi_t(\cdot\mid x)}
        \left[
            \|s_\Phi(x,t)-a\|^2
        \right].
    \end{equation}
\end{proof}

We next prove a bias--variance decomposition mirroring that of
\citet{kadkhodaie2023generalization}.
The first decomposition
is a score-space identity, while the second is a distributional identity in
which the model information appears as the variance-like term.

\begin{proposition}
    \label{prop:app_model_information_bias_variance}
    For every fixed $x$ such that the relevant conditional
    expectations are well-defined,
    \begin{equation}
        \mathbb E_{\Phi\sim\Pi_t(\cdot\mid x)}
        [
            \|s_\Phi(x,t)-s_t^\star(x)\|^2
        ]
        =
        \|\bar s_t(x)-s_t^\star(x)\|^2
        +
        \mathbb E_{\Phi\sim\Pi_t(\cdot\mid x)}
        [
            \|s_\Phi(x,t)-\bar s_t(x)\|^2
        ].
        \label{eq:app_score_bias_variance_decomposition}
    \end{equation}
    Moreover,
    \begin{equation}
        \mathbb E_{\Phi\sim\Pi}
        [
            D_{\mathrm{KL}}(p_t^\Phi\|p_t^{\phi_\star})
        ]
        =
        I(\Phi;\hat X_t)
        +
        D_{\mathrm{KL}}(\bar p_t\|p_t^{\phi_\star}).
        \label{eq:app_kl_bias_variance_decomposition}
    \end{equation}
\end{proposition}

\begin{proof}
    For the score-space decomposition, fix $x$ and write
    \begin{equation}
        s_\Phi(x,t)-s_t^\star(x)
        =
        \left(s_\Phi(x,t)-\bar s_t(x)\right)
        +
        \left(\bar s_t(x)-s_t^\star(x)\right).
    \end{equation}
    Taking squared norms and averaging over $\Phi\mid \hat X_t=x$ gives
    \begin{align}
        \mathbb E_{\Phi\sim\Pi_t(\cdot\mid x)}
        \left[
            \|s_\Phi(x,t)-s_t^\star(x)\|^2
        \right]
        &=
        \mathbb E_{\Phi\sim\Pi_t(\cdot\mid x)}
        \left[
            \|s_\Phi(x,t)-\bar s_t(x)\|^2
        \right]
        +
        \|\bar s_t(x)-s_t^\star(x)\|^2 \nonumber \\
        &\quad+
        2
        \left\langle
        \mathbb E_{\Phi\sim\Pi_t(\cdot\mid x)}
        \left[
            s_\Phi(x,t)-\bar s_t(x)
        \right],
        \bar s_t(x)-s_t^\star(x)
        \right\rangle.
    \end{align}
    The cross term vanishes by
    Proposition~\ref{prop:posterior_score_identity}, since $\mathbb
    E_{\Phi\sim\Pi_t(\cdot\mid x)}[s_\Phi(x,t)-\bar s_t(x)]=0$. This proves
    \eqref{eq:app_score_bias_variance_decomposition}.

    For the distributional decomposition, decompose the logarithm
    inside the KL divergence as
    \begin{equation}
        \log \frac{p_t^\Phi(x)}{p_t^{\phi_\star}(x)}
        =
        \log \frac{p_t^\Phi(x)}{\bar p_t(x)}
        +
        \log \frac{\bar p_t(x)}{p_t^{\phi_\star}(x)}.
    \end{equation}
    Therefore,
    \begin{align}
        \mathbb E_{\Phi\sim\Pi}
        [
            D_{\mathrm{KL}}(p_t^\Phi\|p_t^{\phi_\star})
        ]
        &=
        \mathbb E_{\Phi\sim\Pi}
        \int p_t^\Phi(x)
        \log \frac{p_t^\Phi(x)}{\bar p_t(x)}
        dx \nonumber \\
        &\quad+
        \mathbb E_{\Phi\sim\Pi}
        \int p_t^\Phi(x)
        \log \frac{\bar p_t(x)}{p_t^{\phi_\star}(x)}
        dx.
    \end{align}
    The first term is the mutual information between $\Phi$ and $\hat X_t$:
    \begin{equation}
        \mathbb E_{\Phi\sim\Pi}
        \int p_t^\Phi(x)
        \log \frac{p_t^\Phi(x)}{\bar p_t(x)}
        dx
        =
        I(\Phi;\hat X_t).
    \end{equation}
    For the second term, using $\mathbb
    E_{\Phi\sim\Pi}[p_t^\Phi(x)]=\bar p_t(x)$ gives
    \begin{equation}
        \mathbb E_{\Phi\sim\Pi}
        \int p_t^\Phi(x)
        \log \frac{\bar p_t(x)}{p_t^{\phi_\star}(x)}
        dx
        =
        \int \bar p_t(x)
        \log \frac{\bar p_t(x)}{p_t^{\phi_\star}(x)}
        dx
        =
        D_{\mathrm{KL}}\left(\bar p_t\,\|\,p_t^{\phi_\star}\right).
    \end{equation}
    Combining the two terms proves
    \eqref{eq:app_kl_bias_variance_decomposition}.
\end{proof}

In score space, the term $\mathbb E_{\Phi\sim\Pi_t(\cdot\mid
x)}[\|s_\Phi(x,t)-\bar s_t(x)\|^2]$ measures the posterior dispersion of
trained-model scores around the ensemble score, while $\|\bar
s_t(x)-s_t^\star(x)\|^2$ measures the systematic discrepancy between the
ensemble score and the population score. In distribution space, $I(\Phi;\hat
X_t)$ measures the part of the average KL discrepancy due to the
distinguishability of individual learned models from the ensemble, while
$D_{\mathrm{KL}}(\bar p_t\,\|\,p_t^{\phi_\star})$ measures the discrepancy
between the ensemble law and the population law.

We may connect the two terms in
\eqref{eq:app_score_bias_variance_decomposition} to a conditional bias-variance
decomposition, as discussed by
\citet{kadkhodaie2023generalization}.
The first term $\|\nabla \log \bar
p_t(x)-\nabla \log p_t^{\phi_{\star}}(x)\|^2$ measures the systematic
discrepancy between
the ensemble score and the population score, and therefore
corresponds to a bias term.
The second term $\mathbb E_{\Phi \sim \Pi \mid
\{\hat{X}_t = x\}}
[
    \|\nabla \log p_t^{\Phi}(x)-\nabla \log \bar{p}_t(x)\|^2
]$ measures the dispersion of trained model scores around their
posterior average, and therefore corresponds to a variance term.

\subsection{Observable and memorization bounds}
\label{app:observable_memorization_bounds}

We next prove that model information controls the extent to which
model-dependent observables can fluctuate between an individual learned model
and the ensemble model. We allow the observable to additionally depend on side
information $Z$ jointly distributed with $\Phi$, provided that
$Z\to\Phi\to\hat X_t$ forms a Markov chain.

\begin{proposition}
    \label{prop:subgaussian_observable_stability}
    Let $(Z,\Phi)$ be jointly distributed, with marginal $\Phi\sim\Pi$,
    and suppose that $Z\to\Phi\to\hat X_t$ forms a Markov chain. Let
    $(z,\phi)\mapsto f_{z,\phi}:\mathcal X\to\mathbb R$ be a measurable
    assignment of observables. Assume that, for every $(z,\phi)$ in the
    support of the joint law, the centered observable
    $f_{z,\phi}-\mathbb E_{X_t\sim\bar p_t}[f_{z,\phi}]$ is
    $\sigma^2$-sub-Gaussian under $\bar p_t$, uniformly in $(z,\phi)$,
    meaning that for all $\lambda\in\mathbb R$,
    \begin{equation}
        \log \mathbb E_{X_t\sim\bar p_t}
        \left[
            \exp\left(
                \lambda\left(f_{z,\phi}(X_t)-\mathbb E_{X_t\sim\bar
                p_t}[f_{z,\phi}]\right)
            \right)
        \right]
        \leq
        \frac{\sigma^2\lambda^2}{2}.
    \end{equation}
    Then
    \begin{equation}
        \mathbb E_{Z,\Phi}
        \left[
            \left|
            \mathbb E_{X_t\sim p_t^\Phi}[f_{Z,\Phi}]
            -
            \mathbb E_{X_t\sim\bar p_t}[f_{Z,\Phi}]
            \right|
        \right]
        \leq
        \sqrt{2\sigma^2 I_m(t)}.
        \label{eq:subgaussian_observable_stability}
    \end{equation}
\end{proposition}

\begin{proof}
    Fix $(z,\phi)$ and write $P=p_t^\phi$, $Q=\bar p_t$, and
    $f=f_{z,\phi}$. By the variational form of the KL divergence, for
    every $\lambda>0$,
    \begin{equation}
        \lambda\left(\mathbb E_{X_t\sim P}[f]-\mathbb E_{X_t\sim Q}[f]\right)
        \leq
        D_{\mathrm{KL}}(P\|Q)
        +
        \log\mathbb E_{X_t\sim Q}
        \left[
            \exp\left(\lambda(f-\mathbb E_{X_t\sim Q}[f])\right)
        \right].
    \end{equation}
    Using the sub-Gaussian assumption gives
    \begin{equation}
        \mathbb E_{X_t\sim P}[f]-\mathbb E_{X_t\sim Q}[f]
        \leq
        \frac{1}{\lambda}D_{\mathrm{KL}}(P\|Q)
        +
        \frac{\sigma^2\lambda}{2}.
    \end{equation}
    Optimizing over $\lambda>0$ yields
    \begin{equation}
        \mathbb E_{X_t\sim P}[f]-\mathbb E_{X_t\sim Q}[f]
        \leq
        \sqrt{2\sigma^2D_{\mathrm{KL}}(P\|Q)}.
    \end{equation}
    Applying the same argument to $-f$ gives
    \begin{equation}
        \left|
        \mathbb E_{X_t\sim p_t^\phi}[f_{z,\phi}]
        -
        \mathbb E_{X_t\sim\bar p_t}[f_{z,\phi}]
        \right|
        \leq
        \sqrt{2\sigma^2D_{\mathrm{KL}}(p_t^\phi\|\bar p_t)}.
    \end{equation}
    Averaging over $(Z,\Phi)$ and using Jensen's inequality gives
    \begin{align}
        \mathbb E_{Z,\Phi}
        \left[
            \left|
            \mathbb E_{X_t\sim p_t^\Phi}[f_{Z,\Phi}]
            -
            \mathbb E_{X_t\sim\bar p_t}[f_{Z,\Phi}]
            \right|
        \right]
        &\leq
        \sqrt{2\sigma^2}
        \mathbb E_{Z,\Phi}
        \left[
            \sqrt{
                D_{\mathrm{KL}}(p_t^\Phi\|\bar p_t)
            }
        \right] \nonumber \\
        &=
        \sqrt{2\sigma^2}
        \mathbb E_{\Phi\sim\Pi}
        \left[
            \sqrt{
                D_{\mathrm{KL}}(p_t^\Phi\|\bar p_t)
            }
        \right] \nonumber \\
        &\leq
        \sqrt{
            2\sigma^2
            \mathbb E_{\Phi\sim\Pi}
            \left[
                D_{\mathrm{KL}}(p_t^\Phi\|\bar p_t)
            \right]
        }.
    \end{align}
    Since
    \begin{equation}
        \mathbb E_{\Phi\sim\Pi}
        \left[
            D_{\mathrm{KL}}(p_t^\Phi\|\bar p_t)
        \right]
        =
        I(\Phi;\hat X_t)
        =
        I_m(t),
    \end{equation}
    we obtain the claim.
\end{proof}

A useful special case is obtained when the observables are uniformly bounded.

\begin{corollary}
    \label{cor:bounded_observable_stability}
    Under the assumptions above, let
    $(z,\phi)\mapsto f_{z,\phi}:\mathcal X\to\mathbb R$ be a measurable
    assignment of observables such that $f_{z,\phi}(x)\in[a,b]$ for all
    $x$ and all $(z,\phi)$ in the support of the joint law. Then
    \begin{equation}
        \mathbb E_{Z,\Phi}
        \left[
            \left|
            \mathbb E_{X_t\sim p_t^\Phi}[f_{Z,\Phi}]
            -
            \mathbb E_{X_t\sim\bar p_t}[f_{Z,\Phi}]
            \right|
        \right]
        \leq
        (b-a)\sqrt{1-\exp(-I_m(t))}.
        \label{eq:bounded_observable_stability}
    \end{equation}
\end{corollary}

\begin{proof}
    Fix $(z,\phi)$ and write $P=p_t^\phi$, $Q=\bar p_t$, and
    $f=f_{z,\phi}$. Since $f\in[a,b]$,
    \begin{equation}
        \left|
        \mathbb E_{X_t\sim P}[f]
        -
        \mathbb E_{X_t\sim Q}[f]
        \right|
        \leq
        (b-a)d_{\mathrm{TV}}(P,Q).
    \end{equation}
    The Bretagnolle--Huber inequality gives
    \begin{equation}
        d_{\mathrm{TV}}(P,Q)
        \leq
        \sqrt{1-\exp\left(-D_{\mathrm{KL}}(P\|Q)\right)}.
    \end{equation}
    Hence
    \begin{equation}
        \left|
        \mathbb E_{X_t\sim p_t^\phi}[f_{z,\phi}]
        -
        \mathbb E_{X_t\sim\bar p_t}[f_{z,\phi}]
        \right|
        \leq
        (b-a)
        \sqrt{
            1-\exp\left(-D_{\mathrm{KL}}(p_t^\phi\|\bar p_t)\right)
        }.
    \end{equation}
    Averaging over $(Z,\Phi)$ and using Jensen's inequality, together
    with the concavity of $u\mapsto\sqrt{1-e^{-u}}$, gives
    \begin{align}
        \mathbb E_{Z,\Phi}
        \left[
            \left|
            \mathbb E_{X_t\sim p_t^\Phi}[f_{Z,\Phi}]
            -
            \mathbb E_{X_t\sim\bar p_t}[f_{Z,\Phi}]
            \right|
        \right]
        &\leq
        (b-a)
        \sqrt{
            1-\exp\left(
                -\mathbb E_{\Phi\sim\Pi}
                \left[
                    D_{\mathrm{KL}}(p_t^\Phi\|\bar p_t)
                \right]
            \right)
        } \nonumber\\
        &=
        (b-a)\sqrt{1-\exp(-I_m(t))}.
    \end{align}
\end{proof}

Alternatively, applying Pinsker's inequality gives the bound
$(b-a)\sqrt{I_m(t)/2}$ in
Corollary~\ref{cor:bounded_observable_stability}. This is tighter for
small $I_m(t)$, whereas the exponential form remains bounded by $b-a$
for arbitrary model information.

The event-stability bound follows by applying the bounded-observable result to
indicator functions.

\begin{corollary}
    \label{cor:event_stability_bound}
    For any measurable assignment $(z,\phi)\mapsto
    A_{z,\phi}\subseteq\mathcal X$, we have
    \begin{equation}
        \mathbb E_{Z,\Phi}
        \left[
            \left|
            P_t^\Phi(A_{Z,\Phi})
            -
            \bar P_t(A_{Z,\Phi})
            \right|
        \right]
        \leq
        \sqrt{1-\exp(-I_m(t))}.
        \label{eq:app_event_stability_bound}
    \end{equation}
\end{corollary}

\begin{proof}
    Take $f_{Z,\Phi}=\mathbf 1_{A_{Z,\Phi}}$. Then
    $f_{Z,\Phi}\in[0,1]$, and
    \begin{equation}
        \mathbb E_{X_t\sim p_t^\Phi}[f_{Z,\Phi}]
        =
        P_t^\Phi(A_{Z,\Phi}),
        \qquad
        \mathbb E_{X_t\sim\bar p_t}[f_{Z,\Phi}]
        =
        \bar P_t(A_{Z,\Phi}).
    \end{equation}
    The claim follows from Corollary~\ref{cor:bounded_observable_stability}.
\end{proof}

We now specialize this bound to dataset-specific memorization
diagnostics. Let $\mathcal D$ be the random training dataset jointly
inducing the learned model $\Phi$, and let
$A_{\mathrm{mem}}(\mathcal D)\subseteq\mathcal X$ be the set of samples
declared memorized relative to $\mathcal D$ by a chosen diagnostic. For
instance, a nearest-neighbor-ratio diagnostic with threshold
$\tau\in(0,1)$:
\begin{equation}
    A_{\mathrm{mem}}(\mathcal D)
    =
    \left\{
        x\in\mathcal X:
        \frac{
            \|x-x_{(1)}(\mathcal D)\|
        }{
            \|x-x_{(2)}(\mathcal D)\|
        }
        \leq
        \tau
    \right\},
    \label{eq:memorization_set_definition}
\end{equation}
where $x_{(1)}(\mathcal D)$ and $x_{(2)}(\mathcal D)$ denote the nearest
and second-nearest elements of $\mathcal D$ to $x$. Since generation
from the trained model satisfies
$\mathcal D\to\Phi\to\hat X_t$, applying
Corollary~\ref{cor:event_stability_bound} with $Z=\mathcal D$ and
$A_{\mathcal D,\Phi}=A_{\mathrm{mem}}(\mathcal D)$ gives
\begin{equation}
    \mathbb E_{\mathcal D,\Phi}
    \left[
        P_t^\Phi
        \left(
            A_{\mathrm{mem}}(\mathcal D)
        \right)
    \right]
    \leq
    \mathbb E_{\mathcal D}
    \left[
        \bar P_t
        \left(
            A_{\mathrm{mem}}(\mathcal D)
        \right)
    \right]
    +
    \sqrt{1-\exp(-I_m(t))}.
    \label{eq:app_memorization_ratio_bound}
\end{equation}
The left-hand side is the expected memorization ratio of samples
generated by models relative to their own training datasets. The
first term on the right-hand side is an ensemble baseline: it
measures how often the ensemble model generates samples that the same
diagnostic would call memorized relative to the random training
dataset $\mathcal D$. Thus, model information controls the excess
dataset-specific memorization above this baseline. Following the
discussion in Appendix~\ref{app:bias_variance_decompositions}, if the
space of possible learned models has low bias, then the ensemble
model approximates the population distribution, and this baseline
approaches the memorization rate intrinsic to the data distribution.

\subsection{Extensions: alternative hypotheses}
\label{app:extensions_alternative_hypotheses}

We briefly discuss the broader viewpoint from which the model information
construction arose.

The definition of model information does not require
diffusion models: for any random generative mechanism $\Phi$ and any sample
$\hat X\sim \Phi$, the mutual information $I(\Phi;\hat X)$ measures how much
the sample reveals about the mechanism that produced it. Diffusion is useful
here because it supplies a canonical noising channel $\hat X\mapsto \hat X_t$
and turns this quantity into a time-resolved information curve with a
tractable score-based production rate. From this perspective, the starting
point is to view diffusion as a noisy communication channel through which
information about a random source hypothesis $\Phi$ is progressively erased.
The main modeling question is then which hypothesis $\Phi$ one chooses to
track. In this paper, we take $\Phi$ to be the learned model induced by
data set resampling, because this choice turns the resulting mutual information
into a measure of finite-data dependence and, hence, into a measure of
generalization and memorization. Other choices of $\Phi$ may reveal different
structures erased by diffusion, and remain interesting directions for future
work.

Let $\Phi$ be a source hypothesis taking values in a hypothesis space
$\mathcal H$ with prior $\Pi$. Conditional on $\Phi=\phi$, a clean
sample is drawn from a source-specific law $p_0^\phi$, and then
transmitted through the forward Gaussian channel
\begin{equation}
    X_t=\alpha_tX_0+\sigma_t\epsilon,
    \qquad
    \epsilon\sim\mathcal N(0,I).
\end{equation}
Thus the relevant communication chain is
\begin{equation}
    \Phi \;\longrightarrow\; X_0 \;\longrightarrow\; X_t.
\end{equation}
The receiver observes $X_t$ and attempts to infer which source
hypothesis generated the sample. The choice of $\mathcal H$ and $\Pi$
therefore specifies which structure is tracked through the noising channel.

Each hypothesis $\phi$ induces the noised law
\begin{equation}
    p_t^\phi(x_t)
    :=
    \int p_t(x_t\mid x_0)p_0^\phi(x_0)\,dx_0,
\end{equation}
and the corresponding prior-averaged law is
\begin{equation}
    \bar p_t(x_t)
    :=
    \int p_t^\phi(x_t)\,\Pi(d\phi).
\end{equation}
Whenever $\bar p_t(x_t)>0$, Bayes' rule gives
\begin{equation}
    \Pi_t(d\phi\mid x_t)
    =
    \frac{p_t^\phi(x_t)\Pi(d\phi)}{\bar p_t(x_t)}.
    \label{eq:posterior_source_hypothesis}
\end{equation}
The associated hypothesis information is
\begin{equation}
    I(\Phi;X_t)
    =
    \mathbb E_{\Phi\sim\Pi}
    \left[
        D_{\mathrm{KL}}\left(p_t^\Phi\,\|\,\bar p_t\right)
    \right]
    =
    \mathbb E_{X_t\sim\bar p_t}
    \left[
        D_{\mathrm{KL}}\left(\Pi_t(\cdot\mid X_t)\,\|\,\Pi\right)
    \right].
    \label{eq:general_hypothesis_information}
\end{equation}
Hence $I(\Phi;X_t)$ measures how much observing the noised sample
changes the receiver's belief about the source hypothesis.

Whenever the hypothesis-conditioned laws evolve under the same
forward diffusion, the same information-production identity applies.
In the scalar diffusion setting considered in the paper,
\begin{equation}
    \frac{d}{dt}I(\Phi;X_t)
    =
    -\frac{g(t)^2}{2}
    \mathbb E_{\Phi\sim\Pi,\,X_t\sim p_t^\Phi}
    \left[
        \left\|
        \nabla\log p_t^\Phi(X_t)
        -
        \nabla\log\bar p_t(X_t)
        \right\|^2
    \right].
    \label{eq:general_hypothesis_information_rate}
\end{equation}
Thus, the rate identity is common across choices of $\Phi$, while its
interpretation is determined by the hypothesis space. For example,
taking $\Phi=X_0$ measures the information retained about the clean
sample \citep{kong2023information,stancevic2026entropic}; taking
$\Phi$ to be a class or latent semantic variable tracks the emergence
of semantic structure \citep{franzese2025latent,handke2026the};
taking $\Phi$ to be a conditioning variable, such as a text caption,
quantifies the information relating the conditioning signal to the
generated sample \citep{kong2023interpretable}. Training-example
identities provide another choice, whose information content has
recently been related to memorization-generalization transitions
\citep{hunt2026exact}.

\section{A Linear-Gaussian Model}
\label{app:linear_gaussian}

In this section, we study a linear-Gaussian model in which model
information can be characterized explicitly. The purpose of this
example is to provide a
tractable analogue of the scaling behavior observed empirically. When the
number $N$ of training samples is small relative to the dimension $d$,
model information is approximately constant in $N$ and generated samples
remain strongly dependent on the particular training dataset. Around
$N=d$, this dependence decreases rapidly. With sufficiently
many samples, model information enters a generalization regime in which it
decays as $1/N$. An example of the resulting model-information curve is shown
in Figure~\ref{fig:linear_models}.

\begin{figure}[t]
    \centering

    \begin{tikzpicture}
        \begin{axis}[
            width=0.60\linewidth,
            height=0.43\linewidth,
            xmode=log,
            ymode=log,
            xmin=1e-3,
            xmax=1e1,
            ymin=5e-2,
            ymax=2,
            xlabel={$N/d$},
            ylabel={model information per dimension, $I/d$},
            xtick={1e-3,1e-2,1e-1,1,10},
            ytick={1e-2,1e-1,1},
            grid=major,
            grid style={line width=0.1pt, draw=gray!24},
            tick label style={font=\scriptsize},
            label style={font=\scriptsize},
            axis line style={black!65},
            tick style={black!65},
            scale only axis,
            enlargelimits=false,
            trim axis left,
            trim axis right
        ]

        \fill[gray!28]
            (axis cs:1e-3,5e-2)
            rectangle
            (axis cs:1,2);

        \fill[gray!12]
            (axis cs:1,5e-2)
            rectangle
            (axis cs:10,2);

        \addplot[
            black,
            very thick
        ] table[row sep=\\] {
            x          y \\
            0.0010     1.5218 \\
            0.0013     1.5216 \\
            0.0017     1.5214 \\
            0.0022     1.5210 \\
            0.0028     1.5203 \\
            0.0036     1.5192 \\
            0.0046     1.5173 \\
            0.0058     1.5144 \\
            0.0073     1.5106 \\
            0.0089     1.5070 \\
            0.0100     1.5058558 \\
            0.0117478  1.5033520 \\
            0.0138010  1.5004742 \\
            0.0162131  1.4971683 \\
            0.0190467  1.4933726 \\
            0.0223756  1.4890169 \\
            0.0262863  1.4840217 \\
            0.0308806  1.4782966 \\
            0.0362777  1.4717394 \\
            0.0426182  1.4642344 \\
            0.0500668  1.4556512 \\
            0.0588173  1.4458428 \\
            0.0690972  1.4346441 \\
            0.0811737  1.4218702 \\
            0.0953609  1.4073144 \\
            0.112028   1.3907468 \\
            0.131607   1.3719124 \\
            0.154609   1.3505303 \\
            0.181631   1.3262923 \\
            0.213376   1.2988637 \\
            0.250669   1.2678837 \\
            0.294480   1.2329690 \\
            0.345948   1.1937191 \\
            0.406411   1.1497257 \\
            0.477442   1.1005877 \\
            0.560887   1.0459339 \\
            0.658917   0.9854580 \\
            0.774079   0.9189702 \\
            0.909370   0.8464752 \\
            1.06831    0.7682867 \\
            1.25502    0.6851969 \\
            1.47437    0.5987151 \\
            1.73205    0.5113655 \\
            2.03477    0.4269005 \\
            2.39040    0.3499264 \\
            2.80818    0.2842584 \\
            3.29899    0.2309785 \\
            3.87557    0.1886641 \\
            4.55293    0.1550838 \\
            5.34867    0.1282325 \\
            6.28349    0.1065611 \\
            7.38169    0.0889178 \\
            8.67183    0.0744465 \\
            10.0000    0.0637 \\
        };

        \node[
            font=\small\bfseries,
            black!80,
            anchor=north west
        ] at (axis cs:0.0025,1.45) {Memorization};
        
        \node[
            font=\small\bfseries,
            black!80,
            anchor=north east
        ] at (axis cs:11.00,1.45) {Generalization};

        \end{axis}
    \end{tikzpicture}

    \caption{
        Model information per dimension in the linear-Gaussian model in the
        $d\to\infty$ limit, shown as a function of $N/d$. The darker and
        lighter shaded regions indicate the memorization and generalization
        regimes.
    }
    \label{fig:linear_models}
\end{figure}
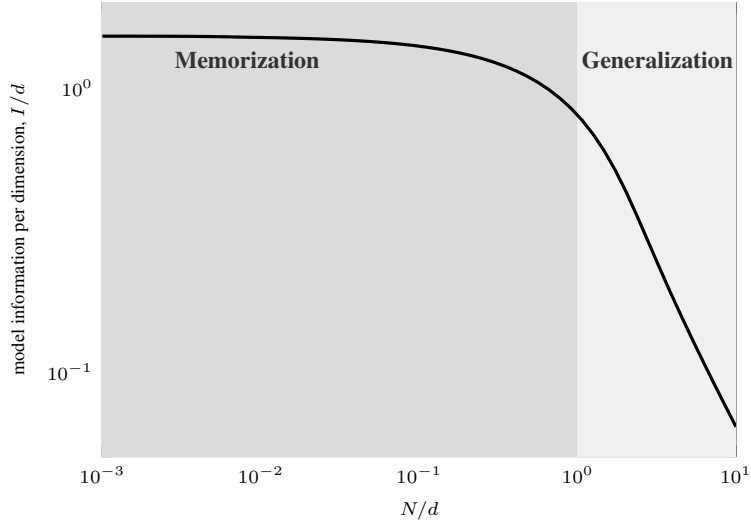

The connection with neural networks is limited, but we use this
toy model because its model information is analytically tractable.
For example, the trained model does not literally reproduce individual
training examples. Instead, for $N<d$, the learned
covariance retains a strongly dataset-specific subspace, which leads to
divergent model information in the zero-noise limit. The fact that even this
simple mechanism is sufficient to produce a high-information regime, a
capacity-dependent transition, and a low-information scaling regime suggests
that these phenomena may be more general. We therefore hope that this toy
model provides some insight into the qualitative shape of the scaling laws
observed in more complex generative models.

\subsection{Problems Setup}

Let the training dataset be
\begin{equation}
    \mathcal{D}=\{x_\mu\}_{\mu=1}^N,
    \qquad
    x_\mu \overset{\mathrm{i.i.d.}}{\sim} \mathcal N(0,I_d),
\end{equation}
and define the empirical covariance
\begin{equation}
    S_D
    =
    \frac{1}{N}\sum_{\mu=1}^N x_\mu x_\mu^\top
    =
    \frac{1}{N}XX^\top,
\end{equation}
where
\begin{equation}
    X=[x_1,\ldots,x_N]\in\mathbb R^{d\times N}.
\end{equation}
We take the learned model to be
\begin{equation}
    \Phi=S_D.
\end{equation}
A sample from the learned model therefore satisfies
\begin{equation}
    \widehat X\mid\Phi
    \sim
    \mathcal N(0,\Phi).
\end{equation}

When $N<d$, the empirical covariance is singular. Therefore, as in our
experiments, we evaluate model information at a later diffusion time, where
a small amount of Gaussian smoothing has been added:
\begin{equation}
    \widehat X_\rho
    =
    \widehat X+\sqrt{\rho}\,\epsilon,
    \qquad
    \epsilon\sim\mathcal N(0,I_d),
\end{equation}
so that
\begin{equation}
    \widehat X_\rho\mid\Phi
    \sim
    \mathcal N(0,\Phi+\rho I_d).
\end{equation}

We write
\begin{equation}
    I_\rho(N,d)
    :=
    I(\Phi;\widehat X_\rho).
    \label{eq:linear_gaussian_information_definition}
\end{equation}

The assumption of unit population covariance is without loss of generality
within the isotropic family. If instead
\begin{equation}
    x_\mu\sim\mathcal N(0,\tau^2I_d),
\end{equation}
then rescaling by $\tau$ reduces the problem to the same model with effective
noise $\rho/\tau^2$.

Since the conditional distribution is Gaussian,
\begin{equation}
    H[\widehat X_\rho\mid\Phi]
    =
    \frac{d}{2}\log(2\pi e)
    +
    \frac{1}{2}
    \mathbb E_D
    \left[
        \log\det(S_D+\rho I_d)
    \right].
    \label{eq:linear_gaussian_conditional_entropy}
\end{equation}
Hence
\begin{equation}
    I_\rho(N,d)
    =
    H[\widehat X_\rho]
    -
    \frac{d}{2}\log(2\pi e)
    -
    \frac{1}{2}
    \mathbb E_D
    \left[
        \log\det(S_D+\rho I_d)
    \right].
    \label{eq:linear_gaussian_information_entropy}
\end{equation}
The remaining difficulty is the marginal entropy $H[\widehat X_\rho]$.
In the isotropic Gaussian case, this marginal distribution has a simple
one-dimensional representation.

\subsection{Exact calculations}

Let
\begin{equation}
    Z\in\mathbb R^{d\times N}
\end{equation}
have independent standard Gaussian entries. Then
\begin{equation}
    S_D=\frac{1}{N}ZZ^\top.
\end{equation}
Conditionally on $Z$, a sample from the model can equivalently be written as
\begin{equation}
    \widehat X_\rho
    =
    \frac{1}{\sqrt N}Zg+\sqrt{\rho}\,\epsilon,
    \label{eq:linear_gaussian_alternative_sampling}
\end{equation}
where
\begin{equation}
    g\sim\mathcal N(0,I_N),
    \qquad
    \epsilon\sim\mathcal N(0,I_d),
\end{equation}
independently. Indeed, conditioned on $Z$,
\begin{equation}
    \frac{1}{\sqrt N}Zg \mid Z
    \sim
    \mathcal N\left(0,\frac{1}{N}ZZ^\top\right).
\end{equation}

We may instead condition on $g$. Since the rows of $Z$ are independent
standard Gaussian vectors,
\begin{equation}
    \frac{1}{\sqrt N}Zg \mid g
    \sim
    \mathcal N\left(
        0,
        \frac{\|g\|^2}{N}I_d
    \right).
\end{equation}
Define
\begin{equation}
    U
    :=
    \frac{\|g\|^2}{N}
    \sim
    \frac{1}{N}\chi_N^2.
\end{equation}
We obtain the exact marginal representation
\begin{equation}
    \widehat X_\rho\mid U
    \sim
    \mathcal N\left(0,(U+\rho)I_d\right).
    \label{eq:linear_gaussian_scale_mixture}
\end{equation}
Thus, after averaging over the random training dataset, the mixture over
random Wishart covariance matrices reduces to a one-dimensional Gaussian
scale mixture.

The density of $U$ is
\begin{equation}
    f_N(u)
    =
    \frac{(N/2)^{N/2}}{\Gamma(N/2)}
    u^{N/2-1}
    e^{-Nu/2},
    \qquad u>0.
\end{equation}
The marginal distribution of $\widehat X_\rho$ is rotationally invariant.
Writing
\begin{equation}
    q_\rho(r)
    =
    \int_0^\infty
    \frac{
        \exp\left(
            -\frac{r^2}{2(u+\rho)}
        \right)
    }{
        [2\pi(u+\rho)]^{d/2}
    }
    f_N(u)\,du,
    \label{eq:linear_gaussian_radial_density}
\end{equation}
we have,
\begin{equation}
    H[\widehat X_\rho]
    =
    -
    |\mathbb S^{d-1}|
    \int_0^\infty
    r^{d-1}
    q_\rho(r)
    \log q_\rho(r)\,dr,
    \label{eq:linear_gaussian_marginal_entropy}
\end{equation}
where
\begin{equation}
    |\mathbb S^{d-1}|
    =
    \frac{2\pi^{d/2}}{\Gamma(d/2)}.
\end{equation}

Equations~\eqref{eq:linear_gaussian_conditional_entropy}
and~\eqref{eq:linear_gaussian_marginal_entropy} therefore provide an
exact finite-dimensional representation of the model information.

\subsection{Asymptotic behavior of model information}

We now consider the proportional asymptotic limit
\begin{equation}
    N,d\rightarrow\infty,
    \qquad
    \frac{N}{d}\rightarrow\eta\in(0,\infty).
\end{equation}

\begin{proposition}
    \label{prop:linear_gaussian_asymptotic}
    Fix $\rho>0$ and suppose that $N,d\rightarrow\infty$ with
    $N/d = \eta$. Then
    \begin{equation}
        \frac{1}{d}I_\rho(N,d)
        \longrightarrow
        \mathcal I_\rho(\eta),
        \label{eq:linear_gaussian_asymptotic_limit}
    \end{equation}
    where
    \begin{equation}
        \mathcal I_\rho(\eta)
        =
        \frac{1}{2}
        \left[
            \log(1+\rho)
            -
            \int
            \log(\lambda+\rho)
            \,d\mu_\eta(\lambda)
        \right],
        \label{eq:linear_gaussian_mp_information}
    \end{equation}
    and $\mu_\eta$ is the Marchenko--Pastur distribution associated with the
    eigenvalues of $N^{-1}XX^\top$.
\end{proposition}

For the convention $\eta=N/d$, the Marchenko--Pastur law is
\begin{equation}
    d\mu_\eta(\lambda)
    =
    (1-\eta)_+\delta_0
    +
    \frac{\eta}{2\pi\lambda}
    \sqrt{
        (\lambda_+-\lambda)
        (\lambda-\lambda_-)
    }
    \,
    \mathbf 1_{[\lambda_-,\lambda_+]}(\lambda)
    \,d\lambda,
    \label{eq:linear_gaussian_mp_distribution}
\end{equation}
where
\begin{equation}
    \lambda_\pm
    =
    \left(
        1\pm\frac{1}{\sqrt\eta}
    \right)^2.
\end{equation}

\begin{proof}
    The conditional entropy follows from convergence of the empirical spectral
    distribution of $S_D$. In particular,
    \begin{equation}
        \frac{1}{d}
        \log\det(S_D+\rho I_d)
        \longrightarrow
        \int
        \log(\lambda+\rho)
        \,d\mu_\eta(\lambda).
        \label{eq:linear_gaussian_logdet_limit}
    \end{equation}

    It remains to characterize the marginal entropy. From
    Equation~\eqref{eq:linear_gaussian_scale_mixture},
    \begin{equation}
        \operatorname{Cov}(\widehat X_\rho)
        =
        (1+\rho)I_d.
    \end{equation}
    Let
    \begin{equation}
        G
        \sim
        \mathcal N(0,(1+\rho)I_d).
    \end{equation}
    Since $G$ has the same mean and covariance as $\widehat X_\rho$, the
    cross-entropy of $\widehat X_\rho$ with respect to $G$ equals the entropy of
    $G$. Therefore,
    \begin{equation}
        D_{\mathrm{KL}}
        \left(
            p_{\widehat X_\rho}
            \,\Vert\,
            p_G
        \right)
        =
        H[G]-H[\widehat X_\rho].
    \end{equation}
    By convexity of KL divergence in its first argument,
    \begin{equation}
        D_{\mathrm{KL}}
        \left(
            p_{\widehat X_\rho}
            \,\Vert\,
            p_G
        \right)
        \leq
        \mathbb E_U
        D_{\mathrm{KL}}
        \left(
            \mathcal N(0,(U+\rho)I_d)
            \,\Vert\,
            \mathcal N(0,(1+\rho)I_d)
        \right).
    \end{equation}
    Define
    \begin{equation}
        a_U
        =
        \frac{U+\rho}{1+\rho}.
    \end{equation}
    The Gaussian KL divergence is
    \begin{equation}
        D_{\mathrm{KL}}
        \left(
            \mathcal N(0,(U+\rho)I_d)
            \,\Vert\,
            \mathcal N(0,(1+\rho)I_d)
        \right)
        =
        \frac{d}{2}
        \left(
            a_U-1-\log a_U
        \right).
    \end{equation}
    Since
    \begin{equation}
        \mathbb E[U]=1,
        \qquad
        \operatorname{Var}(U)=\frac{2}{N},
    \end{equation}
    a Taylor expansion around $U=1$ gives
    \begin{equation}
        \mathbb E
        \left[
            a_U-1-\log a_U
        \right]
        =
        O\left(\frac{1}{N}\right).
    \end{equation}
    Therefore,
    \begin{equation}
        H[G]-H[\widehat X_\rho]
        =
        O\left(\frac{d}{N}\right)
        =
        O(1)
    \end{equation}
    since $\frac{d}{N} = \frac{1}{\eta}$ is kept constant. Dividing
    by $d$ yields
    \begin{equation}
        \frac{1}{d}
        \left(
            H[G]-H[\widehat X_\rho]
        \right)
        \longrightarrow 0.
    \end{equation}
    Hence
    \begin{equation}
        \frac{1}{d}
        H[\widehat X_\rho]
        \longrightarrow
        \frac{1}{2}
        \log\left(
            2\pi e(1+\rho)
        \right).
    \end{equation}
    Combining this expression with
    Equation~\eqref{eq:linear_gaussian_logdet_limit} proves
    Equation~\eqref{eq:linear_gaussian_mp_information}.
\end{proof}

Hence, Equation~\eqref{eq:linear_gaussian_mp_information} gives three
qualitatively
different regimes.

\paragraph{The regime $N\ll d$.}

Consider first
\begin{equation}
    \eta=\frac{N}{d}\ll1.
\end{equation}
Then
\begin{equation}
    \mathcal I_\rho(\eta)
    \longrightarrow
    \frac{1}{2}
    \log\left(
        1+\frac{1}{\rho}
    \right),
    \qquad
    \eta\rightarrow0.
    \label{eq:linear_gaussian_plateau}
\end{equation}
Thus model information approaches a finite limit when the number of
training samples is much smaller than the dimension.

The role of the evaluation noise becomes especially clear as
$\rho\rightarrow0$. For $\eta<1$, the Marchenko--Pastur distribution
contains an atom of mass $1-\eta$ at zero,
\begin{equation}
    \mu_\eta
    =
    (1-\eta)\delta_0
    +
    \mu_\eta^+.
\end{equation}
It follows that
\begin{equation}
    \mathcal I_\rho(\eta)
    =
    \frac{1-\eta}{2}
    \log\frac{1}{\rho}
    +
    O(1),
    \qquad
    \rho\rightarrow0.
    \label{eq:linear_gaussian_small_rho}
\end{equation}

At exactly $\rho=0$, model information diverges whenever $N<d$.
Hence, a positive diffusion time regularizes this singular dataset dependence.

This regime therefore provides a simple analogue of memorization: the model
retains strong information about the particular dataset from which it was
learned, although this information is encoded in a learned subspace rather
than in copies of individual training examples.

\paragraph{The regime $N\approx d$.}

The qualitative transition occurs around
\begin{equation}
    N\simeq d,
\end{equation}
or equivalently
\begin{equation}
    \eta\simeq1.
\end{equation}
For $N<d$, a fraction $1-N/d$ of the eigenvalues of the empirical
covariance are zero. For $N\ge d$, the empirical covariance has full rank
almost surely. For $\rho>0$, $\mathcal I_\rho(\eta)$ is continuous in
$\eta$, so model information changes continuously across $\eta=1$.

Thus, around $N\sim d$, the model moves from a strongly dataset-dependent
regime to one in which independently trained models increasingly agree on
the population covariance.

\paragraph{Generalization regime.}

We now consider the regime
\begin{equation}
    \eta=\frac{N}{d}\gg1.
\end{equation}
In this limit, the empirical covariance concentrates around the population
covariance.

Let $\lambda$ denote an eigenvalue distributed according to
$\mu_\eta$. For our convention $\eta=N/d$, the first two moments of the
Marchenko--Pastur distribution are
\begin{equation}
    \int \lambda\,d\mu_\eta(\lambda)
    =
    1,
\end{equation}
and
\begin{equation}
    \int \lambda^2\,d\mu_\eta(\lambda)
    =
    1+\frac{1}{\eta}.
\end{equation}
Therefore,
\begin{equation}
    \int
    (\lambda-1)^2
    \,d\mu_\eta(\lambda)
    =
    \frac{1}{\eta}.
    \label{eq:mp_variance}
\end{equation}

We can relate these expressions to model information. Recall from
Equation~\eqref{eq:linear_gaussian_mp_information} that
\begin{equation}
    \mathcal I_\rho(\eta)
    =
    \frac{1}{2}
    \left[
        \log(1+\rho)
        -
        \int
        \log(\lambda+\rho)
        \,d\mu_\eta(\lambda)
    \right].
\end{equation}
For $\eta\gg1$, the spectrum is concentrated near $\lambda=1$. Writing
\begin{equation}
    \lambda=1+\delta,
\end{equation}
we expand
\begin{align}
    \log(\lambda+\rho)
    &=
    \log(1+\rho+\delta)
    \\
    &=
    \log(1+\rho)
    +
    \frac{\delta}{1+\rho}
    -
    \frac{\delta^2}{2(1+\rho)^2}
    +
    O(\delta^3).
    \label{eq:log_mp_expansion}
\end{align}

Averaging over the Marchenko--Pastur distribution, the linear term vanishes
because
\begin{equation}
    \int
    \delta
    \,d\mu_\eta
    =
    \int
    (\lambda-1)
    \,d\mu_\eta
    =
    0.
\end{equation}
The first non-vanishing contribution is therefore quadratic. Using
Equation~\eqref{eq:mp_variance},
\begin{align}
    \int
    \log(\lambda+\rho)
    \,d\mu_\eta(\lambda)
    &=
    \log(1+\rho)
    -
    \frac{1}{2(1+\rho)^2}
    \int
    (\lambda-1)^2
    \,d\mu_\eta(\lambda)
    +
    O(\eta^{-2})
    \\
    &=
    \log(1+\rho)
    -
    \frac{1}{2(1+\rho)^2}
    \frac{1}{\eta}
    +
    O(\eta^{-2}).
\end{align}
Substituting this into the expression for model information gives
\begin{equation}
    \mathcal I_\rho(\eta)
    =
    \frac{1}{4(1+\rho)^2}
    \frac{1}{\eta}
    +
    O(\eta^{-2}).
    \label{eq:linear_gaussian_large_eta}
\end{equation}

Since $\eta=N/d$, at fixed effective model dimension we therefore obtain
\begin{equation}
    I_\rho(N,d)
    \propto
    N^{-1}.
    \label{eq:linear_gaussian_inverse_N}
\end{equation}

The Marchenko--Pastur law implies that the
typical spectral fluctuations of the empirical covariance are of order
$\eta^{-1/2}$. Since the empirical covariance is centered on the population
covariance, the linear contribution to model information vanishes, and the
leading contribution is quadratic in these fluctuations.

This observation suggests a more general interpretation of the scaling
regime beyond the linear-Gaussian model. Suppose that, after leaving the
memorization regime, models trained on different datasets fluctuate around
a common population solution, and that the standard deviation of these
dataset-induced fluctuations scales as $N^{-\gamma/2}$, then the
model information scales as $I_m(N) \propto N^{-\gamma}$.

In the present model, the central-limit scaling of the empirical covariance
gives $\gamma=1$, and the measurements on trained networks are
consistent with this: in the $N^{-1}$ regime on ImageNet, model
information minus its estimate at $N=N_{\mathrm{full}}$ has
log-log slope between
$-0.91$ and $-1.02$ in $N$ (Section~\ref{sub:information_scaling_law}).

\subsection{From the Linear-Gaussian Model to the Scaling Law}
\label{sub:linear_gaussian_to_law}

The scaling law \eqref{eq:linear_gaussian_information_law} is built
from the linear-Gaussian model in two steps, which we derive in turn. We then
describe a finite-population correction that we apply on the datasets whose
largest training set is the full dataset.

\paragraph{The small-noise limit.} The measured maximum of model
information over $N$ is far above its values in the generalization regime,
and in \eqref{eq:linear_gaussian_mp_information} this requires small $\rho$.
In this limit
$\mathcal I_\rho$ has a closed form.

\begin{proposition}
    \label{prop:linear_gaussian_small_rho}
    For $\eta\in(0,1)\cup(1,\infty)$, as $\rho\to0$,
    \begin{equation*}
        \mathcal I_\rho(\eta)
        =
        \mathcal I_\ell(\eta)+o(1),
        \qquad
        \ell=\tfrac12\log(1/\rho),
    \end{equation*}
    where $\mathcal I_\ell$ is given by
    \eqref{eq:linear_gaussian_information_profile}.
\end{proposition}

\begin{proof}
    Write $y=1/\eta$. For $y<1$, the Marchenko--Pastur law $\mu_\eta$
    has no atom
    and its support $[\lambda_-,\lambda_+]$ has $\lambda_->0$, so by dominated
    convergence $\int\log(\lambda+\rho)\,d\mu_\eta\to\int\log\lambda\,d\mu_\eta
    =-1+(1-y^{-1})\log(1-y)$ \citep{bai2010spectral}. With
    $\log(1+\rho)\to0$, \eqref{eq:linear_gaussian_mp_information} gives
    $\mathcal I_\rho(\eta)\to\frac12[1+(\eta-1)\log(1-\eta^{-1})]$.

    For $\eta<1$, write $\mu_\eta=(1-\eta)\delta_0+\eta\nu_\eta$. The nonzero
    eigenvalues of $N^{-1}XX^\top$ are those of
    $N^{-1}X^\top X=(d/N)\,d^{-1}X^\top X$, and $d^{-1}X^\top X$ is the
    sample covariance of $d$ samples in dimension $N$, whose spectrum
    converges to
    $\mu_{1/\eta}$. Hence $\nu_\eta$ is the law of $\eta^{-1}\lambda$ with
    $\lambda\sim\mu_{1/\eta}$, which is supported away from zero, and
    by the first
    case
    \[
        \int\log\lambda\,d\nu_\eta
        =
        \log\frac1\eta-1-\frac{1-\eta}{\eta}\log(1-\eta).
    \]
    Therefore
    \[
        \int\log(\lambda+\rho)\,d\mu_\eta
        =
        (1-\eta)\log\rho
        -\eta\log\eta-\eta-(1-\eta)\log(1-\eta)+o(1),
    \]
    and substituting into \eqref{eq:linear_gaussian_mp_information} gives the
    claim.
\end{proof}

Both branches equal $\frac12$ at $\eta=1$. The limit is not uniform near
$\eta=1$, but the fitted values of $\ell$ correspond to $\rho<10^{-9}$, where
the difference is negligible at every \((P,N)\) pair in our experiments. For $\eta\to\infty$ the
second branch is $1/(4\eta)+1/(12\eta^2)+O(\eta^{-3})$, in agreement with
\eqref{eq:linear_gaussian_large_eta}.

\paragraph{Independent groups.} Partition the $d$ coordinates into groups
$G_1,\ldots,G_m$ of sizes $d_1,\ldots,d_m$, and let the learned model
be block diagonal, $\Phi=\operatorname{diag}(\Phi_1,\ldots,\Phi_m)$,
with $\Phi_j$ the empirical
covariance of the coordinates in $G_j$. The $\Phi_j$ are functions of
disjoint sets of independent coordinates, so they are independent. Given
$\Phi$, the blocks $\widehat X_{\rho,j}\sim\mathcal N(0,\Phi_j+\rho I)$ are
independent. Hence the pairs $(\Phi_j,\widehat X_{\rho,j})$ are independent
and
\[
    I(\Phi;\widehat X_\rho)
    =
    \sum_{j=1}^m I(\Phi_j;\widehat X_{\rho,j})
    \approx
    \sum_{j=1}^m d_j\,\mathcal I_\ell(N/d_j)
\]
by Propositions~\ref{prop:linear_gaussian_asymptotic}
and~\ref{prop:linear_gaussian_small_rho}. In the scaling law
\eqref{eq:linear_gaussian_information_law}, the group sizes are
log-normal, $d_j\sim N_{\mathrm c}(P)e^{\sigma z}$, and the number of
groups $m$ is absorbed into $k\widetilde P^{-b}$, so that
$\sum_jd_j\,\mathcal I_\ell(N/d_j)=m\,\E_z[d_z\,\mathcal I_\ell(N/d_z)]$. With a single group size,
the decrease of model information at $N=d$ is sharp; with spread $\sigma$, it extends to about
a factor $e^{2\sigma}$ on either side of $N_{\mathrm c}(P)$.

\paragraph{The finite-population correction.} The scaling law describes
training sets drawn from the data distribution. In the experiments, each
training set is drawn without replacement from a fixed full dataset of
$N_{\mathrm{full}}$ images. Given the full dataset, $S_D$ fluctuates around
the covariance of the full dataset, and each entry of $S_D$ has variance
$\frac1N\frac{N_{\mathrm{full}}-N}{N_{\mathrm{full}}-1}$ times its population
value \citep{cochran1977sampling}. For $N_{\mathrm{full}}\gg1$ this is
$N^{-1}-N_{\mathrm{full}}^{-1}$. In the generalization regime model
information is quadratic in these fluctuations (the derivation of
\eqref{eq:linear_gaussian_large_eta}), so the leading term becomes
$\frac d4(N^{-1}-N_{\mathrm{full}}^{-1})$. On CIFAR-10, Fashion-MNIST, and
FFHQ, the largest training set is the full dataset, and we replace $N/d_z$ by
$1/(d_z(N^{-1}-N_{\mathrm{full}}^{-1}))$ in
\eqref{eq:linear_gaussian_information_law}. This substitution is exact at
leading order in the generalization regime, makes $E(P,N)$ zero at
$N=N_{\mathrm{full}}$, where every model is trained on the same data, and
changes $\eta$ by a factor $1-N/N_{\mathrm{full}}\ge0.95$ wherever
$N<N_{\mathrm c}(P)$ on these datasets. On ImageNet every training-set size is
at most $0.08\,N_{\mathrm{full}}$, and we use
\eqref{eq:linear_gaussian_information_law} as written; the correction does not
change the fit there (Table~\ref{tab:appendix_law_ablation}).

\section{Additional Experimental Results}\label{sec:experiments}

The main text makes three claims on ImageNet: the scaling law
\eqref{eq:linear_gaussian_information_law} describes model information, model
information tracks memorization for small \(N\) and FID for large \(N\), and
together these relationships predict memorization and FID at a model size that
is not used in any fit. This appendix gives the evidence for each claim on all
four datasets. After the experimental protocol
(Appendix~\ref{sub:appendix_experimental_protocol}), we describe how the
scaling law is fitted (Appendix~\ref{sub:appendix_scaling_law_fit}), compare
each fitted term with the linear-Gaussian model
(Appendices~\ref{sub:appendix_scaling_law_derivation}
and~\ref{sub:appendix_scaling_law_bias}), and compare the scaling law with
purely empirical alternatives (Appendix~\ref{sub:appendix_law_comparison}). We
then relate model information to memorization and FID
(Appendix~\ref{sub:appendix_imagenet_details}) and test the resulting
predictions (Appendix~\ref{sub:appendix_extrapolation}).
Appendix~\ref{sub:appendix_cifar} shows the measurements on CIFAR-10,
Fashion-MNIST, and FFHQ, Appendix~\ref{sub:appendix_bank_sensitivity} checks
that the conclusions do not depend on the number of trained models used by the
estimator, and Appendix~\ref{sub:appendix_visualization} shows samples across
the transition.

\subsection{Experimental Protocol}
\label{sub:appendix_experimental_protocol}

\paragraph{Data.} We use four image datasets, summarized in
Table~\ref{tab:appendix_grids}. Each image is converted to RGB, resized by
bicubic interpolation with a center crop, and mapped to $[-1,1]^d$;
Fashion-MNIST is grayscale replicated across three channels, so it shares
CIFAR-10's dimension at half the resolution of ImageNet and FFHQ. For each
$(P,N)$ we draw the training subsets independently and without replacement
from the full training split of $N_{\mathrm{full}}$ images, so subsets may
overlap across models. The dataset draw and all training randomness use
recorded seeds. At $N=N_{\mathrm{full}}$ every subset is the full dataset. On CIFAR-10,
Fashion-MNIST, and FFHQ the largest training-set size $N_{\max}$ is
$N_{\mathrm{full}}$. On ImageNet the largest training-set size in the list,
$N_{\max}=100{,}000$, is $7.9\%$ of the full dataset, and we separately train
models on the full dataset (with minibatches sampled from all
$N_{\mathrm{full}}$ images). On every dataset, the models at
$N=N_{\mathrm{full}}$ are excluded from all fits, and we use them only to
check the fitted $F(P)$.

\begin{table}[htbp!]
    \centering
    \caption{\small \textbf{The four image datasets.} $M$ is the number of
        models trained per $(P,N)$ pair, split into
        $K_{\mathrm{train}}$ reference models that fit the mixture score and
        $K_{\mathrm{val}}$ held-out models that are evaluated against it. We
        sweep over the listed values of $K_{\mathrm{train}}$ in
        Appendix~\ref{sub:appendix_bank_sensitivity}.
        $N_{\max}$ is the largest training-set size in the list of values of
        $N$, and $N_{\mathrm{full}}$ is the size of the full dataset from
        which every training set is drawn; on ImageNet we also train models on
        the full dataset at each model size.}
    \label{tab:appendix_grids}
    \small
    \resizebox{\linewidth}{!}{%
        \begin{tabular}{lrrrrrrrrl}
            \toprule
            dataset & resolution & \(d\) & \(\sigma_{\mathrm{data}}\) &
            \(N\) values & \(N_{\max}\) & \(N_{\mathrm{full}}\) & depths \(L\) &
            \(M\) & \(K_{\mathrm{train}}\) / \(K_{\mathrm{val}}\) \\
            \midrule
            ImageNet & \(64\times64\) & 12288 & 0.55 & 12 & 100{,}000 &
            1{,}271{,}167 & 2,3,4,5 & 4 & \(\{1,2\}\) / 2 \\
            CIFAR-10 & \(32\times32\) & 3072 & 0.52 & 14 & 50{,}000 &
            50{,}000 & 1,2,3,4,5 & 4 & \(\{1,2\}\) / 2 \\
            Fashion-MNIST & \(32\times32\) & 3072 & 0.80 & 7 & 60{,}000 &
            60{,}000 & 1,2,4 & 10 & \(\{2,3,4,5\}\) / 5 \\
            FFHQ & \(64\times64\) & 12288 & 0.55 & 7 & 60{,}000 &
            60{,}000 & 1,2,4 & 4 & \(\{1,2\}\) / 2 \\
            \bottomrule
        \end{tabular}
    }
\end{table}

On ImageNet we use log-spaced training-set sizes,
\[
    N\in\{10,63,400,630,1000,1600,2500,4000,6300,16000,40000,100000\},
\]
on CIFAR-10 we use log-linear-spaced training-set sizes
\[
    N\in\{10,41,170,710,2900,6500,9000,12000,16000,25000,35000,40000,45000,50000\},
\]
and on Fashion-MNIST and FFHQ we use
\(N\in\{10,43,180,780,3300,14000,60000\}\).

\paragraph{Models.}
The denoiser is an unconditional DDPM++ UNet
\citep{song2021scorebasedgenerativemodelingstochastic} under EDM
preconditioning \citep{karras2022elucidating}, predicting $x_0$. Channel
multipliers are $(1,2,2)$ so that the trunk halves twice, with a single
self-attention block at the $16\times16$ resolution. Depth $L$ sets both the
width $h=32L$ and the number of residual blocks per resolution level, so
$P$ grows roughly as $h^3$. The five depths give
\[
    P\in\{1.1 \times 10^{6},\ 6.7 \times 10^{6},\ 2.0 \times
    10^{7},\ 4.4 \times 10^{7},\ 8.3 \times 10^{7}\},
\]
on ImageNet we use $L\in\{2,3,4,5\}$, and on Fashion-MNIST and
FFHQ we use $L\in\{1,2,4\}$. The parameter count does not depend on the input resolution, so the same architectures appear on all four datasets.

\paragraph{Training.} All models use the EDM noise schedule with
$\sigma\in[0.002,80]$, $\rho=7$, and the per-dataset $\sigma_{\mathrm{data}}$
of Table~\ref{tab:appendix_grids}. Diffusion times are sampled from a
logit-normal distribution in log signal-to-noise ratio with location $2.4$ and
scale $2.4$. The batch size is $B=\min(128,N)$ and the number of training steps is
$\max(10000,\,200\lfloor N/B\rfloor)$, so runs below $N=128$ see $10^4 N$
images and runs above it are trained for at least $200$ epochs. We use Muon
with the spectral $\mu$P parameterization, peak learning rate
$3\times10^{-4}$, no weight decay, and no dropout. Training is in BF16. The
ImageNet models on the full dataset are trained for the same number of steps as
at $N=100{,}000$, so they differ from the models at $N=100{,}000$ only in the
training data.
Evaluation uses an exponential moving average whose half-life is a fixed
fraction of the number of training steps, giving decay $0.99922$ at $10^4$ steps. Samples
are generated with a second-order Heun sampler, $18$ steps at $32\times32$ and
$40$ at $64\times64$.

\paragraph{Model information evaluation.} The main text reports
$K_{\mathrm{train}}=2$ and $K_{\mathrm{val}}=2$ on ImageNet: models zero and
one fit the mixture score and models two and three are held out for
evaluation. CIFAR-10 and FFHQ use the same split. Reported estimates
use the model-free classifier estimator. We
generate $8192$ samples from each reference model for the estimator computations
and $256$ evaluation samples from every trained model. The estimator evaluates
$25$ diffusion times spaced uniformly in $t\in[0.05,0.99]$. At each time it averages over $64$ held-out samples as Monte Carlo
evaluation points, and it evaluates the integral in
\eqref{eq:finite_bank_diffusion_density_identity} by quadrature at $16$
log-spaced signal-to-noise ratios $\gamma$, with four shared Gaussian draws
per quadrature node.
The same validation samples and forward noise draws are used throughout. We
integrate the resulting rate according to
\eqref{eq:finite_interval_information_estimate}. On Fashion-MNIST we train
ten models per $(P,N)$ pair rather than four and run all three estimators of
Appendix~\ref{sec:model_information_computations} on the same evaluation
points, which allows the comparison of
Appendix~\ref{sub:appendix_bank_sensitivity}.

\paragraph{Memorization.}
We report a nearest-neighbor diagnostic alongside model
information. For each trained model, we generate $256$ clean samples and scan
its complete finite training set. A sample $\hat{X}$ is counted as memorized
if its nearest training example \(X^{(1)}(\hat{X})\) has distance less than
one-third of the distance to its second nearest training example
\(X^{(2)}(\hat{X})\) \citep{yoon2023diffusion}, i.e.,
\begin{equation}
    \hat{X}\ \text{is memorized} \iff \frac{\|\hat{X} -
    X^{(1)}(\hat{X})\|_2}{\|\hat{X} - X^{(2)}(\hat{X})\|_2} \leq \frac{1}{3}.
    \label{eq:appendix_copy_ratio}
\end{equation}
The figures use Euclidean distance on the complete image. Reported means and
standard errors are taken across the trained models at each $(P,N)$, and
exclude samples drawn from the fitted mixture, which is not in the support of
$\Pi$.

\paragraph{Sample quality.} We report clean-FID
\citep{heusel2017gans,parmar2022aliased} between $8192$ generated images and
$50000$ reference images, computed with a fixed resizing and quantization
protocol to remove the implementation sensitivity of the standard metric. We
also report the FID between two disjoint sets of $8192$ real images. It is
$5.33$ on ImageNet, $3.92$ on CIFAR-10, and $2.44$ on both Fashion-MNIST and
FFHQ, and it approximates the smallest FID a model can attain at this number
of samples.

\subsection{Fitting the Scaling Law}
\label{sub:appendix_scaling_law_fit}

We fit \eqref{eq:linear_gaussian_information_law} by least squares in
log-information over all \((P,N)\) pairs with \(N<N_{\mathrm{full}}\). The
models at \(N=N_{\mathrm{full}}\) are not part of the fit; we use them to check
the fitted \(F(P)\) (Appendix~\ref{sub:appendix_scaling_law_bias}). We choose
\(P_{\mathrm{ref}}\) as the median model size of the dataset, evaluate the
expectation over \(z\) by \(31\)-point Gauss--Hermite quadrature, start the
optimizer from \(16\) initial points, and keep the best fit. We bound
\(|\beta|\le3\) and \(\ell\le e^5\).
Table~\ref{tab:appendix_scaling_law_coefficients} gives the fitted parameters.

The scaling law describes the measurements on all four datasets
(Table~\ref{tab:appendix_scaling_law_coefficients} and
Figures~\ref{fig:imagenet_scaling}, \ref{fig:edm_scaling},
and~\ref{fig:appendix_replication}). Its two terms are determined by different
measurements. The parameters of \(E(P,N)\) are determined mostly by the
measurements below and around the transition. \(F(P)\) is determined by the
measurements in the \(N^{-1}\) regime, where \(E(P,N)\) is a known function of
\(N\). Table~\ref{tab:appendix_scaling_law_coefficients} compares the fitted
\(F(P)\) with the estimates at \(N=N_{\mathrm{full}}\), which are not part of
the fit, on all four datasets. On CIFAR-10, \(\beta\) reaches its bound, so \(F(P)\) is constant
except at the smallest model size.

Table~\ref{tab:appendix_law_ablation} shows that the parts of the scaling law
that come from the linear-Gaussian model are needed. A single group size
(\(\sigma=0\)) cannot describe the gradual transition, and fixing \(\ell\) at
\(\log(\sigma_{\mathrm{data}}/\sigma_{t_{\min}})\), its value in the
linear-Gaussian model at the evaluation noise, cannot describe the ratio
of the measurements below and above the transition. The finite-population
correction of Appendix~\ref{sub:linear_gaussian_to_law}, which we apply on the
other three datasets, makes no difference on ImageNet, where every
training-set size in the fit is at most \(0.08\,N_{\mathrm{full}}\).

\begin{table}[htbp!]
    \centering
    \caption{\small \textbf{Scaling-law coefficients.} The nine parameters of
        \eqref{eq:linear_gaussian_information_law}: those of \(E(P,N)\),
        \((k,b,N_{\mathrm c}(P_{\mathrm{ref}}),a,\sigma,\ell)\), and those of
        \(F(P)\), \((c_0,c_1,\beta)\). \(P_{\mathrm{ref}}\) is the median model
        size of the dataset. The fit bounds are \(|\beta|\le3\) and
        \(\ell\le e^5\approx148.4\). The last two columns are the log RMSE over all
        \((P,N)\) pairs, and the log RMSE of the fitted \(F(P)\) against the
        estimates at \(N=N_{\mathrm{full}}\),
    which are not part of the fit.}
    \label{tab:appendix_scaling_law_coefficients}
    \small
    \resizebox{\linewidth}{!}{%
        \begin{tabular}{lrrrrrrrrrrrrr}
            \toprule
            dataset & \((P,N)\) pairs & \(P_{\mathrm{ref}}\) & \(k\) & \(b\) &
            \(N_{\mathrm c}(P_{\mathrm{ref}})\) & \(a\) & \(\sigma\)
            & \(\ell\) &
            \(c_0\) & \(c_1\) & \(\beta\) & log RMSE &
            \(F(P)\) error \\
            \midrule
            ImageNet UNet & 48 & \(3.2{\times}10^{7}\) & 7.51 & 0.483
            & 603 & 0.345 & 0.622 & 13.75 & 220.1 & 40.4 & 0.80 &
            0.112 & 0.088 \\
            CIFAR-10 UNet & 65 & \(2.0{\times}10^{7}\) & 0.911 &
            0.518 & 1330 & 0.409 & 0.503 & 16.59 & 110.0 &
            \(9.9{\times}10^{-3}\) & 3.00 & 0.101 & 0.066 \\
            Fashion-MNIST UNet & 18 & \(6.7{\times}10^{6}\) & 0.0942
            & 0.378 & 1093 & 0.314 & 0.642 & 148.41 & 101.4 & 0.223 &
            \(-3.00\) & 0.074 & 0.746 \\
            FFHQ UNet & 18 & \(6.7{\times}10^{6}\) & 4.41 & 0.598 &
            485 & 0.488 & 0.571 & 36.18 & 182.5 & 484 & 0.24 & 0.109 &
            0.297 \\
            \bottomrule
        \end{tabular}
    }
\end{table}

\begin{table}[htbp!]
    \centering
    \caption{\small \textbf{Log RMSE of the scaling law and of restricted
        versions of it.} ``\(\ell\) fixed'' sets
        \(\ell=\log(\sigma_{\mathrm{data}}/\sigma_{t_{\min}})\) and refits the
        other eight parameters. ``Correction switched'' adds the
        finite-population correction of Appendix~\ref{sub:linear_gaussian_to_law}
        on ImageNet and removes it on the other three datasets. On ImageNet, the
        fit with \(\sigma=0\) is off by up to a factor of two at the \((P,N)\)
        pairs around \(N=N_{\mathrm c}(P)\). A dash marks a variant that we did
    not run.}
    \label{tab:appendix_law_ablation}
    \small
    \begin{tabular}{lcccc}
        \toprule
        variant & ImageNet & CIFAR-10 & Fashion-MNIST & FFHQ \\
        \midrule
        \eqref{eq:linear_gaussian_information_law} & 0.112 & 0.101 & 0.074 &
        0.109 \\
        single group size, \(\sigma=0\) & 0.249 & 0.159 & --- & --- \\
        \(\ell\) fixed & 0.242 & 0.246 & 0.161 & 0.224 \\
        correction switched & 0.112 & 0.103 & 0.074 & 0.108 \\
        \bottomrule
    \end{tabular}
\end{table}

\subsection{The Fitted Terms and the Linear-Gaussian Model}
\label{sub:appendix_scaling_law_derivation}

Each term of \eqref{eq:linear_gaussian_information_law} is derived
from the linear-Gaussian model in
Appendix~\ref{sub:linear_gaussian_to_law}. Here we compare the fitted
terms with the values that this model and the measurement protocol give.

\paragraph{The maximum \(\hat I_{\max}(P)\).} Take a model that reproduces its
training set \(S\), and let the reference models it is evaluated against
reproduce an independently drawn set \(S'\) of the same size. At the smallest
evaluation time \(t_{\min}\) both are seen through Gaussian noise
of scale \(\sigma_{t_{\min}}\), so the two distributions are mixtures
of \(N\) isotropic Gaussian components of that width, centered on
\(S\) and on \(S'\). At
\(\sigma_{t_{\min}}=6.3\times10^{-3}\) the components are separated
by many widths
(the typical distance between two CIFAR-10 images is \(\approx 39\) in the
\([-1,1]^{3072}\) cube, against a component radius of
\(\sigma_{t_{\min}}\sqrt d\approx0.35\)), so the entropy of either mixture
is \(\tfrac d2\log(2\pi e\sigma_{t_{\min}}^2)+\log N\) up to exponentially
small corrections, and the divergence between them is governed by the
entropy of the data distribution against that of a single component:
\[
    \Lambda
    \;\approx\;
    h(p_{\mathrm{data}})-\tfrac d2\log(2\pi e\sigma_{t_{\min}}^2)-\log N
    \;\approx\;
    d\log(\sigma_{\mathrm{data}}/\sigma_{t_{\min}})-\log N ,
\]
the second step using a Gaussian approximation of \(p_{\mathrm{data}}\) with
per-pixel standard deviation \(\sigma_{\mathrm{data}}\). The \(\log N\)
correction is at most \(11\) nats in our experiments, against values of
\(\hat I_{\max}(P)\) of \(10^4\) or more, so we drop it. This estimate is the
maximum over \(N\) of model information in the linear-Gaussian model with all
\(d\) input dimensions in the groups and
\(\rho=(\sigma_{t_{\min}}/\sigma_{\mathrm{data}})^2\), since then
\(\ell=\log(\sigma_{\mathrm{data}}/\sigma_{t_{\min}})\). This value depends on the data scale and the diffusion schedule alone, so in theory it should not change with \(P\). Empirically, Table~\ref{tab:appendix_description_length}
compares the value of \(\Lambda\) read off from the problem parameters with the fitted \(\hat I_{\max}(P)-F(P)\). On every dataset and model size, \(\Lambda\) is
within a small factor of the fitted \(\hat I_{\max}(P)-F(P)\), across a large
change in the fitted value between datasets. In the empirical fit, \(\hat I_{\max}(P)-F(P)\) decreases with \(P\) as \(\widetilde P^{a-b}\), but slowly, since \(b-a\) is small on all four datasets.

\begin{table}[htbp!]
    \centering
    \caption{\small \textbf{Predicted and fitted maximum of model
        information.} The prediction
        \(\Lambda = d\log(\sigma_{\mathrm{data}}/\sigma_{t_{\min}})\) at
        \(\sigma_{t_{\min}}=6.266\times10^{-3}\), against
        \(\hat I_{\max}(P)-F(P)\) of the fitted scaling law
        \eqref{eq:linear_gaussian_information_law} over the model sizes of
        each dataset; the prediction contains no fitted quantity. The ratio is
        fitted over predicted. The fitted \(\hat I_{\max}(P)-F(P)\) is proportional
        to \(\widetilde P^{a-b}\).}
    \label{tab:appendix_description_length}
    \small
    \begin{tabular}{lrrrrr}
        \toprule
        dataset & \(d\) & predicted \(\Lambda\) &
        fitted \(\hat I_{\max}(P)-F(P)\) & ratio & \(b-a\) \\
        \midrule
        ImageNet & 12288 & \(5.50{\times}10^4\) &
        \(6.63\)--\(9.37{\times}10^4\) & 1.21--1.70 & 0.14 \\
        CIFAR-10 & 3072 & \(1.36{\times}10^4\) &
        \(1.95\)--\(3.13{\times}10^4\) & 1.44--2.30 & 0.11 \\
        Fashion-MNIST & 3072 & \(1.49{\times}10^4\) &
        \(1.66\)--\(2.10{\times}10^4\) & 1.12--1.41 & 0.06 \\
        FFHQ & 12288 & \(5.50{\times}10^4\) &
        \(7.39\)--\(11.08{\times}10^4\) & 1.34--2.01 & 0.11 \\
        \bottomrule
    \end{tabular}
\end{table}

\paragraph{Training sets from the data distribution.} The estimator evaluates
each model against \(K_{\mathrm{train}}\) reference models whose training sets
are drawn independently of the training set of the evaluated model. For
\(N\ll N_{\mathrm c}(P)\) these training sets rarely overlap, since
\(K_{\mathrm{train}}N\ll N_{\mathrm{full}}\), so the reference models stand
in for models trained on independent draws from the data distribution, and the
estimate is the quantity \(\Lambda\) above, which does not depend on
\(N_{\mathrm{full}}\).

\paragraph{The transition.} In the scaling law, model information decreases
where \(N\) passes the group sizes \(d_z\), and \(\sigma\) sets the width of
this range of \(N\). The fitted \(\sigma\) is similar on the four datasets
(Table~\ref{tab:appendix_scaling_law_coefficients}): on each, the middle
two-thirds of the groups lie within a factor \(e^{\sigma}\approx1.6\) to
\(1.9\) of \(N_{\mathrm c}(P)\). A single group size describes the
measurements much worse (Table~\ref{tab:appendix_law_ablation}). The
memorization ratio also changes at similar values of \(N/N_{\mathrm c}(P)\) on
all four datasets (Table~\ref{tab:appendix_memorization}).

\paragraph{The \(N^{-1}\) regime.} Above the transition, \(E(P,N)\) decays as
\(k\widetilde P^{-b}N_{\mathrm
c}(P)^2\,e^{2\sigma^2}/(4N)\) to leading order, with \(N^{-1}\) replaced by
\(N^{-1}-N_{\mathrm{full}}^{-1}\) where we apply the finite-population
correction. The exponent of \(N\) is fixed at \(-1\) by the central-limit
scaling of the empirical covariance, and is not a fitted parameter. The
measurements agree with it: on ImageNet, the measured model information minus
its estimate at \(N=N_{\mathrm{full}}\) has log-log slope close to \(-1\) at
every model size (Table~\ref{tab:imagenet_fit}).

\subsection{The Term \texorpdfstring{$F(P)$}{F(P)}}
\label{sub:appendix_scaling_law_bias}

For a deterministic training map that converges to a unique population
solution, $I_\model(t;P,N)$ converges to zero as $N\to\infty$ at fixed $t>0$.
In our experiments the estimates remain positive at \(N=N_{\mathrm{full}}\),
which the scaling law models by \(F(P)\). Three terms contribute to \(F(P)\).

First, the mixture-score approximation produces the integrated error
\begin{equation}
    B_{\mathrm{mix}}
    =
    \frac{1}{2}\int_{t_{\min}}^T
    g(t)^2
    \E_{X_t\sim\bar p_t}
    \|\bar s_t(X_t)-\hat{\bar s}_t(X_t)\|^2dt,
    \label{eq:appendix_integrated_mixture_bias}
\end{equation}
as follows from \eqref{eq:two_bank_rate_decomposition}. This term
contains the error from finite $K_{\mathrm{train}}$ and the numerical error of
the selected posterior estimator. It is nonnegative, so fewer reference models
should give a larger estimate;
Appendix~\ref{sub:appendix_bank_sensitivity} confirms this.

Second, the distribution of trained models includes optimization randomness.
Let $\mathcal{D}$ be the training dataset, let $\mathcal{U}$ collect
the initialization and
minibatch randomness, and write $\Phi=A(\mathcal{D},\mathcal{U})$.
Since a generated sample
$\hat X$ is conditionally independent of $\mathcal{D}$ given $\Phi$, the chain
rule gives
\begin{equation}
    I(\Phi;\hat X_t)
    =
    I(\mathcal{D};\hat X_t)
    +
    I(\Phi;\hat X_t\mid \mathcal{D}).
    \label{eq:appendix_dataset_seed_information_split}
\end{equation}
The first term measures dataset dependence. The second term measures the
variation between training runs on a fixed dataset. At $N=N_{\mathrm{full}}$, every model trains on the full dataset, so the first term is exactly zero, and the estimate there is entirely the conditional term together with the numerical contributions above.

Third, the lowest evaluation times receive negligible probability under the
training-time distribution, so the trained scores can vary substantially in
this extrapolation region.

We consequently fit \(F(P)=c_0+c_1\widetilde
P^{-\beta}\). As \(P\) grows, models trained on the same data become
more alike, but the variation between training seeds remains, so
\(F(P)\) tends to \(c_0\). Table~\ref{tab:appendix_critical_sample_sizes}
compares the fitted \(F(P)\) with the estimates at \(N=N_{\mathrm{full}}\),
which are not part of the fit. On every
dataset the estimate at \(N=N_{\mathrm{full}}\) is below one percent of
\(\hat I_{\max}(P)\), and divided by the input dimension it is similar across the
four datasets. Empirically, it varies mildly with $P$.

\begin{table}[htbp!]
    \centering
    \caption{\small \textbf{Fitted critical training-set sizes and
        \(F(P)\).} The critical
        training-set size \(N_{\mathrm c}(P)=N_{\mathrm
        c}(P_{\mathrm{ref}})\widetilde P^{a}\), its ratio to the width
        \(h\), and \(\hat I_{\max}(P)-F(P)\). Then the fitted \(F(P)\) and the
        estimate \(\hat I_\model(P,N_{\mathrm{full}})\), where all
        models share one dataset and which is not part of the fit, followed by
        this estimate divided by \(\hat I_{\max}(P)-F(P)\) and by the input
        dimension \(d\) (in nats).}
    \label{tab:appendix_critical_sample_sizes}
    \small
    \resizebox{\linewidth}{!}{%
        \begin{tabular}{llrrrrrrrr}
            \toprule
            dataset & architecture & \(P\) & \(N_{\mathrm c}(P)\) &
            \(N_{\mathrm c}/h\) &
            \(\hat I_{\max}-F\) & \(F(P)\) & \(\hat I_\model(P,N_{\mathrm{full}})\) &
            \(\hat I_\model(P,N_{\mathrm{full}})/(\hat I_{\max}-F)\) &
            \(\hat I_\model(P,N_{\mathrm{full}})/d\) \\
            \midrule
            ImageNet UNet & \(h=64,L=2\) & \(6.7{\times}10^{6}\) &
            350 & 5.47 & 93683 & 361.7 & 330.1 & 0.4\% & 0.027 \\
            ImageNet UNet & \(h=96,L=3\) & \(2.0{\times}10^{7}\) &
            511 & 5.33 & 80611 & 279.2 & 266.1 & 0.3\% & 0.022 \\
            ImageNet UNet & \(h=128,L=4\) & \(4.4{\times}10^{7}\) &
            673 & 5.26 & 72264 & 251.4 & 225.7 & 0.3\% & 0.018 \\
            ImageNet UNet & \(h=160,L=5\) & \(8.3{\times}10^{7}\) &
            836 & 5.23 & 66303 & 239.1 & 218.2 & 0.3\% & 0.018 \\
            \midrule
            CIFAR-10 UNet & \(h=32,L=1\) & \(1.1{\times}10^{6}\) &
            410 & 12.82 & 31282 & 165.8 & 148.2 & 0.5\% & 0.048 \\
            CIFAR-10 UNet & \(h=64,L=2\) & \(6.7{\times}10^{6}\) &
            850 & 13.28 & 25731 & 110.2 & 109.8 & 0.4\% & 0.036 \\
            CIFAR-10 UNet & \(h=96,L=3\) & \(2.0{\times}10^{7}\) &
            1330 & 13.85 & 22822 & 110.0 & 101.5 & 0.4\% & 0.033 \\
            CIFAR-10 UNet & \(h=128,L=4\) & \(4.4{\times}10^{7}\) &
            1842 & 14.39 & 20915 & 110.0 & 104.6 & 0.5\% & 0.034 \\
            CIFAR-10 UNet & \(h=160,L=5\) & \(8.3{\times}10^{7}\) &
            2380 & 14.88 & 19526 & 110.0 & 111.0 & 0.6\% & 0.036 \\
            \midrule
            Fashion-MNIST UNet & \(h=32,L=1\) & \(1.1{\times}10^{6}\)
            & 625 & 19.53 & 21046 & 101.4 & 81.5 & 0.4\% & 0.027 \\
            Fashion-MNIST UNet & \(h=64,L=2\) & \(6.7{\times}10^{6}\)
            & 1093 & 17.08 & 18781 & 101.7 & 56.4 & 0.3\% & 0.018 \\
            Fashion-MNIST UNet & \(h=128,L=4\) &
            \(4.4{\times}10^{7}\) & 1979 & 15.46 & 16645 &
            166.5 & 53.9 & 0.3\% & 0.018 \\
            \midrule
            FFHQ UNet & \(h=32,L=1\) & \(1.1{\times}10^{6}\) & 203 &
            6.35 & 110781 & 923.2 & 751.2 & 0.7\% & 0.061 \\
            FFHQ UNet & \(h=64,L=2\) & \(6.7{\times}10^{6}\) & 485 &
            7.58 & 90995 & 666.8 & 478.9 & 0.5\% & 0.039 \\
            FFHQ UNet & \(h=128,L=4\) & \(4.4{\times}10^{7}\) & 1220
            & 9.53 & 73852 & 491.0 & 350.8 & 0.5\% & 0.029 \\
            \bottomrule
        \end{tabular}
    }
\end{table}

\subsection{Comparison with Other Functional Forms}
\label{sub:appendix_law_comparison}

To test whether the form derived from the linear-Gaussian model matters, we
compare \eqref{eq:linear_gaussian_information_law} with two forms in which
we replace \(\mathcal I_\ell\) by a sigmoid in \(\log N\).
Both use the
term \(F(P)\) of \eqref{eq:linear_gaussian_information_law} and a critical
size \(N_{\mathrm{sig}}(P)=N_{\mathrm{ref}}\lambda\widetilde P^{\gamma}\), with
\(N_{\mathrm{ref}}\) the median training-set size of the dataset. In the first,
we multiply by a term that is linear in \(\log N\) and zero at
\(N_{\mathrm{full}}\):
\begin{equation}
    \label{eq:appendix_sigmoid_ramp}
    \hat I_\model(P,N)\approx F(P)+R(N,P)\bigg[
    \frac{C\widetilde P^{-\delta}}{1+(N/N_{\mathrm{sig}}(P))^{s}}+A\bigg],
\end{equation}
where
\begin{equation*}
    R(N,P)=\min\bigg\{1,\frac{\log(N_{\mathrm{full}}/N)}{\log(N_{\mathrm{full}}/N_{\mathrm{sig}}(P))}\bigg\}_+,
\end{equation*}
with the nine parameters \(c_0,c_1,\beta,A,C,\delta,s,\lambda,\gamma\). In the
second, we add a power law in \(N\) instead:
\begin{equation}
    \label{eq:appendix_sigmoid_power}
    \hat I_\model(P,N)\approx F(P)+
    \frac{C\widetilde P^{-\delta}}{1+(N/N_{\mathrm{sig}}(P))^{s}}
    +B\big(N_{\mathrm{ref}}(N^{-1}-N_{\mathrm{full}}^{-1})\big)^{\alpha},
\end{equation}
with the ten parameters \(c_0,c_1,\beta,C,\delta,s,\lambda,\gamma,B,\alpha\).
We fit all three forms in the same way, with \(16\) or more initial points
each, and without the models at \(N=N_{\mathrm{full}}\).
Table~\ref{tab:appendix_law_comparison} reports the log RMSE over all
\((P,N)\) pairs, the Bayesian information criterion
\(\mathrm{BIC}=n\log(\mathrm{RSS}/n)+k\log n\) for \(n\) points and \(k\)
parameters, the log RMSE of the fitted \(F(P)\) against the estimates at
\(N=N_{\mathrm{full}}\), and two tests in which we withhold some of the
\((P,N)\) pairs.

The Marchenko--Pastur form fits best: it has the lowest log RMSE and the
lowest BIC on all four datasets, with as many or fewer parameters. The forms differ most in \(F(P)\). In the
Marchenko--Pastur form the exponent of \(N\) in the \(N^{-1}\) regime is fixed
at \(-1\), so \(F(P)\) is determined by the measurements in this regime, and
on ImageNet only this form has \(F(P)\) close to the estimates at
\(N=N_{\mathrm{full}}\), which no fit uses. In the sigmoid--power form this
exponent is a fitted parameter; on ImageNet the fit makes the decay much
flatter than the measured slope near \(-1\), and its \(F(P)\) is far from the
estimates at \(N=N_{\mathrm{full}}\).

When we withhold the largest model size, the Marchenko--Pastur form has the
smallest error on ImageNet, CIFAR-10, and FFHQ. When we withhold the
training-set sizes above the median, it has the smallest error on ImageNet
and CIFAR-10.

\begin{table}[htbp!]
    \centering
    \caption{\small \textbf{Comparison of functional forms.} MP is
        \eqref{eq:linear_gaussian_information_law}; sigmoid--log is
        \eqref{eq:appendix_sigmoid_ramp}; sigmoid--power is
        \eqref{eq:appendix_sigmoid_power}. Errors are root-mean-square in
        log-information. Column \(F(P)\) compares the fitted \(F(P)\) with the
        estimates at \(N=N_{\mathrm{full}}\), which no fit uses. The last two columns give the
        error on withheld points / the error on the same points when they are
        included in the fit. In ``largest \(P\)'' we withhold every
        measurement at
        the largest model size, and in ``large \(N\)'' every dataset size
        above the median. A dash marks a test with fewer than four remaining
        dataset sizes. Lower is better in every column, and the best
        value on each
        dataset is in bold. On ImageNet the sigmoid--power form has fitted
    exponent \(\alpha=0.15\), a log-log slope of \(-0.15\).}
    \label{tab:appendix_law_comparison}
    \small
    \resizebox{\linewidth}{!}{%
        \begin{tabular}{llrrrrrr}
            \toprule
            dataset & form & params & log RMSE & BIC & \(F(P)\) &
            largest \(P\) & large \(N\) \\
            \midrule
            ImageNet & MP & 9 & \textbf{0.112} & \(\mathbf{-175.0}\)
            & \textbf{0.088} & \textbf{0.183} / 0.143 &
            \textbf{0.950} / 0.072 \\
            ImageNet & sigmoid--log & 9 & 0.141 & \(-153.0\) & 0.544
            & 0.221 / 0.173 & 1.529 / 0.085 \\
            ImageNet & sigmoid--power & 10 & 0.142 & \(-148.9\) &
            2.407 & 0.223 / 0.170 & 1.500 / 0.089 \\
            \midrule
            CIFAR-10 & MP & 9 & \textbf{0.101} & \(\mathbf{-260.8}\)
            & 0.066 & \textbf{0.189} / 0.123 & \textbf{0.105} / 0.068 \\
            CIFAR-10 & sigmoid--log & 9 & 0.102 & \(-259.3\) &
            \textbf{0.047} & 0.200 / 0.142 & 0.205 / 0.069 \\
            CIFAR-10 & sigmoid--power & 10 & 0.104 & \(-252.3\) &
            0.074 & 0.204 / 0.146 & 0.189 / 0.071 \\
            \midrule
            Fashion-MNIST & MP & 9 & \textbf{0.074} &
            \(\mathbf{-67.6}\) & 0.746 & 0.278 / 0.078 & --- \\
            Fashion-MNIST & sigmoid--log & 9 & 0.118 & \(-50.9\) &
            0.628 & 0.223 / 0.154 & --- \\
            Fashion-MNIST & sigmoid--power & 10 & 0.098 & \(-54.8\) &
            \textbf{0.573} & \textbf{0.203} / 0.116 & --- \\
            \midrule
            FFHQ & MP & 9 & \textbf{0.109} & \(\mathbf{-53.7}\) &
            0.297 & \textbf{0.400} / 0.110 & --- \\
            FFHQ & sigmoid--log & 9 & 0.154 & \(-41.3\) & 0.311 &
            0.518 / 0.151 & --- \\
            FFHQ & sigmoid--power & 10 & 0.151 & \(-39.1\) &
            \textbf{0.291} & 0.525 / 0.134 & --- \\
            \bottomrule
        \end{tabular}
    }
\end{table}

\subsection{Memorization and Sample Quality}
\label{sub:appendix_imagenet_details}

The main text relates model information to memorization for small \(N\) and to
FID in the \(N^{-1}\) regime. Table~\ref{tab:appendix_memorization} shows that
the transition in memorization is at similar values of \(N/N_{\mathrm c}(P)\)
on all four datasets: the memorization ratio halves well below
\(N_{\mathrm c}(P)\), and FID has a local maximum near \(N_{\mathrm c}(P)\), at
the first training-set sizes at which the samples are no longer copies of
training images, at every model size on ImageNet, CIFAR-10, and FFHQ, and at
the largest model size on Fashion-MNIST. Across all \((P,N)\) pairs of each dataset, model information
and the memorization ratio are strongly correlated.

Table~\ref{tab:appendix_imagenet_statistics} gives the statistics behind
Figure~\ref{fig:imagenet_correlation} on ImageNet. In the \(N^{-1}\) regime,
\(N\ge e^{2\sigma}N_{\mathrm c}(P)\), we fit
\begin{equation}
    \label{eq:fid_map}
    \log\mathrm{FID}=\omega+\kappa\log E(P,N)+\xi\log P,
\end{equation}
which has \(R^2=0.95\). 

\begin{table}[htbp!]
    \centering
    \caption{\small \textbf{Memorization and FID across the transition.} The
        training-set sizes, in units of \(N_{\mathrm c}(P)\), at which the
        memorization ratio is one half and at which FID has its local maximum,
        as ranges over model sizes (Table~\ref{tab:imagenet_fit} gives them per
        model size on ImageNet). The smallest \(N\) with a memorization ratio below
        \(1\%\) lies between \(0.64\,N_{\mathrm c}(P)\) and
        \(3.4\,N_{\mathrm c}(P)\). On ImageNet the local maximum of
        FID is at this \(N\). On ImageNet the memorization ratio at \(N=10\) is
        between \(0.84\) and \(0.89\). The last column is the Pearson
        correlation between model information and the memorization ratio over
        all \((P,N)\) pairs.}
    \label{tab:appendix_memorization}
    \small
    \begin{tabular}{lccc}
        \toprule
        dataset & memorization ratio \(\tfrac12\) & FID maximum &
        info--mem.\ correlation \\
        \midrule
        ImageNet & 0.22--0.28 & 1.1--1.5 & 0.927 \\
        CIFAR-10 & 0.36--0.42 & 0.84--2.2 & 0.965 \\
        Fashion-MNIST & 0.29--0.36 & 1.7 (\(L=4\)) & 0.952 \\
        FFHQ & 0.14--0.24 & 0.64--1.6 & 0.844 \\
        \bottomrule
    \end{tabular}
\end{table}

\begin{table}[htbp!]
    \centering
    \caption{\small \textbf{Model information, memorization, and FID on
        ImageNet.} Statistics behind Figure~\ref{fig:imagenet_correlation} and
        the predictions of Section~\ref{sub:extrapolation}. The logistic
        function has two parameters and takes the logarithm of model
        information divided by \(\hat I_{\max}(P)-F(P)\). The last three rows are
        the log RMSE of the predicted FID at \(L=5\), over the five
        training-set sizes in the \(N^{-1}\) regime, for three regressions
        fitted on \(L\in\{2,3,4\}\).}
    \label{tab:appendix_imagenet_statistics}
    \small
    \begin{tabular}{lr}
        \toprule
        statistic & value \\
        \midrule
        Pearson correlation, model information and memorization ratio (48
        pairs) & 0.93 \\
        mean absolute error of the logistic fit to the memorization ratio &
        0.020 \\
        Spearman correlation, model information and FID, \(N^{-1}\) regime (24
        pairs), each model size & \(\ge0.94\) \\
        \(R^2\) of \eqref{eq:fid_map} & 0.95 \\
        \(\kappa\), \(\xi\) in \eqref{eq:fid_map} & 0.12, \(-0.17\) \\
        \(R^2\) of \eqref{eq:fid_map} with measured \(\log I_\model\) in place
        of \(\log E\) & 0.91 \\
        \midrule
        predicted FID: \eqref{eq:fid_map}, fitted in the \(N^{-1}\) regime &
        0.066 \\
        predicted FID: \eqref{eq:fid_map}, fitted on all \(N\) & 0.285 \\
        predicted FID: regression on \(\log N\) and \(\log P\), \(N^{-1}\)
        regime & 0.070 \\
        \bottomrule
    \end{tabular}
\end{table}

\subsection{Predictions at a Withheld Model Size}
\label{sub:appendix_extrapolation}

\paragraph{ImageNet at \(L=5\).} We fit the scaling law, the logistic
function, and \eqref{eq:fid_map} on the \(36\) \((P,N)\) pairs at
\(L\in\{2,3,4\}\). In both regressions we use \(\hat I_{\max}(P)\), \(F(P)\), and
\(E(P,N)\) of the scaling law fitted on these pairs. The critical training-set
size at \(L=5\), and with it the range \(N\ge e^{2\sigma}N_{\mathrm c}\) in
which we predict FID, is therefore also a prediction. The predictions are in
Table~\ref{tab:imagenet_prediction}. Model information is predicted almost as
accurately as it is fitted when \(L=5\) is included, the memorization ratio is
predicted closely, and so is FID in the \(N^{-1}\) regime
(Table~\ref{tab:appendix_extrapolation}). Fitting \eqref{eq:fid_map} only in
the \(N^{-1}\) regime is essential: fitted on all training-set sizes, its
prediction is much worse, as FID is not monotonic while model information is (Table~\ref{tab:appendix_imagenet_statistics}).

We also trained four models at \(L=5\) on \(N=250{,}000\) images, outside the
range of every fit, with the same protocol. At this \(N\), \(E(P,N)\) is a few
nats, so the error of the prediction of model information is mostly an error
in \(F(P)\): with all model sizes in the fit, the fitted \(F(P)\) at \(L=5\)
exceeds the estimate at \(N=N_{\mathrm{full}}\)
(Table~\ref{tab:appendix_critical_sample_sizes}) by about as much as the
prediction exceeds the measurement (Table~\ref{tab:imagenet_prediction}).

\paragraph{All datasets.} We repeat this procedure on every dataset: we
withhold the largest model size from every fit, fit
\eqref{eq:linear_gaussian_information_law} and the two regressions on the
other model sizes, and predict model information, the memorization ratio, and
FID at the withheld size. The FID regression uses the \((P,N)\) pairs with
\(N\ge e^{2\sigma}N_{\mathrm c}(P)\), where \(\sigma\) and
\(N_{\mathrm c}(P)\) come from the fitted scaling law.
Table~\ref{tab:appendix_extrapolation} reports the results. The memorization
ratio is predicted accurately on all four datasets, and on CIFAR-10 the FID
predictions are as accurate as on ImageNet.

\begin{table}[htbp!]
    \centering
    \caption{\small \textbf{Predictions at depth \(L=5\) from fits to
        \(L\in\{2,3,4\}\) on ImageNet.} Each entry is predicted / measured.
        The fit to \(L\in\{2,3,4\}\) has \(N_{\mathrm c}=965\) at \(L=5\),
        against \(836\) in the fit to all model sizes, so we predict FID for
        \(N\ge e^{2\sigma}N_{\mathrm c}=3124\). The last two rows are the new
        training-set size \(N=250{,}000\approx299\,N_{\mathrm c}(P)\), outside
        the range of every fit, predicted from the fit to
        \(L\in\{2,3,4\}\) and from the fit to all model sizes.}
    \label{tab:imagenet_prediction}
    \small
    \begin{tabular}{lrrr}
        \toprule
        \(N\) & model information (nats) & memorization ratio & FID \\
        \midrule
        10 & \(64{,}800\) / \(74{,}700\) & 0.85 / 0.88 & --- \\
        400 & \(42{,}600\) / \(41{,}700\) & 0.22 / 0.17 & --- \\
        630 & \(31{,}700\) / \(20{,}100\) & 0.033 / 0.027 & --- \\
        1000 & \(19{,}200\) / \(14{,}200\) & 0.001 / 0.004 & --- \\
        4000 & \(1219\) / \(1098\) & 0 / 0 & 58.2 / 59.9 \\
        16{,}000 & \(357\) / \(395\) & 0 / 0 & 46.1 / 46.5 \\
        100{,}000 & \(247\) / \(240\) & 0 / 0 & 37.2 / 33.3 \\
        \midrule
        250{,}000, fit to \(L\in\{2,3,4\}\) & \(235\) / \(221\) & 0 / 0 &
        33.5 / 28.5 \\
        250{,}000, fit to all \(L\) & \(246.3\) / \(221.2\) & --- & --- \\
        \bottomrule
    \end{tabular}
\end{table}

\begin{table}[htbp!]
    \centering
    \caption{\small \textbf{Predictions at the withheld largest model size.}
        Model information: log RMSE of the prediction / log RMSE of the same
        points in the fit to all model sizes. Memorization: mean absolute error
        of the predicted memorization ratio on the withheld \((P,N)\) pairs / on
        the pairs used in the fit, and Pearson correlation between the
        predicted and measured ratio on the withheld pairs. FID: log RMSE
        on the withheld pairs with \(N\ge e^{2\sigma}N_{\mathrm
        c}(P)\) / on the
        pairs used in the fit, and the number of withheld pairs in
        this range. Fashion-MNIST and FFHQ have three model sizes, so their
        withheld fits use two. \({}^*\)On Fashion-MNIST the FID regression has
        three parameters and is fitted on three pairs, so it fits them exactly.
        The last row pools the \(37\) withheld pairs of the four
    datasets.}
    \label{tab:appendix_extrapolation}
    \small
    \begin{tabular}{lcccc}
        \toprule
        dataset & model information & memorization (MAE) &
        correlation & FID (pairs) \\
        \midrule
        ImageNet & 0.183 / 0.143 & 0.009 / 0.008 & 0.999 & 0.066 / 0.052 (5) \\
        CIFAR-10 & 0.189 / 0.123 & 0.022 / 0.009 & 0.991 & 0.120 / 0.098 (7) \\
        Fashion-MNIST & 0.278 / 0.078 & 0.039 / 0.033 & 0.989 & 0.569 / \(0^*\) (1) \\
        FFHQ & 0.400 / 0.110 & 0.026 / 0.017 & 0.997 & 0.075 / 0.066 (1) \\
        \midrule
        all four & --- & 0.021 & 0.993 & --- \\
        \bottomrule
    \end{tabular}
\end{table}

\begin{table}[htbp!]
    \centering
    \caption{\small \textbf{The fitted scaling law, memorization, and FID at
        each model size on ImageNet.} \(F(P)\) and the estimate at
        \(N=N_{\mathrm{full}}\) are in nats. ``Slope'' is the log-log slope of
        the measured model information minus its estimate at
        \(N=N_{\mathrm{full}}\), between \(N=4000\) and \(N=40{,}000\). The last two columns are the training-set sizes at
        which the memorization ratio is one half and at which FID has its local
    maximum, in units of \(N_{\mathrm c}(P)\).}
    \label{tab:imagenet_fit}
    \small
    \setlength{\tabcolsep}{4pt}
    \begin{tabular}{rrrrrrrr}
        \toprule
        \(L\) & \(P\) & \(N_{\mathrm c}(P)\) &
        fitted \(F(P)\) & estimate at \(N_{\mathrm{full}}\) & slope &
        mem.\ ratio \(\tfrac12\) & FID peak \\
        \midrule
        2 & \(6.7{\times}10^{6}\) & 350 & 362 & 330 &
        \(-0.96\) & 0.26 & 1.14 \\
        3 & \(2.0{\times}10^{7}\) & 511 & 279 & 266 &
        \(-0.91\) & 0.28 & 1.23 \\
        4 & \(4.4{\times}10^{7}\) & 673 & 251 & 226 &
        \(-0.94\) & 0.26 & 1.49 \\
        5 & \(8.3{\times}10^{7}\) & 836 & 239 & 218 &
        \(-1.02\) & 0.22 & 1.20 \\
        \bottomrule
    \end{tabular}
\end{table}

\subsection{CIFAR-10, Fashion-MNIST, and FFHQ}
\label{sub:appendix_cifar}
\label{sub:appendix_replication}

We repeat the measurement of the main text on three more datasets with the
same architectures, training recipe, and estimator
(Table~\ref{tab:appendix_grids}). CIFAR-10 \citep{krizhevsky2009learning}
covers five model sizes and fourteen training-set sizes, and its largest
training-set size is the full training set. Fashion-MNIST keeps CIFAR-10's
resolution and changes the content, and FFHQ quadruples the input dimension.

On every dataset we see the same behavior as on ImageNet
(Figures~\ref{fig:edm_scaling}, \ref{fig:edm_regimes},
and~\ref{fig:appendix_replication}). Model information is close to
\(\hat I_{\max}(P)\) for small \(N\), decreases around \(N_{\mathrm c}(P)\), and in
the \(N^{-1}\) regime decreases to \(F(P)\) as
\(N^{-1}-N_{\mathrm{full}}^{-1}\), with the finite-population correction of
Appendix~\ref{sub:linear_gaussian_to_law}. The critical training-set size increases
with the model size (Table~\ref{tab:appendix_critical_sample_sizes}). The
models memorize below \(N_{\mathrm c}(P)\): with ten training images, the
samples are repeated copies of the training images
(Figures~\ref{fig:edm_regimes} and~\ref{fig:appendix_replication_samples}).
The memorization ratio is zero in the \(N^{-1}\) regime, where FID decreases
with \(N\), and FID has a local maximum just after memorization ends, at
every model size on CIFAR-10 and FFHQ and at the largest model size on
Fashion-MNIST (Table~\ref{tab:appendix_memorization}).

\begin{figure}[htbp!]
    \centering
    \input{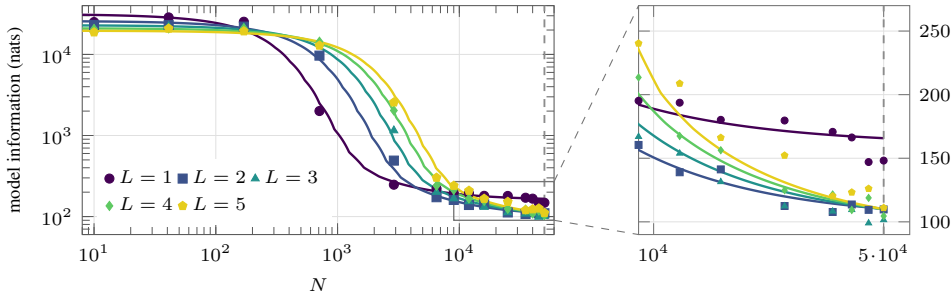}

\noindent\makebox[\linewidth][c]{%
    \begin{tikzpicture}
        \begin{axis}[
                name=main,
                width=0.56\linewidth,
                height=0.33\linewidth,
                xmode=log,
                ymode=log,
                xmin=8, xmax=6e4, ymin=60, ymax=4.0e4,
                xlabel={$N$},
                ylabel={model information (nats)},
                grid=major,
                grid style={line width=0.1pt, draw=gray!24},
                tick label style={font=\scriptsize},
                label style={font=\scriptsize},
                legend style={font=\scriptsize, draw=none, fill=white,
                at={(0.03,0.04)}, anchor=south west, legend columns=3},
                axis line style={black!65},
                tick style={black!65}
            ]
            \edmRawFit{figures/pgfplots/data/cifar10_lawfit_l1.csv}{edmDepthOne}
            \edmRawPoints{figures/pgfplots/data/cifar10_raw_l1.csv}{edmDepthOne}{*}{$L=1$}
            \edmRawFit{figures/pgfplots/data/cifar10_lawfit_l2.csv}{edmDepthTwo}
            \edmRawPoints{figures/pgfplots/data/cifar10_raw_l2.csv}{edmDepthTwo}{square*}{$L=2$}
            \edmRawFit{figures/pgfplots/data/cifar10_lawfit_l3.csv}{edmDepthThree}
            \edmRawPoints{figures/pgfplots/data/cifar10_raw_l3.csv}{edmDepthThree}{triangle*}{$L=3$}
            \edmRawFit{figures/pgfplots/data/cifar10_lawfit_l4.csv}{edmDepthFour}
            \edmRawPoints{figures/pgfplots/data/cifar10_raw_l4.csv}{edmDepthFour}{diamond*}{$L=4$}
            \edmRawFit{figures/pgfplots/data/cifar10_lawfit_l5.csv}{edmDepthFive}
            \edmRawPoints{figures/pgfplots/data/cifar10_raw_l5.csv}{edmDepthFive}{pentagon*}{$L=5$}
            \addplot[dashed, black!45, line width=0.7pt, forget plot]
            coordinates {(50000,60) (50000,4.0e4)};
            \draw[black!55, line width=0.5pt]
            (axis cs:9000,90) rectangle (axis cs:6e4,270);
            \coordinate (zoomtop) at (axis cs:6e4,270);
            \coordinate (zoombot) at (axis cs:6e4,90);
        \end{axis}
        \begin{axis}[
                name=zoom,
                at={(main.south east)}, anchor=south west,
                xshift=0.08\linewidth,
                width=0.37\linewidth,
                height=0.33\linewidth,
                xmode=log,
                xmin=9000, xmax=6e4, ymin=90, ymax=270,
                restrict y to domain*=90:270,
                title={},
                xtick={1e4,5e4}, xticklabels={$10^4$,$5\!\cdot\!10^4$},
                ytick={100,150,200,250}, yticklabel pos=right,
                grid=major,
                grid style={line width=0.1pt, draw=gray!24},
                tick label style={font=\scriptsize},
                axis line style={black!55},
                tick style={black!55}
            ]
            \edmRawFit{figures/pgfplots/data/cifar10_lawfit_l1.csv}{edmDepthOne}
            \edmZoomPoints{figures/pgfplots/data/cifar10_raw_l1.csv}{edmDepthOne}{*}
            \edmRawFit{figures/pgfplots/data/cifar10_lawfit_l2.csv}{edmDepthTwo}
            \edmZoomPoints{figures/pgfplots/data/cifar10_raw_l2.csv}{edmDepthTwo}{square*}
            \edmRawFit{figures/pgfplots/data/cifar10_lawfit_l3.csv}{edmDepthThree}
            \edmZoomPoints{figures/pgfplots/data/cifar10_raw_l3.csv}{edmDepthThree}{triangle*}
            \edmRawFit{figures/pgfplots/data/cifar10_lawfit_l4.csv}{edmDepthFour}
            \edmZoomPoints{figures/pgfplots/data/cifar10_raw_l4.csv}{edmDepthFour}{diamond*}
            \edmRawFit{figures/pgfplots/data/cifar10_lawfit_l5.csv}{edmDepthFive}
            \edmZoomPoints{figures/pgfplots/data/cifar10_raw_l5.csv}{edmDepthFive}{pentagon*}
            \addplot[dashed, black!45, line width=0.7pt, forget plot]
            coordinates {(50000,90) (50000,270)};
        \end{axis}
        \draw[black!55, line width=0.4pt, dashed] (zoomtop) --
        (zoom.north west);
        \draw[black!55, line width=0.4pt, dashed] (zoombot) --
        (zoom.south west);
    \end{tikzpicture}%
}
    \vspace{-1.5em}
    \caption{\small \textbf{Model information scaling on CIFAR-10.}
        Model information as a function of training-set size for five
        model sizes, together with the fit of
        \eqref{eq:linear_gaussian_information_law}. The largest dataset
        size is the full dataset, $N_{\max}=N_{\mathrm{full}}=50{,}000$, where
        all models share the same training data and the remaining estimate
        is due to the variation between training seeds and the bias of the
        estimator; these points are not part of the fit. The right panel magnifies the boxed region,
        where $E(P,N)$ decays to zero as
    $N^{-1}-N_{\mathrm{full}}^{-1}$.}
    \label{fig:edm_scaling}
\end{figure}

\begin{figure}[htbp!]
    \centering
    \input{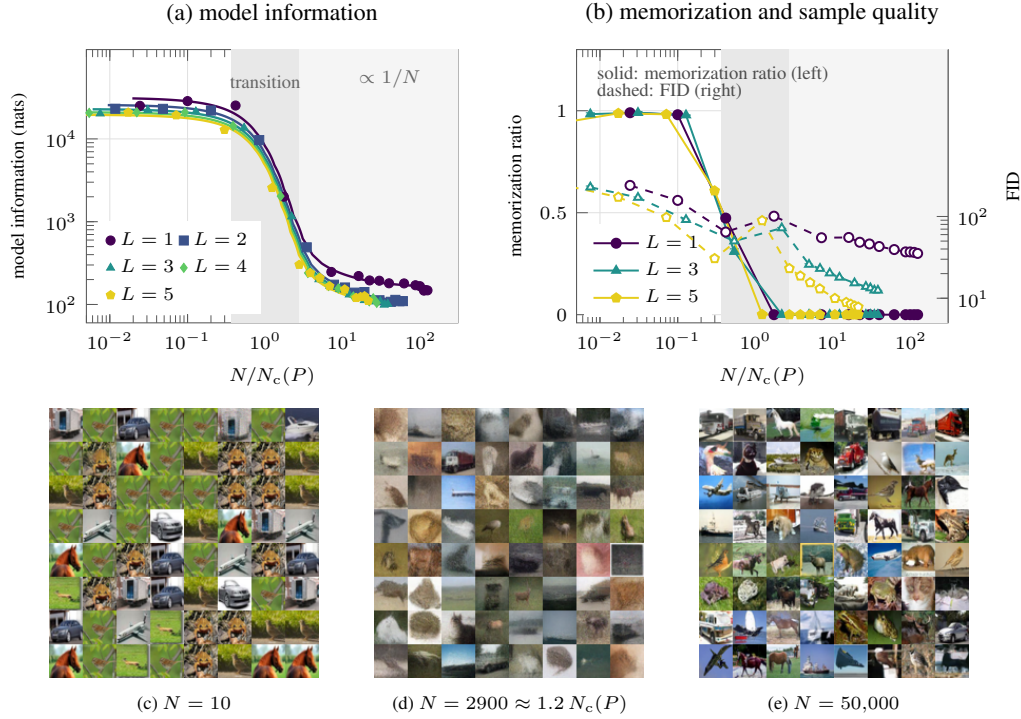}
    \caption{\small \textbf{The two regimes on CIFAR-10.} Layout and shading
        follow Figure~\ref{fig:imagenet_regimes}. The sample size is
        normalized by the critical training-set size $N_{\mathrm c}(P)$; the
        dark band is
        $N/N_{\mathrm c}(P)\in[e^{-2\sigma},e^{2\sigma}]$, and the
        lighter band the
        $N^{-1}$ regime. Below, we show random samples from
    models at depth $L=5$ at three training set sizes $N$.}
    \label{fig:edm_regimes}
\end{figure}

\begin{figure}[htbp!]
    \centering
    \input{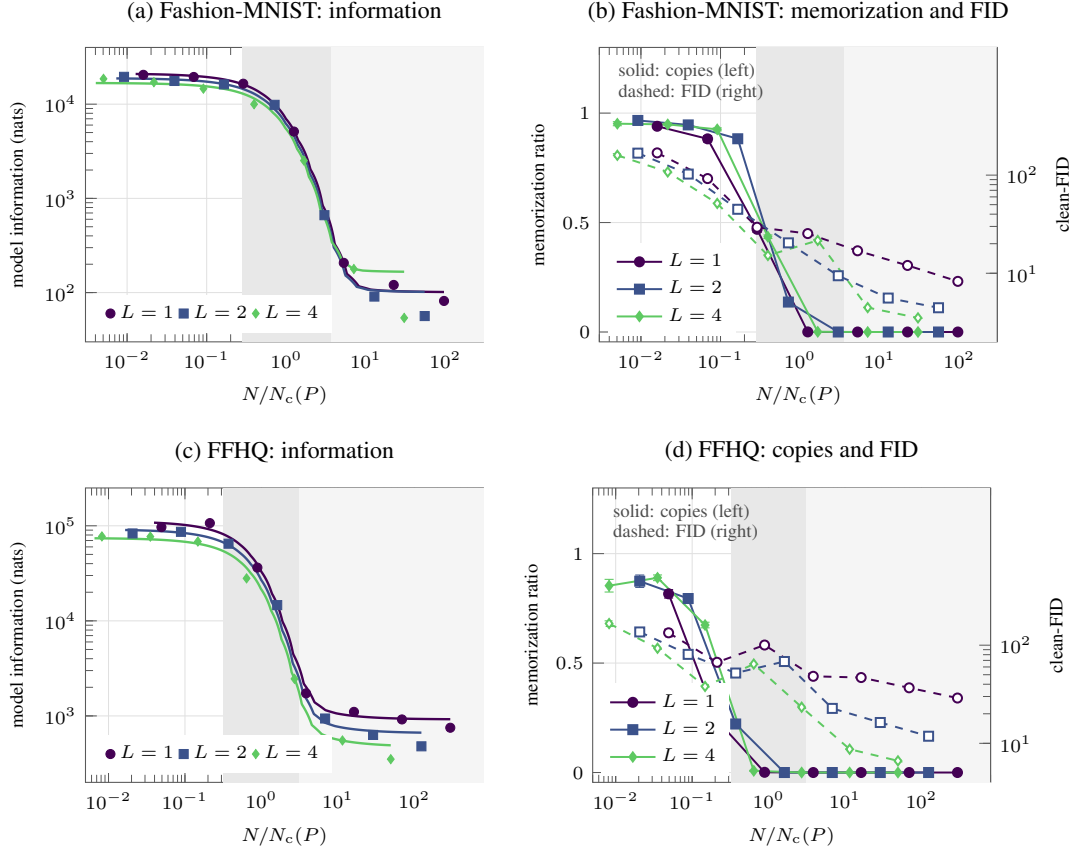}

\begin{minipage}[t]{0.49\linewidth}
    \centering
    \begin{tikzpicture}
        \begin{axis}[
                width=\linewidth, height=0.80\linewidth,
                xmode=log, ymode=log,
                xmin=0.003, xmax=300, ymin=30, ymax=4e4,
                xlabel={$N/N_{\mathrm c}(P)$},
                ylabel={model information (nats)},
                title={(a) Fashion-MNIST: information},
                ytick={1e2,1e3,1e4},
                grid=major, grid style={line width=0.1pt, draw=gray!24},
                tick label style={font=\scriptsize}, label
                style={font=\scriptsize},
                title style={font=\small},
                legend style={font=\scriptsize, draw=none, fill=white,
                at={(0.03,0.03)}, anchor=south west, legend columns=3},
                axis line style={black!65}, tick style={black!65}
            ]
            \edmBandsAt{30}{4e4}{0.277}{3.61}{300}
            \edmLawFit{figures/pgfplots/data/fashion_mnist_lawfit_l1.csv}{edmDepthOne}
            \edmInfoPoints{figures/pgfplots/data/fashion_mnist_raw_l1.csv}{edmDepthOne}{*}{$L=1$}
            \edmLawFit{figures/pgfplots/data/fashion_mnist_lawfit_l2.csv}{edmDepthTwo}
            \edmInfoPoints{figures/pgfplots/data/fashion_mnist_raw_l2.csv}{edmDepthTwo}{square*}{$L=2$}
            \edmLawFit{figures/pgfplots/data/fashion_mnist_lawfit_l4.csv}{edmDepthFour}
            \edmInfoPoints{figures/pgfplots/data/fashion_mnist_raw_l4.csv}{edmDepthFour}{diamond*}{$L=4$}
        \end{axis}
    \end{tikzpicture}
\end{minipage}\hfill
\begin{minipage}[t]{0.49\linewidth}
    \centering
    \begin{tikzpicture}
        \begin{axis}[
                name=fmcopies,
                width=\linewidth, height=0.80\linewidth,
                xmode=log,
                xmin=0.003, xmax=300, ymin=-0.045, ymax=1.30,
                xlabel={$N/N_{\mathrm c}(P)$},
                ylabel={memorization ratio},
                title={(b) Fashion-MNIST: memorization and FID},
                ytick={0,0.5,1},
                grid=major, grid style={line width=0.1pt, draw=gray!24},
                tick label style={font=\scriptsize}, label
                style={font=\scriptsize},
                title style={font=\small},
                legend style={font=\scriptsize, draw=none, fill=white,
                at={(0.02,0.02)}, anchor=south west, legend columns=1},
                axis line style={black!65}, tick style={black!65}
            ]
            \edmBandsAt{-0.045}{1.30}{0.277}{3.61}{300}
            \edmMetricSeries{figures/pgfplots/data/fashion_mnist_memorization_pixel_l1.csv}{edmDepthOne}{*}{$L=1$}
            \edmMetricSeries{figures/pgfplots/data/fashion_mnist_memorization_pixel_l2.csv}{edmDepthTwo}{square*}{$L=2$}
            \edmMetricSeries{figures/pgfplots/data/fashion_mnist_memorization_pixel_l4.csv}{edmDepthFour}{diamond*}{$L=4$}
        \end{axis}
        \begin{axis}[
                at={(fmcopies.south west)}, anchor=south west,
                width=\linewidth, height=0.80\linewidth,
                xmode=log, ymode=log,
                xmin=0.003, xmax=300, ymin=2, ymax=2e3,
                axis x line=none, axis y line*=right,
                ytick={1e1,1e2},
                ylabel={clean-FID},
                tick label style={font=\scriptsize}, label
                style={font=\scriptsize},
                grid=none,
                axis line style={black!65}, tick style={black!65}
            ]
            \edmDashedSeriesNoLegend{figures/pgfplots/data/fashion_mnist_fid_l1.csv}{edmDepthOne}{*}
            \edmDashedSeriesNoLegend{figures/pgfplots/data/fashion_mnist_fid_l2.csv}{edmDepthTwo}{square*}
            \edmDashedSeriesNoLegend{figures/pgfplots/data/fashion_mnist_fid_l4.csv}{edmDepthFour}{diamond*}
            \node[font=\scriptsize, black!70, anchor=north west] at
            (axis cs:0.004,1.7e3) {solid: copies (left)};
            \node[font=\scriptsize, black!70, anchor=north west] at
            (axis cs:0.004,1.05e3) {dashed: FID (right)};
        \end{axis}
    \end{tikzpicture}
\end{minipage}

\vspace{6pt}

\begin{minipage}[t]{0.49\linewidth}
    \centering
    \begin{tikzpicture}
        \begin{axis}[
                width=\linewidth, height=0.80\linewidth,
                xmode=log, ymode=log,
                xmin=0.005, xmax=800, ymin=200, ymax=2.5e5,
                xlabel={$N/N_{\mathrm c}(P)$},
                ylabel={model information (nats)},
                title={(c) FFHQ: information},
                ytick={1e3,1e4,1e5},
                grid=major, grid style={line width=0.1pt, draw=gray!24},
                tick label style={font=\scriptsize}, label
                style={font=\scriptsize},
                title style={font=\small},
                legend style={font=\scriptsize, draw=none, fill=white,
                at={(0.03,0.03)}, anchor=south west, legend columns=3},
                axis line style={black!65}, tick style={black!65}
            ]
            \edmBandsAt{200}{2.5e5}{0.319}{3.13}{800}
            \edmLawFit{figures/pgfplots/data/ffhq_lawfit_l1.csv}{edmDepthOne}
            \edmInfoPoints{figures/pgfplots/data/ffhq_raw_l1.csv}{edmDepthOne}{*}{$L=1$}
            \edmLawFit{figures/pgfplots/data/ffhq_lawfit_l2.csv}{edmDepthTwo}
            \edmInfoPoints{figures/pgfplots/data/ffhq_raw_l2.csv}{edmDepthTwo}{square*}{$L=2$}
            \edmLawFit{figures/pgfplots/data/ffhq_lawfit_l4.csv}{edmDepthFour}
            \edmInfoPoints{figures/pgfplots/data/ffhq_raw_l4.csv}{edmDepthFour}{diamond*}{$L=4$}
        \end{axis}
    \end{tikzpicture}
\end{minipage}\hfill
\begin{minipage}[t]{0.49\linewidth}
    \centering
    \begin{tikzpicture}
        \begin{axis}[
                name=ffcopies,
                width=\linewidth, height=0.80\linewidth,
                xmode=log,
                xmin=0.005, xmax=800, ymin=-0.045, ymax=1.30,
                xlabel={$N/N_{\mathrm c}(P)$},
                ylabel={memorization ratio},
                title={(d) FFHQ: copies and FID},
                ytick={0,0.5,1},
                grid=major, grid style={line width=0.1pt, draw=gray!24},
                tick label style={font=\scriptsize}, label
                style={font=\scriptsize},
                title style={font=\small},
                legend style={font=\scriptsize, draw=none, fill=white,
                at={(0.02,0.02)}, anchor=south west, legend columns=1},
                axis line style={black!65}, tick style={black!65}
            ]
            \edmBandsAt{-0.045}{1.30}{0.319}{3.13}{800}
            \edmMetricSeries{figures/pgfplots/data/ffhq_memorization_pixel_l1.csv}{edmDepthOne}{*}{$L=1$}
            \edmMetricSeries{figures/pgfplots/data/ffhq_memorization_pixel_l2.csv}{edmDepthTwo}{square*}{$L=2$}
            \edmMetricSeries{figures/pgfplots/data/ffhq_memorization_pixel_l4.csv}{edmDepthFour}{diamond*}{$L=4$}
        \end{axis}
        \begin{axis}[
                at={(ffcopies.south west)}, anchor=south west,
                width=\linewidth, height=0.80\linewidth,
                xmode=log, ymode=log,
                xmin=0.005, xmax=800, ymin=4, ymax=4e3,
                axis x line=none, axis y line*=right,
                ytick={1e1,1e2},
                ylabel={clean-FID},
                tick label style={font=\scriptsize}, label
                style={font=\scriptsize},
                grid=none,
                axis line style={black!65}, tick style={black!65}
            ]
            \edmDashedSeriesNoLegend{figures/pgfplots/data/ffhq_fid_l1.csv}{edmDepthOne}{*}
            \edmDashedSeriesNoLegend{figures/pgfplots/data/ffhq_fid_l2.csv}{edmDepthTwo}{square*}
            \edmDashedSeriesNoLegend{figures/pgfplots/data/ffhq_fid_l4.csv}{edmDepthFour}{diamond*}
            \node[font=\scriptsize, black!70, anchor=north west] at
            (axis cs:0.007,3.4e3) {solid: copies (left)};
            \node[font=\scriptsize, black!70, anchor=north west] at
            (axis cs:0.007,2.1e3) {dashed: FID (right)};
        \end{axis}
    \end{tikzpicture}
\end{minipage}
    \caption{\small \textbf{The two regimes on Fashion-MNIST and FFHQ.}
        Layout and shading follow Figure~\ref{fig:edm_regimes}: the darker
        band is $N/N_{\mathrm c}(P)\in[e^{-2\sigma},e^{2\sigma}]$,
        where most groups cross
        $N=d_z$, and the lighter band beyond it is the $N^{-1}$ regime.
        \textbf{(a,c)} model information for three model sizes with the fitted
        scaling law. \textbf{(b,d)} the memorization ratio (solid, left axis) and
        clean-FID (dashed, right axis), with standard errors across the
        trained models at each $(P,N)$. On both datasets the
        results are as in the main text: memorization ends near
        $N=N_{\mathrm c}(P)$,
        and FID falls
        monotonically through the $N^{-1}$ regime. On FFHQ, FID increases by a
        factor of \(1.3\) to \(1.7\) between \(N=43\) and \(N=180\) at \(L=1\),
        and between \(N=180\) and \(N=780\) at \(L=2\) and \(L=4\). On
        Fashion-MNIST at \(L=4\) it increases from \(15.2\) at \(N=780\)
    to \(21.6\) at \(N=3300\).}
    \label{fig:appendix_replication}
\end{figure}

\begin{figure}[htbp!]
    \centering
    \begin{minipage}[t]{0.30\linewidth}
        \centering
        \includegraphics[width=\linewidth]{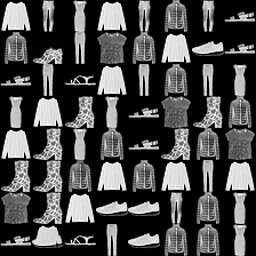}
        \\[1pt] {\scriptsize (a) $N=10$}
    \end{minipage}\hspace{0.02\linewidth}
    \begin{minipage}[t]{0.30\linewidth}
        \centering
        \includegraphics[width=\linewidth]{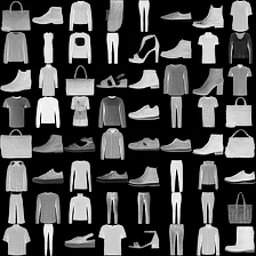}
        \\[1pt] {\scriptsize (b) $N=3300\approx1.7\,N_{\mathrm c}(P)$}
    \end{minipage}\hspace{0.02\linewidth}
    \begin{minipage}[t]{0.30\linewidth}
        \centering
        \includegraphics[width=\linewidth]{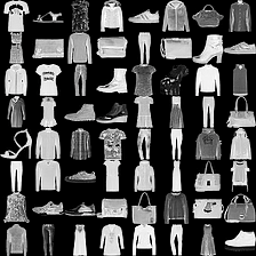}
        \\[1pt] {\scriptsize (c) $N=60{,}000$}
    \end{minipage}

    \vspace{5pt}

    \begin{minipage}[t]{0.30\linewidth}
        \centering
        \includegraphics[width=\linewidth]{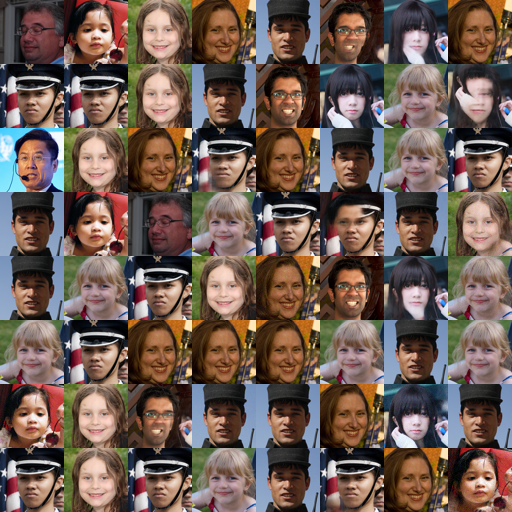}
        \\[1pt] {\scriptsize (d) $N=10$}
    \end{minipage}\hspace{0.02\linewidth}
    \begin{minipage}[t]{0.30\linewidth}
        \centering
        \includegraphics[width=\linewidth]{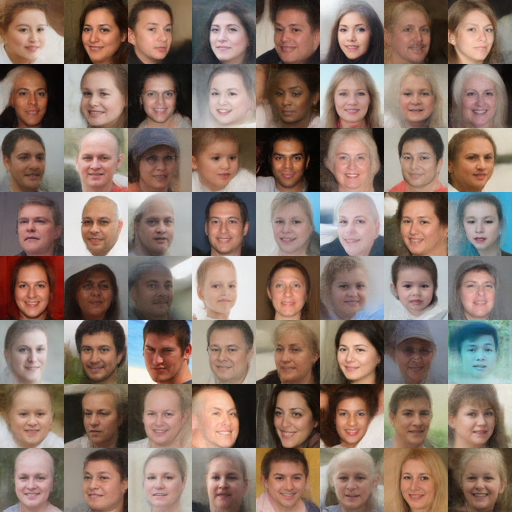}
        \\[1pt] {\scriptsize (e) $N=780\approx0.64\,N_{\mathrm c}(P)$}
    \end{minipage}\hspace{0.02\linewidth}
    \begin{minipage}[t]{0.30\linewidth}
        \centering
        \includegraphics[width=\linewidth]{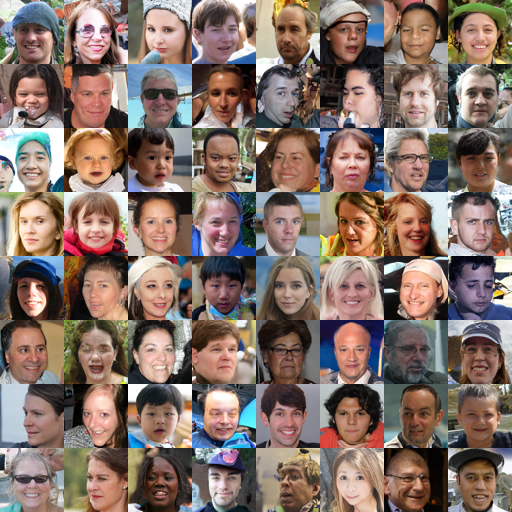}
        \\[1pt] {\scriptsize (f) $N=60{,}000$}
    \end{minipage}
    \caption{\small \textbf{Samples across the transition on Fashion-MNIST
        (top) and FFHQ (bottom).} All six panels come from models at $L=4$.
        \textbf{(a,d)} ten training images, repeated verbatim.
        \textbf{(b,e)} near $N=N_{\mathrm c}(P)$, where the
        memorization ratio has just dropped to nearly zero and FID has its local
        maximum: the images are new, but the least sharp of the three. \textbf{(c,f)} the full dataset, where no copies
    are detected and FID is $3.5$ and $6.6$ respectively.}
    \label{fig:appendix_replication_samples}
\end{figure}

\subsection{Sensitivity of the Model Information to the Number of
Trained Models}
\label{sub:appendix_bank_sensitivity}

The number of reference models $K_{\mathrm{train}}$ and the number of
held-out models $K_{\mathrm{val}}$ are parameters of the estimator and not of
model information, so the fitted scaling law should not depend on them. We
check this in two ways on CIFAR-10, Fashion-MNIST, and FFHQ. First, we sweep
$K_{\mathrm{train}}$ over $\{2,3,4,5\}$ on Fashion-MNIST and over
$\{1,2\}$ on CIFAR-10 and FFHQ, with $K_{\mathrm{val}}$ fixed at five on
Fashion-MNIST and two on the other two datasets. Second, since we train
$K_{\mathrm{train}}+K_{\mathrm{val}}$ models per $(P,N)$ pair, we sweep the two
numbers together on Fashion-MNIST, as they would change with the number of
trained models. We refit \eqref{eq:linear_gaussian_information_law}
independently at each value. Unlike the fits in the rest of this appendix,
these refits include the estimates at $N=N_{\mathrm{full}}$ as measurements of
$F(P)$. This keeps $F(P)$ fixed across values of $K_{\mathrm{train}}$, so that
the comparison isolates the effect of $K_{\mathrm{train}}$ on the other
parameters. On ImageNet we compare $K_{\mathrm{train}}=1$ with
$K_{\mathrm{train}}=2$ on the same two held-out models, without the estimates
at $N=N_{\mathrm{full}}$.

Fewer reference models give slightly larger estimates
(Figure~\ref{fig:appendix_ktrain}). This is the sign that
\eqref{eq:appendix_integrated_mixture_bias} predicts, since the mixture-score
error is nonnegative and decreases as $K_{\mathrm{train}}$ grows, and it means
the main-text numbers at $K_{\mathrm{train}}=2$ are upper bounds on the
estimates with more reference models. The fitted parameters change much less
than the estimates do (Table~\ref{tab:appendix_bank_sensitivity}). To first
order, fewer reference models multiply the estimates by a nearly constant
factor, which changes the prefactor $k$ and not the exponents.

The total number of trained models has the same effect
(Figure~\ref{fig:appendix_budget} and
Table~\ref{tab:appendix_budget_sensitivity}). This matters because the
datasets differ in it: on Fashion-MNIST we train ten models per $(P,N)$ pair,
and on the other datasets four. We rerun the Fashion-MNIST estimates with
$K_{\mathrm{train}}=K_{\mathrm{val}}=k$ for $k$ from two to five, reusing the
trained models so that only the estimator changes. The ten-model row is the
configuration of the other Fashion-MNIST experiments and reproduces them
exactly. With fewer models the estimates are larger on average, by an amount
that decreases as the number of models grows, which is again the sign that
\eqref{eq:appendix_integrated_mixture_bias} predicts. Individual ratios vary
more than in the $K_{\mathrm{train}}$ sweep, because each number of models
uses its own held-out evaluation points instead of a shared set. With four
models per $(P,N)$ pair, model information decreases at the same $N$ as with
ten.

In the same sweep, the three estimators of
Appendix~\ref{sec:model_information_computations} differ
(Table~\ref{tab:appendix_bank_sensitivity}). The fits to the two classifier
estimators are stable at every $K_{\mathrm{train}}$ and qualitatively the
same. The fits to the distillation estimator are not: at small
$K_{\mathrm{train}}$ the fit has $\sigma\approx0$ and $\ell$ at its upper
bound, so the fitted curve differs from those of the classifier estimators,
and only at the largest $K_{\mathrm{train}}$ we ran is it close to them. Many
reference models are necessary to distill a mixture of $K$ models into one
network, and the classifier estimators do not form this mixture explicitly.
For this reason we use the model-free classifier estimator throughout the main
text.

\begin{figure}[htbp!]
    \centering
    \input{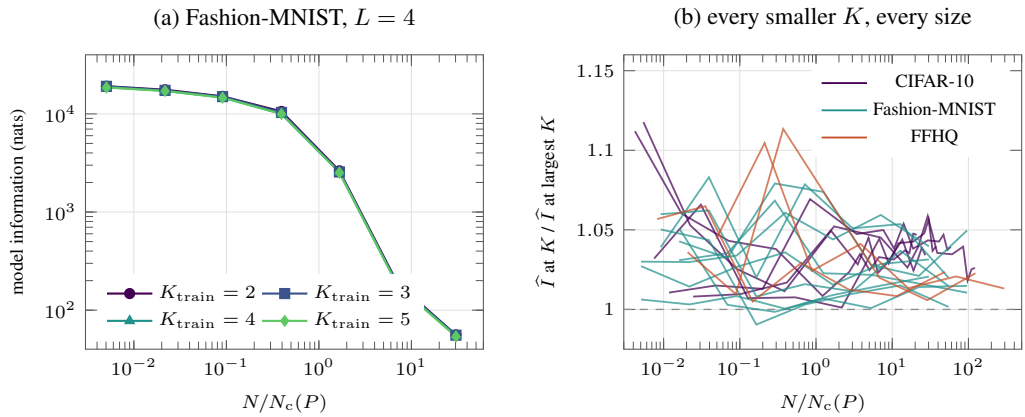}

\begin{minipage}[t]{0.49\linewidth}
\centering
\begin{tikzpicture}
\begin{axis}[
  width=\linewidth, height=0.80\linewidth,
  xmode=log, ymode=log,
  xmin=0.003, xmax=60, ymin=40, ymax=4e4,
  xlabel={$N/N_{\mathrm c}(P)$},
  ylabel={model information (nats)},
  title={(a) Fashion-MNIST, $L=4$},
  ytick={1e2,1e3,1e4},
  grid=major, grid style={line width=0.1pt, draw=gray!24},
  tick label style={font=\scriptsize}, label style={font=\scriptsize},
  title style={font=\small},
  legend style={font=\scriptsize, draw=none, fill=white,
    at={(0.03,0.03)}, anchor=south west, legend columns=2},
  axis line style={black!65}, tick style={black!65}
]
\edmInfoSeries{figures/pgfplots/data/fashion_mnist_information_l4_k2.csv}{edmDepthOne}{*}{$K_{\mathrm{train}}=2$}
\edmInfoSeries{figures/pgfplots/data/fashion_mnist_information_l4_k3.csv}{edmDepthTwo}{square*}{$K_{\mathrm{train}}=3$}
\edmInfoSeries{figures/pgfplots/data/fashion_mnist_information_l4_k4.csv}{edmDepthThree}{triangle*}{$K_{\mathrm{train}}=4$}
\edmInfoSeries{figures/pgfplots/data/fashion_mnist_information_l4_k5.csv}{edmDepthFour}{diamond*}{$K_{\mathrm{train}}=5$}
\end{axis}
\end{tikzpicture}
\end{minipage}\hfill
\begin{minipage}[t]{0.49\linewidth}
\centering
\begin{tikzpicture}
\begin{axis}[
  width=\linewidth, height=0.80\linewidth,
  xmode=log,
  xmin=0.003, xmax=500, ymin=0.975, ymax=1.16,
  xlabel={$N/N_{\mathrm c}(P)$},
  ylabel={$\widehat I$ at $K$ / $\widehat I$ at largest $K$},
  title={(b) every smaller $K$, every size},
  ytick={1.00,1.05,1.10,1.15},
  grid=major, grid style={line width=0.1pt, draw=gray!24},
  tick label style={font=\scriptsize}, label style={font=\scriptsize},
  title style={font=\small},
  legend style={font=\scriptsize, draw=none, fill=white,
    at={(0.97,0.97)}, anchor=north east, legend columns=1},
  axis line style={black!65}, tick style={black!65}
]
\addplot[black!45, line width=0.5pt, dashed, no marks, forget plot]
  coordinates {(0.003,1.0) (500,1.0)};
\edmRatio{figures/pgfplots/data/cifar10_kratio_l1_k1.csv}{edmDepthOne}
\edmRatio{figures/pgfplots/data/cifar10_kratio_l2_k1.csv}{edmDepthOne}
\edmRatio{figures/pgfplots/data/cifar10_kratio_l3_k1.csv}{edmDepthOne}
\edmRatio{figures/pgfplots/data/cifar10_kratio_l4_k1.csv}{edmDepthOne}
\edmRatio{figures/pgfplots/data/cifar10_kratio_l5_k1.csv}{edmDepthOne}
\edmRatio{figures/pgfplots/data/fashion_mnist_kratio_l1_k2.csv}{edmDepthThree}
\edmRatio{figures/pgfplots/data/fashion_mnist_kratio_l1_k3.csv}{edmDepthThree}
\edmRatio{figures/pgfplots/data/fashion_mnist_kratio_l1_k4.csv}{edmDepthThree}
\edmRatio{figures/pgfplots/data/fashion_mnist_kratio_l2_k2.csv}{edmDepthThree}
\edmRatio{figures/pgfplots/data/fashion_mnist_kratio_l2_k3.csv}{edmDepthThree}
\edmRatio{figures/pgfplots/data/fashion_mnist_kratio_l2_k4.csv}{edmDepthThree}
\edmRatio{figures/pgfplots/data/fashion_mnist_kratio_l4_k2.csv}{edmDepthThree}
\edmRatio{figures/pgfplots/data/fashion_mnist_kratio_l4_k3.csv}{edmDepthThree}
\edmRatio{figures/pgfplots/data/fashion_mnist_kratio_l4_k4.csv}{edmDepthThree}
\edmRatio{figures/pgfplots/data/ffhq_kratio_l1_k1.csv}{edmSetWarm}
\edmRatio{figures/pgfplots/data/ffhq_kratio_l2_k1.csv}{edmSetWarm}
\edmRatio{figures/pgfplots/data/ffhq_kratio_l4_k1.csv}{edmSetWarm}
\addlegendimage{color=edmDepthOne, line width=0.7pt}\addlegendentry{CIFAR-10}
\addlegendimage{color=edmDepthThree, line width=0.7pt}\addlegendentry{Fashion-MNIST}
\addlegendimage{color=edmSetWarm, line width=0.7pt}\addlegendentry{FFHQ}
\end{axis}
\end{tikzpicture}
\end{minipage}
    \caption{\small \textbf{The number of reference models changes the
        estimates by a nearly constant factor.} \textbf{(a)}
        Fashion-MNIST at $L=4$ with $K_{\mathrm{train}}$ from two to five;
        the four curves lie on top of one another.
        \textbf{(b)} every smaller $K_{\mathrm{train}}$ on every dataset and
        model size, as a ratio to the estimate at the largest
        $K_{\mathrm{train}}$ for that dataset: seventeen curves, all within
        $[0.99,1.12]$ and almost all above one. With fewer reference models
        the estimates are larger, which is the sign of the error term
    \eqref{eq:appendix_integrated_mixture_bias}.}
    \label{fig:appendix_ktrain}
\end{figure}

\begin{table}[htbp!]
    \centering
    \caption{\small \textbf{Refitting the scaling law at every
        \(K_{\mathrm{train}}\).} We fit
        \eqref{eq:linear_gaussian_information_law} independently at each
        $K_{\mathrm{train}}$, on the same $(P,N)$ pairs. The column
        $\hat I_{\max}-F$ is $\hat I_{\max}(P)-F(P)$ at $P=P_{\mathrm{ref}}$. Across
        $K_{\mathrm{train}}$, $N_{\mathrm c}(P)$ changes by at most $2.1\%$ on
        Fashion-MNIST, $2.8\%$ on CIFAR-10, $0.5\%$ on FFHQ, and $1.3\%$ on
        ImageNet. On ImageNet the fit excludes the estimates at
        $N=N_{\mathrm{full}}$; with $K_{\mathrm{train}}=1$ instead of $2$, the
        median ratio of the estimates is $1.016$, $85\%$ of the ratios are at
        most $1.05$, and the largest is $1.19$, at $N=10$ and $L=2$. The fits
        to the two classifier estimators are stable in $K_{\mathrm{train}}$;
        the fit to the distillation estimator is close to them only at
        the largest $K_{\mathrm{train}}$.}
    \label{tab:appendix_bank_sensitivity}
    \small
    \resizebox{\linewidth}{!}{%
        \begin{tabular}{llrrrrrrr}
            \toprule
            dataset & estimator & \(K_{\mathrm{train}}\) &
            \(a\) & \(\sigma\) & \(\ell\) & \(b\) & \(\hat I_{\max}-F\) & log RMSE \\
            \midrule
            ImageNet & model-free & 1 & 0.348 & 0.619 & --- & --- & --- & --- \\
            ImageNet & model-free & 2 & 0.345 & 0.622 & 13.75 & 0.483 &
            \(7.556{\times}10^4\) & 0.112 \\
            \midrule
            CIFAR-10 & model-free & 1 & 0.402 & 0.502 & 15.85 & 0.502
            & \(2.388{\times}10^4\) & 0.1015 \\
            CIFAR-10 & model-free & 2 & 0.409 & 0.500 & 16.20 & 0.519
            & \(2.282{\times}10^4\) & 0.1009 \\
            \midrule
            FFHQ & model-free & 1 & 0.508 & 0.503 & 15.39 & 0.631 &
            \(9.576{\times}10^4\) & 0.1304 \\
            FFHQ & model-free & 2 & 0.509 & 0.505 & 15.10 & 0.625 &
            \(9.108{\times}10^4\) & 0.1205 \\
            \midrule
            Fashion-MNIST & model-free & 2 & 0.357 & 0.506 & 11.66 &
            0.421 & \(1.958{\times}10^4\) & 0.1059 \\
            Fashion-MNIST & model-free & 3 & 0.361 & 0.507 & 11.49 &
            0.428 & \(1.941{\times}10^4\) & 0.1025 \\
            Fashion-MNIST & model-free & 4 & 0.363 & 0.516 & 11.64 &
            0.430 & \(1.903{\times}10^4\) & 0.1057 \\
            Fashion-MNIST & model-free & 5 & 0.358 & 0.513 & 11.64 &
            0.420 & \(1.872{\times}10^4\) & 0.1026 \\
            \midrule
            Fashion-MNIST & model-based & 2 & 0.371 & 0.532 & 12.56 &
            0.428 & \(1.975{\times}10^4\) & 0.1164 \\
            Fashion-MNIST & model-based & 3 & 0.376 & 0.534 & 12.49 &
            0.440 & \(1.951{\times}10^4\) & 0.1156 \\
            Fashion-MNIST & model-based & 4 & 0.380 & 0.548 & 12.69 &
            0.447 & \(1.924{\times}10^4\) & 0.1159 \\
            Fashion-MNIST & model-based & 5 & 0.377 & 0.543 & 12.57 &
            0.440 & \(1.890{\times}10^4\) & 0.1189 \\
            \midrule
            Fashion-MNIST & distilled & 2 & 0.090 & 0.000 & 148.41 &
            0.082 & \(0.522{\times}10^4\) & 0.5536 \\
            Fashion-MNIST & distilled & 3 & 0.047 & 0.001 & 148.41 &
            0.017 & \(0.514{\times}10^4\) & 0.5955 \\
            Fashion-MNIST & distilled & 4 & 0.205 & 0.708 & 147.95 &
            0.215 & \(0.512{\times}10^4\) & 0.4951 \\
            Fashion-MNIST & distilled & 5 & 0.299 & 0.765 & 147.36 &
            0.328 & \(0.491{\times}10^4\) & 0.3405 \\
            \bottomrule
        \end{tabular}
    }
\end{table}

\begin{figure}[htbp!]
    \centering
    \input{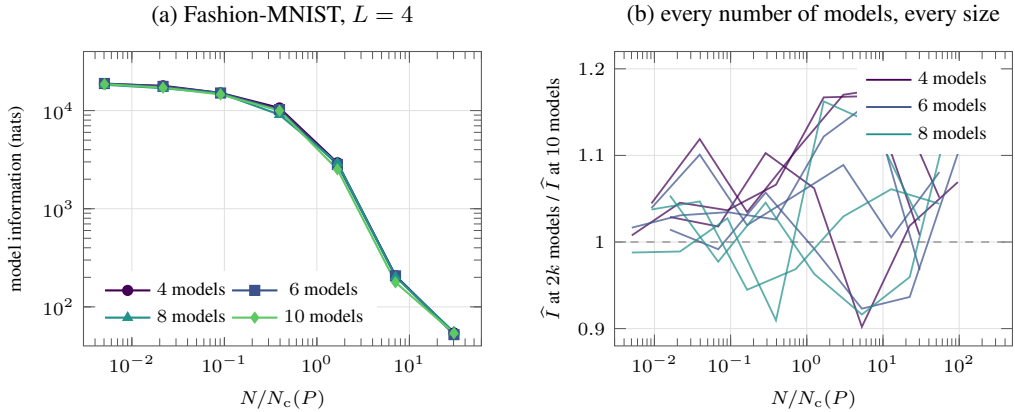}

\begin{minipage}[t]{0.49\linewidth}
\centering
\begin{tikzpicture}
\begin{axis}[
  width=\linewidth, height=0.80\linewidth,
  xmode=log, ymode=log,
  xmin=0.003, xmax=60, ymin=40, ymax=4e4,
  xlabel={$N/N_{\mathrm c}(P)$},
  ylabel={model information (nats)},
  title={(a) Fashion-MNIST, $L=4$},
  ytick={1e2,1e3,1e4},
  grid=major, grid style={line width=0.1pt, draw=gray!24},
  tick label style={font=\scriptsize}, label style={font=\scriptsize},
  title style={font=\small},
  legend style={font=\scriptsize, draw=none, fill=white,
    at={(0.03,0.03)}, anchor=south west, legend columns=2},
  axis line style={black!65}, tick style={black!65}
]
\edmInfoSeries{figures/pgfplots/data/fashion_mnist_budget_l4_m4.csv}{edmDepthOne}{*}{$4$ models}
\edmInfoSeries{figures/pgfplots/data/fashion_mnist_budget_l4_m6.csv}{edmDepthTwo}{square*}{$6$ models}
\edmInfoSeries{figures/pgfplots/data/fashion_mnist_budget_l4_m8.csv}{edmDepthThree}{triangle*}{$8$ models}
\edmInfoSeries{figures/pgfplots/data/fashion_mnist_budget_l4_m10.csv}{edmDepthFour}{diamond*}{$10$ models}
\end{axis}
\end{tikzpicture}
\end{minipage}\hfill
\begin{minipage}[t]{0.49\linewidth}
\centering
\begin{tikzpicture}
\begin{axis}[
  width=\linewidth, height=0.80\linewidth,
  xmode=log,
  xmin=0.003, xmax=500, ymin=0.88, ymax=1.22,
  xlabel={$N/N_{\mathrm c}(P)$},
  ylabel={$\widehat I$ at $2k$ models / $\widehat I$ at $10$ models},
  title={(b) every number of models, every size},
  ytick={0.9,1.0,1.1,1.2},
  grid=major, grid style={line width=0.1pt, draw=gray!24},
  tick label style={font=\scriptsize}, label style={font=\scriptsize},
  title style={font=\small},
  legend style={font=\scriptsize, draw=none, fill=white,
    at={(0.97,0.97)}, anchor=north east, legend columns=1},
  axis line style={black!65}, tick style={black!65}
]
\addplot[black!45, line width=0.5pt, dashed, no marks, forget plot]
  coordinates {(0.003,1.0) (500,1.0)};
\edmRatio{figures/pgfplots/data/fashion_mnist_budgetratio_l1_m4.csv}{edmDepthOne}
\edmRatio{figures/pgfplots/data/fashion_mnist_budgetratio_l2_m4.csv}{edmDepthOne}
\edmRatio{figures/pgfplots/data/fashion_mnist_budgetratio_l4_m4.csv}{edmDepthOne}
\edmRatio{figures/pgfplots/data/fashion_mnist_budgetratio_l1_m6.csv}{edmDepthTwo}
\edmRatio{figures/pgfplots/data/fashion_mnist_budgetratio_l2_m6.csv}{edmDepthTwo}
\edmRatio{figures/pgfplots/data/fashion_mnist_budgetratio_l4_m6.csv}{edmDepthTwo}
\edmRatio{figures/pgfplots/data/fashion_mnist_budgetratio_l1_m8.csv}{edmDepthThree}
\edmRatio{figures/pgfplots/data/fashion_mnist_budgetratio_l2_m8.csv}{edmDepthThree}
\edmRatio{figures/pgfplots/data/fashion_mnist_budgetratio_l4_m8.csv}{edmDepthThree}
\addlegendimage{color=edmDepthOne, line width=0.7pt}\addlegendentry{$4$ models}
\addlegendimage{color=edmDepthTwo, line width=0.7pt}\addlegendentry{$6$ models}
\addlegendimage{color=edmDepthThree, line width=0.7pt}\addlegendentry{$8$ models}
\end{axis}
\end{tikzpicture}
\end{minipage}
    \caption{\small \textbf{Four and ten models per \((P,N)\) pair give
        nearly the same estimates.}
        Fashion-MNIST with $K_{\mathrm{train}}=K_{\mathrm{val}}=k$, for a total
        of $2k$ trained models per $(P,N)$ pair; four models
        is the CIFAR-10, FFHQ, and ImageNet configuration. \textbf{(a)} $L=4$ for each
        number of models, as in Figure~\ref{fig:fmnist_budget}. \textbf{(b)}
        the estimates with four, six, and eight models at every model size,
        as a ratio to the estimates with ten models: the average ratio
    is above one and decreases toward one as the number of models grows.}
    \label{fig:appendix_budget}
\end{figure}

\begin{table}[htbp!]
    \centering
    \caption{\small \textbf{Refitting the scaling law for every number of
        trained models.} The
        model-free estimator on Fashion-MNIST with
        $K_{\mathrm{train}}=K_{\mathrm{val}}=k$, indexed by the total number
        $2k$ of trained models per $(P,N)$ pair. The ten-model row is the
        configuration of the other Fashion-MNIST experiments. ``Mean ratio'' is
        the average ratio of the estimates to those with ten models; individual
        ratios lie in $[0.90,1.18]$, and with four models the median deviation
        is $6.2\%$. With four models against ten, $N_{\mathrm c}(P)$ changes
    by at most $8.5\%$.}
    \label{tab:appendix_budget_sensitivity}
    \small
    \begin{tabular}{rrrrrrrr}
        \toprule
        models per \((P,N)\) & mean ratio & \(a\) & \(\sigma\) & \(\ell\) &
        \(b\) & \(\hat I_{\max}-F\) & log RMSE \\
        \midrule
        4 & 1.07 & 0.383 & 0.510 & 11.50 & 0.449 & \(1.963{\times}10^4\) &
        0.1296 \\
        6 & 1.04 & 0.382 & 0.529 & 12.05 & 0.441 & \(1.935{\times}10^4\) &
        0.1276 \\
        8 & 1.02 & 0.397 & 0.519 & 11.80 & 0.472 & \(1.872{\times}10^4\) &
        0.1312 \\
        10 & 1 & 0.358 & 0.513 & 11.64 & 0.420 & \(1.872{\times}10^4\) &
        0.1026 \\
        \bottomrule
    \end{tabular}
\end{table}

\subsection{Visualizing Models Trained on Different Subsets}
\label{sub:appendix_visualization}

Figure~\ref{fig:appendix_sample_regimes} enlarges the CIFAR-10 samples at
two training-set sizes, before and at the local maximum of FID,
from models at $L=5$. Before the maximum, the models produce sharp
near-replicas of their training images, and FID is low. At the maximum, the
models no longer reproduce training images, and their new samples are less
sharp than the reproduced training images, so FID is much higher. At larger $N$, FID
decreases below its value before the maximum and continues to decrease up to
the full dataset.

\begin{figure}[htbp!]
    \centering
    \begin{minipage}[t]{0.49\linewidth}
        \centering
        \includegraphics[width=\linewidth]
        {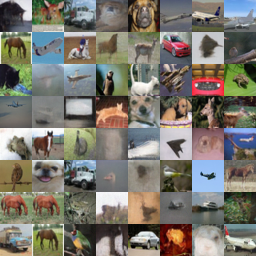}
        \par\smallskip
        \small (a) $L=5$, $N=710\approx0.30\,N_{\mathrm c}(P)$: memorization
        ratio $0.607$, FID $30.6$
    \end{minipage}
    \hfill
    \begin{minipage}[t]{0.49\linewidth}
        \centering
        \includegraphics[width=\linewidth]
        {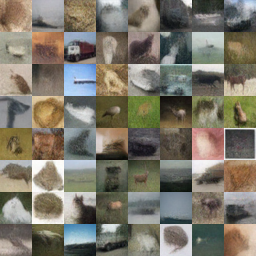}
        \par\smallskip
        \small (b) $L=5$, $N=2900\approx1.2\,N_{\mathrm c}(P)$: memorization
        ratio $0.001$, FID $89.2$
    \end{minipage}
    \caption{\small \textbf{Samples across the memorization transition.}
        Both panels show samples from the same CIFAR-10 $L=5$ architecture,
        with $N_{\mathrm c}(P)\approx2380$.
        The $N=710$ model sometimes repeats sharp training-set images
        and scores better on
        FID. The $N=2900$ model has stopped memorizing and scores substantially
        worse. At $N=6500$ FID is $23.1$, and on the full dataset it is
    $7.8$.}
    \label{fig:appendix_sample_regimes}
\end{figure}

\end{document}